\documentclass{article} 
\usepackage{iclr2025_conference,times}

\usepackage{amsmath,amsfonts,bm}

\def\eqref#1{equation~\ref{#1}}
\def\1{\bm{1}}

\DeclareMathAlphabet{\mathsfit}{\encodingdefault}{\sfdefault}{m}{sl}
\SetMathAlphabet{\mathsfit}{bold}{\encodingdefault}{\sfdefault}{bx}{n}

\usepackage{hyperref}
\usepackage{url}

\usepackage{longtable}
\usepackage{multirow}
\usepackage{booktabs}
\usepackage{caption}
\usepackage{subcaption}
\usepackage[toc,page,header]{appendix}
\usepackage{minitoc}

\newcommand{\halfday}{$> 12$ hours }
\newcommand{\oneday}{$> 1$ day}

\usepackage{amsthm, amssymb}
\usepackage{tikz}
\usetikzlibrary{shapes,arrows,calc,positioning,decorations.markings}
\usepackage{float}
\usepackage{subcaption}
\usepackage{tcolorbox}
\usepackage{xcolor}
\usepackage{enumitem}
\usepackage{algorithm}
\usepackage{algpseudocode}
\usepackage{dashbox}
\usepackage{arydshln}
\usepackage{colortbl}



\newtheoremstyle{compactstyle}%
	{\topsep}   % space above
	{0pt}   % space below 
	{\itshape}  % body font 
	{}      % no indent 
	{\bfseries} % theorem head font 
	{.}     % punctuation after theorem head 
	{ }     % space after theorem head 
	{}      % theorem head spec

\theoremstyle{compactstyle}
\newtheorem{theorem}{Theorem}[section]
\newtheorem{proposition}[theorem]{Proposition}
\newtheorem{lemma}[theorem]{Lemma}
\newtheorem{remark}[theorem]{Remark}
\newtheorem{definition}[theorem]{Definition}
\newtheoremstyle{assumptionstyle} 
  {\topsep}                       
  {0pt}                       
  {\itshape}                      
  {0pt}                           
  {\bfseries}                     
  {.}                             
  { }                             
  {\thmname{#1}\thmnumber{ #2}\thmnote{ (#3)}}

\newtheorem{assumption}{Assumption}

\newcommand{\takeawaymessage}[2]{
    \begin{tcolorbox}[
        colback=gray!10,
        colframe=gray!50,
        fonttitle=\bfseries,
        title=Takeaway Message #1,
        left=5pt,
        right=5pt,
        top=5pt,
        bottom=5pt,
        arc=0pt]
        #2
    \end{tcolorbox}
}

\newcommand{\AGG}{\operatorname{AGG}}
\newcommand{\UPDATE}{\operatorname{UPDATE}}

\definecolor{gold}{RGB}{255,215,0}
\definecolor{silver}{RGB}{192,192,192}
\definecolor{babyblue}{RGB}{173,216,230}
\definecolor{babypink}{RGB}{255,182,193}
\definecolor{myblue}{rgb}{0.38, 0.31, 0.86}
\definecolor{myred}{rgb}{1.0, 0.71, 0.76}
\definecolor{mygreen}{rgb}{0.0, 0.5, 0.0}
\definecolor{mypurple}{rgb}{0.6, 0.2, 0.8}
\newcommand{\increased}[1]{\textcolor[rgb]{0,0.5,0}{#1}}
\newcommand{\decreased}[1]{\textcolor{red}{#1}}

\title{Demystifying Topological Message-Passing with Relational Structures: A Case Study on Oversquashing in Simplicial Message-Passing}

\author{
\footnotemark[1]\hspace{0.5em}\textbf{Diaaeldin Taha}${}^{1}$,
\footnotemark[1]\hspace{0.5em}\textbf{James Chapman}$^{2}$,
\footnotemark[2]\hspace{0.5em}\textbf{Marzieh Eidi}$^{1,3}$,
\footnotemark[2]\hspace{0.5em}\textbf{Karel Devriendt}$^{4}$,
\textbf{Guido Mont\'ufar}$^{1,2}$ \\[6pt]
{$^{1}$Max Planck Institute for Mathematics in the Sciences, Leipzig, Germany} \\
{$^{2}$UCLA, CA, USA} \\
{$^{3}$Center for Scalable Data Analytics and Artificial Intelligence, Leipzig, Germany} \\
{$^{4}$University of Oxford, Oxford, UK} \\[2pt]
\texttt{taha@mis.mpg.de, chapman20j@math.ucla.edu, meidi@mis.mpg.de,} \\
\texttt{karel.devriendt@maths.ox.ac.uk, montufar@math.ucla.edu} \\[2pt]
}

\iclrfinalcopy 

\begin{document}

\maketitle

\footnotetext[1]{Equal contribution.}
\footnotetext[2]{Equal contribution.}

%added for appendix table of contents
\doparttoc
\faketableofcontents

\begin{abstract}
Topological deep learning (TDL) has emerged as a powerful tool for modeling higher-order interactions in relational data. However, phenomena such as oversquashing in topological message-passing remain understudied and lack theoretical analysis. We propose a unifying axiomatic framework that bridges graph and topological message-passing by viewing simplicial and cellular complexes and their message-passing schemes through the lens of relational structures. This approach extends graph-theoretic results and algorithms to higher-order structures, facilitating the analysis and mitigation of oversquashing in topological message-passing networks. Through theoretical analysis and empirical studies on simplicial networks, we demonstrate the potential of this framework to advance TDL.
\end{abstract}

\section{Introduction}
\label{sec:introduction}

Recent years have witnessed a growing recognition that traditional machine learning, rooted in Euclidean spaces, often fails to capture the complex structure and relationships present in real-world data. This shortcoming has driven the development of geometric deep learning (GDL) \citep{bronstein2021geometric} and, more recently, topological deep learning (TDL) \citep{hajij2022topological}, for handling non-Euclidean and relational data. TDL, in particular, has emerged as a promising frontier for relational learning, that extends beyond graph neural networks (GNNs). TDL offers tools to capture and analyze higher-order interactions and topological features in complex data and higher-order structures, such as simplicial complexes, cell complexes, and sheaves \citep{hajij2022topological}. However, the TDL field is young, and the TDL community has yet many open theoretical and practical questions relating to, e.g., oversquashing and rewiring (research directions 2 and 9 of \citealp{pmlr-v235-papamarkou24a}).

Oversquashing is a challenging failure mode in GNNs, where information struggles to propagate across long paths due to the compression of an exponentially growing number of messages into fixed-size vectors \citep{alon2021on}. This phenomenon has been examined through various perspectives, including curvature \citep{topping2022understanding}, graph expansion \citep{9929363}, effective resistance \citep{pmlr-v202-black23a}, and spectral properties \citep{karhadkar2023fosr}. Despite the potential of higher-order message passing architectures—such as simplicial neural networks \citep{ebli2020simplicial}, message passing simplicial networks \citep{pmlr-v139-bodnar21a}, and CW networks \citep{bodnar2021weisfeiler}—there remains a lack of unified frameworks for analyzing and mitigating oversquashing in these settings.

In this paper, we take a first step toward studying oversquashing in TDL by showing that simplicial complexes and their message passing schemes can be interpreted as relational structures, making it possible to extend key GNN insights and tools to higher-order message passing architectures. The conceptual framework and theoretical results developed in this paper address pressing questions in the TDL community (e.g., research directions 2 and 9 of \citealp{pmlr-v235-papamarkou24a}).

\textbf{Contributions.} Our contributions are threefold:
\begin{itemize}[noitemsep,nolistsep]
\item \textbf{Axiomatic}: We provide a unifying view of simplicial complexes and their message passing schemes through the lens of relational structures. 
\item \textbf{Theoretical}: We introduce \textit{influence graphs} which enable novel extensions of prior graph analyses to higher-order structures, where existing methods for analysis do not apply. We extend graph-theoretic concepts and results on oversquashing to relational structures, analyzing network sensitivity (Lemma~\ref{lemma: jacobian bound higher order}), local geometry (Proposition~\ref{prop:local-geometry}), and the impact of network depth (Theorem~\ref{theorem: impact of depth}) and hidden dimensions (Section~\ref{subsec:impact-of-hidden-dimensions}). 
\item \textbf{Practical}: We propose a heuristic to extend oversquashing-mitigation techniques from graph-based models to relational structures.
\end{itemize}

\textbf{Related Work.} Our work sits at the intersection of graph neural networks, topological deep learning, relational learning, and the study of oversquashing and graph rewiring in graph neural networks. We review related work in Appendix~\ref{appendix:related-work}.

The rest of this paper is organized as follows: 
Section~\ref{sec:relational_message_passing} provides the axiomatic groundwork for relating simplicial and relational message passing. Section~\ref{sec:sensitivity_analysis} presents our theoretical analysis of oversquashing in this context.
Section~\ref{sec:rewiring} introduces a heuristic rewiring strategy to mitigate oversquashing in relational message passing.
Section~\ref{sec:experiments} presents our experimental results, followed by a discussion and conclusions in Section~\ref{sec:discussion-and-conclusions}.

\section{Simplicial Complexes are Relational Structures}
\label{sec:relational_message_passing}

In this section, we reinterpret simplicial complexes and message passing through the lens of \emph{relational structures}. We begin by recalling simplicial complexes and a representative simplicial message passing scheme, then reframe these notions within a relational framework. We illustrate the definitions in this section with a small worked example in Appendx~\ref{appendix:worked-example}. We note that the connection we establish here extends to other higher-order structures, such as cellular complexes \citep{hansen2019toward,bodnar2021weisfeiler,giusti2024cinplusplus}.

\subsection{Simplicial Complexes and Message Passing}
\label{subsec:simplicial-complexes-and-message-passing}

Simplicial complexes are mathematical structures that generalize graphs to higher dimensions, capturing relationships among vertices, edges, triangles, and higher-dimensional objects. 

\begin{definition}[Simplicial Complex, \citealp{nanda2021computational}]
Let $V$ be a non-empty set. 
A \emph{simplicial complex} $\mathcal{K}$ is a collection of non-empty subsets of $V$ that contains all the singleton subsets of $V$ and is closed under the operation of taking non-empty subsets. 
\end{definition}

A member $\sigma = \{v_0, v_1, \ldots, v_d\}\in\mathcal{K}$ with cardinality $|\sigma| = d + 1$ is called a \emph{$d$-simplex}. Geometrically, $0$-simplices are vertices, $1$-simplices are edges, $2$-simplices are triangles, and so on.

\begin{definition}[Boundary Incidence Relation]
\label{def:boundary_incidence}
We say that \emph{$\tau$ covers $\sigma$}, written $\sigma \prec \tau$, iff $\sigma \subsetneq \tau$ and there is no $\delta \in \mathcal{K}$ such that $\sigma \subsetneq \delta \subsetneq \tau$. 
\end{definition}

The incidence relations from Definition~\ref{def:boundary_incidence} can be used to construct four types of (local) adjacencies.

\begin{definition}
\label{def:adjacency_relations}
Consider a simplex $\sigma \in \mathcal{K}$. Four types of adjacent simplices can be defined:
\begin{enumerate}[noitemsep,nolistsep]
    \item \emph{Boundary adjacency}: $\mathcal{B}(\sigma) = \{\tau \colon \tau \prec \sigma\}$;  
    \item \emph{Co-boundary adjacency}: $\mathcal{C}(\sigma) = \{\tau \colon \sigma \prec \tau\}$; 
    \item \emph{Lower adjacency}: $\mathcal{N}_{\downarrow}(\sigma) = \{\tau \colon \exists \delta \text{ such that } \delta \prec \tau \text{ and } \delta \prec \sigma\}$;  
    \item \emph{Upper adjacency}: $\mathcal{N}_{\uparrow}(\sigma) = \{\tau \colon \exists \delta \text{ such that } \tau \prec \delta \text{ and } \sigma \prec \delta\}$. 
\end{enumerate}
\end{definition} 

In Figure~\ref{fig:simplicial-relational}, we illustrate an example of a simplicial complex and its adjacency relations.

\definecolor{myblue}{rgb}{0.38, 0.31, 0.86}
\definecolor{myred}{rgb}{1.0, 0.71, 0.76}

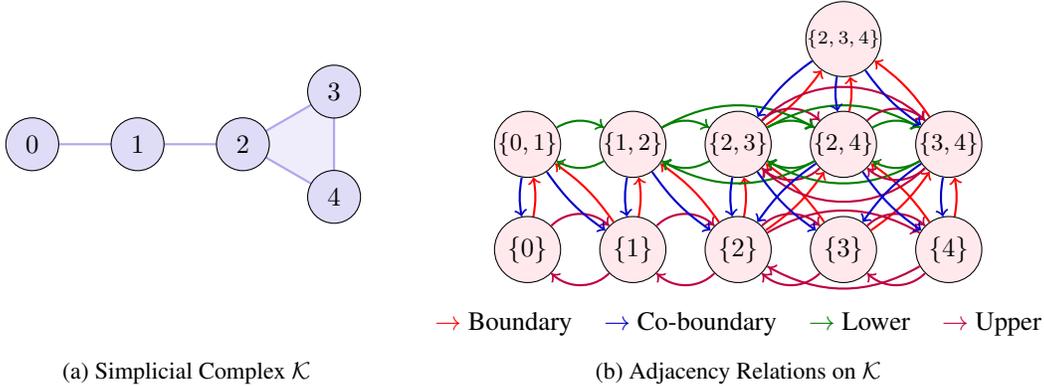
\begin{figure}[htbp]
    \centering
    \begin{subfigure}[b]{0.35\textwidth}
        \centering
        \begin{tikzpicture}[scale=0.7]
            \tikzstyle{vertex}=[circle,fill=myblue!20,draw,minimum size=20pt,inner sep=0pt]
            \tikzstyle{edge}=[thick,-,myblue!50]
            \tikzstyle{face}=[fill=myblue!20,opacity=0.5]
            
            \filldraw[face] (4,0) -- (5.732, 1) -- (5.732, -1) -- cycle;
            
            \draw[edge] (0,0) -- (2,0) -- (4,0);
            \draw[edge] (4,0) -- (5.732, 1) -- (5.732, -1) -- (4,0);

            \node[vertex] at (0,0) {$0$};
            \node[vertex] at (2,0) {$1$};
            \node[vertex] at (4,0) {$2$};
            \node[vertex] at (5.732, 1) {$3$};
            \node[vertex] at (5.732, -1) {$4$};
        \end{tikzpicture}
        \vspace{1.6cm}
        \caption{Simplicial Complex $\mathcal{K}$}
        \label{fig:simplicial-complex}
    \end{subfigure}
    \hfill
    \begin{subfigure}[b]{0.6\textwidth}
        \centering
        \begin{tikzpicture}[scale=0.7]
            \tikzstyle{rgnode}=[circle,fill=myred!30,draw, minimum size=25pt,inner sep=0pt]
            \tikzstyle{boundary}=[->,thick,blue!80!black,bend right=10]
            \tikzstyle{coboundary}=[->,thick,red,bend right=10]
            \tikzstyle{lower}=[thick,green!50!black,decoration={markings,mark=at position 0.99 with {\arrow{>}}},postaction={decorate}]
            \tikzstyle{upper}=[thick,purple,decoration={markings,mark=at position 0.99 with {\arrow{>}}},postaction={decorate}]
            
            % nodes
            \node[rgnode] (234) at (6,2) {\scriptsize$\{2,3,4\}$};
            \node[rgnode] (01) at (0,0) {\small$\{0,1\}$};
            \node[rgnode] (12) at (2,0) {\small$\{1,2\}$};
            \node[rgnode] (23) at (4,0) {\small$\{2,3\}$};
            \node[rgnode] (24) at (6,0) {\small$\{2,4\}$};
            \node[rgnode] (34) at (8,0) {\small$\{3,4\}$};
            \node[rgnode] (0) at (0,-2) {$\{0\}$};
            \node[rgnode] (1) at (2,-2) {$\{1\}$};
            \node[rgnode] (2) at (4,-2) {$\{2\}$};
            \node[rgnode] (3) at (6,-2) {$\{3\}$};
            \node[rgnode] (4) at (8,-2) {$\{4\}$};
            
            % edges
            \draw[coboundary] (23) to (234);
            \draw[coboundary] (24) to (234);
            \draw[coboundary] (34) to (234);
            \draw[lower] (01) to[out=30,in=150] (12);
            \draw[coboundary] (0) to (01);
            \draw[coboundary] (1) to (01);
            \draw[lower] (12) to[out=-150,in=-30] (01);
            \draw[lower] (12) to[out=30,in=150] (23);
            \draw[lower] (12) to[out=30,in=150] (24);
            \draw[coboundary] (1) to (12);
            \draw[coboundary] (2) to (12);
            \draw[boundary] (234) to (23);
            \draw[lower] (23) to[out=-150,in=-30] (12);
            \draw[lower] (23) to[out=30,in=150] (24);
            \draw[upper] (23) to[out=45,in=135] (24);
            \draw[lower] (23) to[out=30,in=150] (34);
            \draw[upper] (23) to[out=45,in=135] (34);
            \draw[coboundary] (2) to (23);
            \draw[coboundary] (3) to (23);
            \draw[boundary] (234) to (24);
            \draw[lower] (24) to[out=-150,in=-30] (12);
            \draw[lower] (24) to[out=-150,in=-30] (23);
            \draw[upper] (24) to[out=-135,in=-45] (23);
            \draw[lower] (24) to[out=30,in=150] (34);
            \draw[upper] (24) to[out=45,in=135] (34);
            \draw[coboundary] (2) to (24);
            \draw[coboundary] (4) to (24);
            \draw[boundary] (234) to (34);
            \draw[lower] (34) to[out=-150,in=-30] (23);
            \draw[upper] (34) to[out=-135,in=-45] (23);
            \draw[lower] (34) to[out=-150,in=-30] (24);
            \draw[upper] (34) to[out=-135,in=-45] (24);
            \draw[coboundary] (3) to (34);
            \draw[coboundary] (4) to (34);
            \draw[boundary] (01) to (0);
            \draw[upper] (0) to[out=45,in=135] (1);
            \draw[boundary] (01) to (1);
            \draw[boundary] (12) to (1);
            \draw[upper] (1) to[out=-135,in=-45] (0);
            \draw[upper] (1) to[out=45,in=135] (2);
            \draw[boundary] (12) to (2);
            \draw[boundary] (23) to (2);
            \draw[boundary] (24) to (2);
            \draw[upper] (2) to[out=-135,in=-45] (1);
            \draw[upper] (2) to[out=45,in=135] (3);
            \draw[upper] (2) to[out=30,in=150] (4);
            \draw[boundary] (23) to (3);
            \draw[boundary] (34) to (3);
            \draw[upper] (3) to[out=-135,in=-45] (2);
            \draw[upper] (3) to[out=45,in=135] (4);
            \draw[boundary] (24) to (4);
            \draw[boundary] (34) to (4);
            \draw[upper] (4) to[out=-150,in=-30] (2);
            \draw[upper] (4) to[out=-135,in=-45] (3);

            % legend
            \node[anchor=north] at (4,-3.0) {\textcolor{red}{$\rightarrow$} Boundary \hspace{0.25cm} \textcolor{blue!80!black}{$\rightarrow$} Co-boundary   \hspace{0.25cm} \textcolor{green!50!black}{$\rightarrow$} Lower \hspace{0.25cm} \textcolor{purple}{$\rightarrow$} Upper};
            
        \end{tikzpicture}
        \caption{Adjacency Relations on $\mathcal{K}$}
        \label{fig:relational-graph}
    \end{subfigure}
    \caption{The left panel shows a simplicial complex $\mathcal{K}$ consisting of five vertices, five edges, and one triangle. The right panel shows the corresponding adjacency relations depicted as arrows to each simplex $\sigma\in\mathcal{K}$ emanating from each of its adjacent simplices in $\mathcal{B}(\sigma)$, $\mathcal{C}(\sigma)$, $\mathcal{N}_{\downarrow}(\sigma)$, $\mathcal{N}_{\uparrow}(\sigma)$.}
    \label{fig:simplicial-relational}
\end{figure}

We now, following \citet[Section 4]{pmlr-v139-bodnar21a}, review a general scheme for message passing on simplicial complexes. In Appendix~\ref{appendix:related-work}, we provide references for topological message passing architectures that fit this scheme. We refer readers to Appendix~\ref{appendix:implementation:layers} for specific instantiations of this scheme in our graph and topological message passing models.

Simplicial message passing extends graph message passing from pairwise node connections to higher-dimensional adjacency connections between simplices.
Message passing schemes on simplicial complexes iteratively update feature vectors assigned to simplices by exchanging messages between adjacent simplices.
For a simplicial complex $\mathcal{K}$, we denote the feature vector of a simplex $\sigma \in \mathcal{K}$ as $\mathbf{h}_\sigma \in \mathbb{R}^p$. At each iteration (layer) $t$, the feature vectors $\mathbf{h}_\sigma^{(t)}$ of simplices $\sigma\in\mathcal{K}$ are updated by aggregating messages from adjacent simplices. For a simplex $\sigma \in \mathcal{K}$, the messages passed from adjacent simplices are defined as follows: 
\begin{align}
\begin{split}
\mathbf{m}_{\mathcal{B}}\textsuperscript{\smash{$(t+1)$}}(\sigma) &= \AGG_{\tau \in \mathcal{B}(\sigma)} \left( M_{\mathcal{B}}(\mathbf{h}_\sigma\textsuperscript{\smash{$(t)$}}, \mathbf{h}_\tau\textsuperscript{\smash{$(t)$}}) \right), \\[-0.2ex]
\mathbf{m}_{\mathcal{C}}\textsuperscript{\smash{$(t+1)$}}(\sigma) &= \AGG_{\tau \in \mathcal{C}(\sigma)} \left( M_{\mathcal{C}}(\mathbf{h}_\sigma\textsuperscript{\smash{$(t)$}}, \mathbf{h}_\tau\textsuperscript{\smash{$(t)$}}) \right), \\[-0.2ex]
\mathbf{m}_{\downarrow}\textsuperscript{\smash{$(t+1)$}}(\sigma) &= \AGG_{\tau \in \mathcal{N}_{\downarrow}(\sigma)} \left( M_{\downarrow}(\mathbf{h}_\sigma\textsuperscript{\smash{$(t)$}}, \mathbf{h}_\tau\textsuperscript{\smash{$(t)$}}, \mathbf{h}_{\sigma \cap \tau}\textsuperscript{\smash{$(t)$}}) \right), \\[-0.2ex]
\mathbf{m}_{\uparrow}\textsuperscript{\smash{$(t+1)$}}(\sigma) &= \AGG_{\tau \in \mathcal{N}_{\uparrow}(\sigma)} \left( M_{\uparrow}(\mathbf{h}_\sigma\textsuperscript{\smash{$(t)$}}, \mathbf{h}_\tau\textsuperscript{\smash{$(t)$}}, \mathbf{h}_{\sigma \cup \tau}\textsuperscript{\smash{$(t)$}}) \right).
\end{split}
\end{align}
Here $\AGG$ is an aggregation function (e.g., sum or mean), and $M_\mathcal{B}, M_\mathcal{C}, M_\downarrow, M_\uparrow$ are message functions (e.g., linear or MLP). Then, an update operation $\UPDATE$ (e.g., MLP) incorporates these four different types of incoming messages:
\begin{equation}
\mathbf{h}_\sigma^{(t+1)} = \UPDATE\left(\mathbf{h}_\sigma^{(t)}, \mathbf{m}_\mathcal{B}^{(t+1)}(\sigma), \mathbf{m}_\mathcal{C}^{(t+1)}(\sigma), \mathbf{m}_{\downarrow}^{(t+1)}(\sigma), \mathbf{m}_{\uparrow}^{(t+1)}(\sigma)\right). 
\end{equation}
Finally, after the last iteration, a read-out function is applied to process the features to perform a desired task, such as classification or regression. 

\subsection{Relational Structures and Message Passing}

We model simplicial complexes and the above message passing scheme using \emph{relational structures}.

\begin{definition}[Relational Structure, \citealp{hodges1993model}]
A \emph{relational structure} $\mathcal{R} = (\mathcal{S}, R_1, \ldots, R_k)$ consists of a finite set $\mathcal{S}$ of \emph{entities}, and relations $R_i \subseteq \mathcal{S}^{n_i}$, 
where $n_i$ is the arity of $R_i$. 
\end{definition}
We note that modeling simplicial complexes as relational structures generalizes a powerful perspective in which simplicial complexes and similar constructs are treated as augmented Hasse diagrams
as demonstrated, for example, by \citet{hajij2022topological}, \citet{eitan2024topological}, and \citet{papillon2024topotune}.

We now introduce a general scheme for message passing on relational structures which encompasses the simplicial message passing scheme from Section~\ref{subsec:simplicial-complexes-and-message-passing}. This scheme is an extension of the relational graph convolution model from \citet{schlichtkrull2018modeling} which allows for relations of different arities, not just binary relations.

\begin{definition}[Relational Message Passing Model]
\label{def:relational_message_passing}
A \emph{relational message passing model} on a relational structure $\mathcal{R} = (\mathcal{S}, R_1, \ldots, R_k)$ consists of:
\begin{itemize}[noitemsep,nolistsep, leftmargin=1em]
    \item \emph{Feature vectors}: $\mathbf{h}_\sigma^{(t)} \in \mathbb{R}^{p_t}$ for each $\sigma \in \mathcal{S}$ at layer $t \geq 0$, initialized as $\mathbf{h}_\sigma^{(0)} = \mathbf{x}_\sigma$ (input features). Here, $p_t$ denotes the dimensionality of the feature vectors at layer $t$.
    \item \emph{Message functions} $\boldsymbol{\psi}_{i}^{(t)}\colon \mathbb{R}^{p_t} \times \cdots \times \mathbb{R}^{p_t} \to \mathbb{R}^{p_{i,t}}$ (with $n_i$ arguments) for each relation $R_i$, where $i=1,\ldots,k$. Each message function takes $n_i$ input feature vectors (corresponding to the target simplex and its $n_i-1$ related simplices) and outputs a message vector of dimension $p_{i,t}$. The parameter $n_i$ represents the arity of the relation $R_i$.
    \item \emph{Update function} $\boldsymbol{\phi}^{(t)}\colon \mathbb{R}^{p_{1,t}} \times \cdots \times \mathbb{R}^{p_{k,t}} \to \mathbb{R}^{p_{t+1}}$. The output dimension $p_{t+1}$ specifies the dimensionality of the feature vectors at layer $t+1$.
    \item \emph{Shift operators} $\mathbf{A}^{R_i} \in \mathbb{R}_{\geq 0}^{|\mathcal{S}|^{n_i}}$ associated with each relation $R_i$, for $i = 1, \ldots, k$. For each $\sigma \in \mathcal{S}$ and $\boldsymbol{\xi} = (\xi_1, \ldots, \xi_{n_i - 1}) \in \mathcal{S}^{n_i-1}$ with $(\sigma, \boldsymbol{\xi}) \in R_i$, the element $\mathbf{A}^{R_i}_{\sigma, \boldsymbol{\xi}}$ represents the strength of the signal passed from $\boldsymbol{\xi}$ to $\sigma$. More specifically, for any combination of entities $(\zeta_1, \zeta_2, \dots, \zeta_{n_i}) \in \mathcal{S}^{n_i}$ where the relation $R_i$ does not hold among the entities $\zeta_1, \zeta_2, \dots, \zeta_{n_i}$ (i.e., $(\zeta_1, \zeta_2, \dots, \zeta_{n_i}) 
\notin R_i$), the tensor $\mathbf{A}^{R_i}$ satisfies $\mathbf{A}_{\zeta_1, \zeta_2, \dots, \zeta_{n_i}}^{R_i} = 0$.
\end{itemize}
The update rule is given by:
\begin{equation}\label{eq: update rule higher order}
\mathbf{h}_\sigma^{(t + 1)} = \boldsymbol{\phi}^{(t)} \left( \mathbf{m}_{\sigma,1}^{(t)}, \ldots, \mathbf{m}_{\sigma,k}^{(t)} \right),
\end{equation}
where for each $i = 1, \ldots, k$, the message $\mathbf{m}_{\sigma,i}^{(t)}$ received by $\sigma$ over $R_i$ is computed as:
\begin{equation}
\mathbf{m}_{\sigma,i}^{(t)} = \sum_{\boldsymbol{\xi} \in \mathcal{S}^{n_i - 1}} \mathbf{A}^{R_i}_{\sigma, \boldsymbol{\xi}}\, \boldsymbol{\psi}_{i}^{(t)}\left( \mathbf{h}_\sigma^{(t)}, \mathbf{h}_{\xi_1}^{(t)}, \ldots, \mathbf{h}_{\xi_{n_i-1}}^{(t)} \right).
\end{equation}
\end{definition}

\begin{remark}
The shift operators in Definition~\ref{def:relational_message_passing} extend the definition of graph shift operators \citep{mateos2019connecting, gama2020stability, dasoulas2021learning} from graphs to relational structures.
Whereas relations indicate whether entities are connected, shift operators numerically encode these connections with weights.
\end{remark}

In the context of message passing, a simplicial complex $\mathcal{K}$ can be viewed as a relational structure $\mathcal{R}(\mathcal{K}) = (\mathcal{S}, R_1, \ldots, R_5)$, where $\mathcal{S} = \mathcal{K}$ are the entities, and $R_i$ are relations defined as follows: $R_1 = \{(\sigma) : \sigma \in \mathcal{K}\}$ (identity), $R_2 = \{(\sigma,\tau) : \sigma \in \mathcal{K}, \tau \in \mathcal{B}(\sigma)\}$ (boundary), $R_3 = \{(\sigma,\tau) : \sigma \in \mathcal{K} \tau \in \mathcal{C}(\sigma)\}$ (co-boundary), $R_4 = \{(\sigma,\tau,\delta) : \sigma \in \mathcal{K}, \tau \in \mathcal{N}_{\downarrow}(\sigma), \delta = \sigma \cap \tau\}$ (lower adjacency), $R_5 = \{(\sigma,\tau,\delta) : \sigma\in\mathcal{K},\tau \in \mathcal{N}_{\uparrow}(\sigma), \delta = \sigma \cup \tau\}$ (upper adjacency). The message functions $\boldsymbol{\psi}_i^{(t)}$ correspond to $M_{\mathcal{B}}, M_{\mathcal{C}}, M_{\downarrow}, M_{\uparrow}$, the update function $\boldsymbol{\phi}^{(t)}$ to $\UPDATE$, and aggregation uses shift operators $\mathbf{A}^{R_i}$. This establishes an equivalence between message passing on the simplicial complex $\mathcal{K}$ and the relational structure $\mathcal{R}(\mathcal{K})$.

\begin{remark}
The relational message passing scheme in Definition~\ref{def:relational_message_passing} encompasses relational graph neural networks \citep{schlichtkrull2018modeling}, simplicial neural networks \citep{pmlr-v139-bodnar21a}, higher-order graph neural networks~\citep{morris2019weisfeiler}, and cellular complex neural networks~\citep{bodnar2021weisfeiler}. We demonstrate how higher-order graphs fit the relational framework in Appendix~\ref{appendix:higher-order-graphs}.
\end{remark}

\takeawaymessage{1 (Axiomatic)}{
Simplicial complexes can be represented as \emph{relational structures}, where the entities are simplices, and the relations capture the adjacency among simplices of different dimensions. Simplicial message passing is an instance of relational message passing on these structures.
}

\section{Oversquashing in Relational Message-Passing}
\label{sec:sensitivity_analysis}

The existing literature on oversquashing in GNNs does not directly address relational message passing.  In this section, we address that gap by deriving new extensions of key results on oversquashing in GNNs to relational message passing. We illustrate the definitions in this section with a small worked example in Appendx~\ref{appendix:worked-example}.

In our analysis of relational structures and message passing schemes, we naturally encounter matrices and graphs that capture the aggregated influence of the underlying shift operators. For convenience, we introduce notation for these matrices and graphs. For each relation $R_i$ of arity $n_i$ 
with shift operator $\mathbf{A}^{R_i}$, we define the matrix $\tilde{\mathbf{A}}^{R_i} \in \mathbb{R}_{\geq 0}^{|\mathcal{S}| \times |\mathcal{S}|}$ as:
\begin{equation}
\tilde{\mathbf{A}}^{R_i}_{\sigma, \tau} = \sum_{j=1}^{n_i - 1} \sum_{\boldsymbol{\xi} \in \mathcal{S}^{n_i - 2}} \mathbf{A}^{R_i}_{\sigma, \xi_1, \ldots, \xi_{j-1}, \tau, \xi_j, \ldots, \xi_{n_i - 2}},\quad \sigma, \tau\in \mathcal{S}.  
\end{equation}
This matrix captures all possible ways an entity $\tau$ can influence entity $\sigma$ via the relation $R_i$. Specifically, it sums over all positions $j$ where $\tau$ can appear among the arguments of the shift operator $\mathbf{A}^{R_i}$, and over all possible combinations of the other entities $\boldsymbol{\xi}$.

We aggregate these matrices over all relations to form the \emph{aggregated influence matrix} $\tilde{\mathbf{A}}\in \mathbb{R}_{\geq 0}^{|\mathcal{S}| \times |\mathcal{S}|}$: 
\begin{equation}
\tilde{\mathbf{A}} = \sum_{i=1}^{k} \tilde{\mathbf{A}}^{R_i}.
\label{eq:aggregated_influence_matrix_sum}
\end{equation}
Next, we define the \emph{augmented influence matrix} $\mathbf{B}$, which plays the role of an augmented adjacency matrix in our analysis: 
\begin{equation}
\mathbf{B} = \gamma \mathbf{I} + \tilde{\mathbf{A}},
\label{eq:augmented_influence_matrix}
\end{equation}
where $\gamma = \max_\sigma \sum_{\boldsymbol{\xi} \in \mathcal{S}^{n_i - 1}} \tilde{\mathbf{A}}_{\sigma, \boldsymbol{\xi}}$ is the maximum row sum of $\tilde{\mathbf{A}}$. 

Lastly, we introduce graphs that capture the aggregated message passing dynamical structure implied by the relational structure and the message passing scheme.
\begin{definition}[Influence Graph]
\label{def:influence_graph}
Given a relational structure $\mathcal{R} = (\mathcal{S}, R_1, \ldots, R_k)$ and a relational message passing scheme with update rule given by Equation~\ref{eq: update rule higher order}, and given $\mathbf{Q} \in \{\tilde{\mathbf{A}}, \mathbf{B}\}$, where $\tilde{\mathbf{A}}$ and $\mathbf{B}$ are defined by Equations~\ref{eq:aggregated_influence_matrix_sum} and \ref{eq:augmented_influence_matrix} respectively, we define the \emph{influence graph} $\mathcal{G}(\mathcal{S}, \mathbf{Q}) = (\mathcal{S}, \mathcal{E}, w)$ as follows: The set of entities (i.e., nodes) is $\mathcal{S}$. The set of edges $\mathcal{E}$ consists of directed edges from entity $\tau$ to entity $\sigma$ for each pair $(\sigma, \tau) \in \mathcal{S} \times \mathcal{S}$ with $\mathbf{Q}_{\sigma, \tau} > 0$. Each edge from $\tau$ to $\sigma$ is assigned a weight $w_{\tau \rightarrow \sigma} = \mathbf{Q}_{\sigma, \tau}$.
\end{definition}
As we will see 
next, these graphs make it possible to leverage and extend graph-theoretic concepts, results, and intuition to
understand and analyze the behavior of relational message passing schemes.

\subsection{Sensitivity Analysis}
\label{subsec:sensitivity-analysis}

We now analyze the sensitivity of relational message passing to changes in the input features. This analysis is crucial for understanding how information propagates through the network and for identifying potential bottlenecks or oversquashing effects. We begin with a standard assumption about the boundedness of the Jacobians of the message and update functions.

\begin{assumption}[Bounded Jacobians]
\label{assump:bounded_jacobian}
All message functions $\boldsymbol{\psi}_{i}\textsuperscript{\smash{$(\ell)$}}$ and update functions $\boldsymbol{\phi}\textsuperscript{\smash{$(\ell)$}}$ are differentiable with bounded Jacobians: 
There exist constants $\beta_{i}\textsuperscript{\smash{$(\ell)$}}$ and $\alpha\textsuperscript{\smash{$(\ell)$}}$ such that 
$\left\Vert \partial \boldsymbol{\psi}_{i}\textsuperscript{\smash{$(\ell)$}} / \partial \mathbf{h}_\sigma \right\Vert_{1} \leq \beta_{i}\textsuperscript{\smash{$(\ell)$}}$ 
for any input feature vector $\mathbf{h}_\sigma$, and 
$\left\Vert \partial \boldsymbol{\phi}\textsuperscript{\smash{$(\ell)$}} / \partial \mathbf{m}_j \right\Vert_{1} \leq \alpha\textsuperscript{\smash{$(\ell)$}}$ 
for any message input $\mathbf{m}_j$. 
We write $\beta\textsuperscript{\smash{$(\ell)$}} = \max_i \beta_{i}\textsuperscript{\smash{$(\ell)$}}$.
\end{assumption}

Our main result on sensitivity is the following, which
is a novel extension of GNN sensitivity analysis results (e.g., \citealp[Lemma~1]{topping2022understanding} and \citealp[Theorem~3.2]{pmlr-v202-di-giovanni23a}) to relational (and topological) message passing. We provide the proof in Appendix~\ref{appendix-subsec-proofs-sensitivity-analysis}. 

\begin{lemma}[Sensitivity Bound for Relational Message Passing] 
\label{lemma: jacobian bound higher order}
Consider a relational structure $\mathcal{R} = (\mathcal{S}, R_1, \ldots, R_k)$ with update rule given by Equation~\ref{eq: update rule higher order} and satisfying Assumption~\ref{assump:bounded_jacobian}. Then, for any $\sigma, \tau \in \mathcal{S}$ and $t > 0$, the Jacobian at layer $t$ with respect to the input features ($t = 0$) 
satisfies 
\begin{equation}
\left\Vert \frac{\partial \mathbf{h}_{\sigma}^{(t)}}{\partial \mathbf{h}^{(0)}_{\tau}} \right\Vert_{1} \leq \left( \prod_{\ell=0}^{t-1} \alpha^{(\ell)} \beta^{(\ell)} \right) \left( \mathbf{B}^{t} \right)_{\sigma,\tau}.
\end{equation}
\end{lemma}

Thus, the bound on the Jacobian of the $\sigma$-feature with respect to the input $\tau$-feature depends on the $(\sigma,\tau)$-entry of the $t$-th matrix power $\mathbf{B}^t$, which reflects the number and strength of $t$-length paths from $\tau$ to $\sigma$ in the graph $\mathcal{G}(\mathcal{S}, \mathbf{B})$. Structural properties of $\mathcal{G}(\mathcal{S}, \mathbf{B})$ that lead to small values of $\left( \mathbf{B}^t \right)_{\sigma,\tau}$, such as bottlenecks or long distances between nodes, therefore contribute to the phenomenon of oversquashing, where the influence of distant entities is diminished. 

As demonstrated throughout this work, our result offers a systematic framework for extending other theoretical findings on oversquashing in graphs, which do not directly apply to simplicial complexes and similar relational structures. This includes the influential works by \citet{topping2022understanding}, \citet{pmlr-v202-di-giovanni23a}, and \citet{fesser2023mitigating}.
By leveraging our axiomatic framework, we derive principled extensions on the impact of local geometry (Section~\ref{subsec:impact-of-local-geometry}), depth (Section~\ref{subsec:impact-of-depth}), and hidden dimensions (Section~\ref{subsec:impact-of-hidden-dimensions}) in higher-order message passing, addressing settings where prior results are not applicable.
Additionally, this result offers a clear approach for deriving analogs of key quantities such as curvature (Definition~\ref{definition:motif-counts}), and can serve as a guide for future work.

\subsection{The Impact of Local Geometry}
\label{subsec:impact-of-local-geometry}

Lemma~\ref{lemma: jacobian bound higher order} shows that the entries of the matrix $\mathbf{B}^t$, which encode the number and strength of connections in a relational message passing scheme, control feature sensitivity. Prior works relate similar bounds to notions of discrete curvature for unweighted undirected graphs, such as balanced Forman curvature \citep{topping2022understanding}, Ollivier-Ricci curvature \citep{pmlr-v202-nguyen23c}, and augmented Forman curvature \citep{fesser2023mitigating}, via counting local motifs, such as triangles and squares, in the underlying graphs. Following this approach, we derive a result analogous to \citet[Proposition~3.4]{fesser2023mitigating}, introducing a motif-counting quantity inspired by the augmented Forman curvature, but adapted for the particular weighted directed graphs arising in our setting.

\begin{definition}
\label{definition:motif-counts}
Let $\mathcal{G} = (\mathcal{S}, \mathcal{E}, w)$ be a weighted directed graph with entities (nodes) $\mathcal{S}$, edges $\mathcal{E}$, and edge weights $w \colon \mathcal{E} \to \mathbb{R}_{\geq 0}$. For each entity $\tau \in \mathcal{S}$, define the \emph{weighted out-degree} $w_{\tau}^{\text{\normalfont out}} = \sum_{(\tau \to \sigma) \in \mathcal{E}} w_{\tau \to \sigma}$ and the \emph{weighted in-degree} $w_{\tau}^{\text{\normalfont in}} = \sum_{(\sigma \to \tau) \in \mathcal{E}} w_{\sigma \to \tau}$. For an edge $(\tau \to \sigma) \in \mathcal{E}$, define the \emph{weighted triangle count} $w_T = \sum_{\xi \in \mathcal{S}} w_{\tau \to \xi} \cdot w_{\xi \to \sigma}$ and the \emph{weighted quadrangle count} $w_F = \sum_{\xi_1, \xi_2 \in \mathcal{S}} w_{\tau \to \xi_1} \cdot w_{\xi_1 \to \xi_2} \cdot w_{\xi_2 \to \sigma}$. Then, the \emph{extended Forman curvature} of the edge $(\tau \to \sigma)$ is defined as:
\begin{equation}
\mathrm{EFC}_{\mathcal{G}}(\tau, \sigma) = 4 - w_{\tau}^{\text{\normalfont out}} - w_{\sigma}^{\text{\normalfont in}} + 3 w_T + 2 w_F.
\end{equation}
\end{definition}

We immediately get the following result, inspired by \citet[Proposition~4.4]{pmlr-v202-nguyen23c} and \citet[Proposition~3.4]{fesser2023mitigating}, and which is exemplary of results connecting sensitivity analysis to notions of discrete curvature. We provide the proof in Appendix~\ref{appendix-subsec-proofs-impact-of-local-geometry}.

\begin{proposition}
\label{prop:local-geometry}
Consider a relational structure $\mathcal{R} = (\mathcal{S}, R_1, \ldots, R_k)$ with update rule given by Equation~\ref{eq: update rule higher order} and satisfying Assumption~\ref{assump:bounded_jacobian}. Denote $\mathcal{G} = \mathcal{G}(\mathcal{S}, \mathbf{B})$. Then, for any $\sigma, \tau \in \mathcal{S}$ with an edge $(\tau \to \sigma) \in \mathcal{G}$, the following holds: 
\begin{equation}
\left\Vert \frac{\partial \mathbf{h}_{\sigma}^{(2)}}{\partial \mathbf{h}^{(0)}_{\tau}} \right\Vert_{1} \leq \frac{1}{3} \left( \prod_{\ell=0}^{1} \alpha^{(\ell)} \beta^{(\ell)} \right) \left[ \mathrm{EFC}_\mathcal{G}(\tau, \sigma) + w_{\tau}^{\text{\normalfont out}} + w_{\sigma}^{\text{\normalfont in}} - 4 \right].
\end{equation}
\end{proposition}

In principle, a similar result using balanced Forman curvature \citep[as in][Theorem~4]{topping2022understanding} is possible using our framework, and we leave that extension for future work. Connections to Ollivier-Ricci curvature are discussed in Appendix~\ref{appendix-remarks-subsec=impact-of-local-geometry}.

We present experimental analyses related to curvature on relational structures in Section~\ref{sec:experiments:curvature_distribution} (edge curvature distribution) and Appendices~\ref{appendix-graph-lifting-curvature} (edge curvature visualization) and ~\ref{appendix:graph_lift_wc} (weighted curvature). We further propose a relational extension of curvature-based rewiring techniques in Section~\ref{sec:rewiring} and empirically analyze the impact of relational rewiring using real-world and synthetic benchmarks in Sections~\ref{sec:experiments:realworld} and ~\ref{sec:experiments:synthetic}, respectively.

\subsection{The Impact of Depth}
\label{subsec:impact-of-depth}

To facilitate our analysis of depth, we make the following non-restrictive assumptions:

\begin{assumption}[Row-Normalized Shift Operators]
\label{assump:normalized_shift_operator}
Each shift operator $\mathbf{A}^{R_i}$ associated with relation $R_i$ is row-normalized, such that for all $\sigma \in \mathcal{S}$,
\begin{equation}
\sum_{\boldsymbol{\xi} \in \mathcal{S}^{n_i - 1}} A_{\sigma, \boldsymbol{\xi}}^{R_i} = \begin{cases}
1, & \text{if } \sum_{\boldsymbol{\xi}} A_{\sigma, \boldsymbol{\xi}}^{R_i} \ne 0, \\
0, & \text{if } \sum_{\boldsymbol{\xi}} A_{\sigma, \boldsymbol{\xi}}^{R_i} = 0.
\end{cases}
\end{equation}
\end{assumption}

\begin{assumption}[Bounded $\alpha^{(\ell)}$ and $\beta^{(\ell)}$]
\label{assump:bounded_alpha_beta}
There exist constants $\alpha_{\text{\normalfont max}} > 0$ and $\beta_{\text{\normalfont max}} > 0$ such that for all layers $\ell$, $\alpha^{(\ell)} \leq \alpha_{\text{\normalfont max}}$ and $\beta^{(\ell)} \leq \beta_{\text{\normalfont max}}$.
\end{assumption}

We now present our main result on the impact of depth in relational message passing, extending a previous result by \citet[Theorem~4.1]{pmlr-v202-di-giovanni23a} to our setting. 
We provide the proof in Appendix~\ref{appendix-subsec-proofs-impact-of-depth}. 
By \emph{the combinatorial distance} from $\tau$ to $\sigma$ in the graph $\mathcal{G}(\mathcal{S}, \tilde{\mathbf{A}})$, we mean the smallest number of edges in a directed path from $\tau$ to $\sigma$ in the graph. Similarly, by \emph{combinatorial length} of a directed path, we mean the number of edges in the path. 

\begin{theorem}[Impact of Depth on Relational Message Passing]
\label{theorem: impact of depth}
Consider a relational structure $\mathcal{R} = (\mathcal{S}, R_1, \ldots, R_k)$ with update rule given by Equation~\ref{eq: update rule higher order} and satisfying Assumptions~\ref{assump:bounded_jacobian},~\ref{assump:normalized_shift_operator}, and~\ref{assump:bounded_alpha_beta}. 
Let $\sigma, \tau \in \mathcal{S}$ be entities such that the combinatorial distance from $\tau$ to $\sigma$ in the graph $\mathcal{G}(\mathcal{S}, \tilde{\mathbf{A}})$ is $r$. Denote by $\omega_\ell(\sigma,\tau)$ the number of directed paths from $\tau$ to $\sigma$ of combinatorial length at most $\ell$ in $\mathcal{G}(\mathcal{S}, \tilde{\mathbf{A}})$. Then, for any $0 \leq m < r$, there exists a constant $C > 0$, depending only on $\alpha_{\text{\normalfont max}}$, $\beta_{\text{\normalfont max}}$, $k$, and $m$, but not on $r$ nor the specific relations in $\mathcal{R}$, such that
\begin{equation}
\left\|\frac{\partial \mathbf{h}_\sigma^{(r+m)}}{\partial \mathbf{h}_\tau^{(0)}}\right\|_1 \leq C \omega_{r+m}(\sigma,\tau) (2 \alpha_{\text{\normalfont max}} \beta_{\text{\normalfont max}} M)^r,
\end{equation}
where $M = \max_{\sigma,\tau} \tilde{\mathbf{A}}_{\sigma,\tau}$.
\end{theorem}

This result indicates that the sensitivity can decay exponentially with depth when $M < 1/(2 \alpha_{\text{max}} \beta_{\text{max}})$, particularly when the number of walks $\omega_t(\sigma,\tau)$ is limited by the structure of $\mathcal{G}(\mathcal{S}, \tilde{\mathbf{A}})$. Such exponential decay is a characteristic of the oversquashing phenomenon, where information from distant nodes becomes increasingly compressed, reducing its influence on the output.

We present experimental validation of this result in Section~\ref{sec:experiments:synthetic}.

\subsection{The Impact of Hidden Dimensions}
\label{subsec:impact-of-hidden-dimensions}

In situations where the Lipschitz constants of the message and update functions from Assumption~\ref{assump:bounded_jacobian} are affected by hyperparameters, such as the widths of neural networks 
implementing
said functions, one can have $\beta_{i}^{(\ell)} = O(p_{i,\ell})$ and $\alpha^{(\ell)} = O(p_{\ell+1})$. This is the case, for instance, when the message and update functions are
shallow neural networks (see 
Appendix~\ref{appendix-remarks-subsec-impact-of-hidden-dimensions} and 
Appendix~\ref{appendix-subsec-proofs-impact-of-hidden-dimensions}).

Writing $p_\ell^\prime = \max_i p_{i,\ell}$ and substituting $\beta^{(\ell)} = O(p_\ell^\prime)$ and $\alpha^{(\ell)} = O(p_{\ell+1})$ into the bound from Lemma~\ref{lemma: jacobian bound higher order}, one gets:
\begin{equation}
\left\| \frac{\partial \mathbf{h}_\sigma^{(t)}}{\partial \mathbf{h}_\tau^{(0)}} \right\|_{1} \leq C \cdot \left( \prod_{\ell=0}^{t-1} p_{\ell+1} \cdot p_\ell^\prime \right) \left( \mathbf{B}^{t} \right)_{\sigma,\tau},
\end{equation}
where $C$ is a constant independent of the layer widths, $p_\ell'$ is the maximum dimension of the message vectors at layer $\ell$, and $p_{\ell+1}$ is the output dimension of the update function at layer $\ell$.

This implies that low hidden dimensions in the message and update functions contribute to a low sensitivity upper bound, which can exacerbate the oversquashing problem. Increasing the hidden dimensions will raise the upper bound,
which can help improve
the model’s ability to propagate information effectively and enhance performance on tasks. However, increasing the hidden dimensions risks overfitting due to the increased model complexity \citep{bartlett2017spectrally}.

We present experimental validation of this result in Section~\ref{sec:experiments:synthetic}.

\takeawaymessage{2 (Theoretical)}{
By reformulating higher order structures as relational structures, key results on oversquashing in graph neural networks extend to relational message passing schemes
through the \emph{aggregated influence matrix} and the \emph{influence graph}. This conceptual framework enables analysis of the impact of local geometry, depth, and hidden dimensions in relational message passing schemes, just as in graph neural networks.
}

\section{Rewiring Heuristics for Relational Structures}
\label{sec:rewiring}

Inspired by First-Order Spectral Rewiring (FoSR) \citep{karhadkar2023fosr}, we propose a rewiring heuristic that integrates additional connections into a relational structure without altering its original connections. To capture the overall connectivity of a relational structure, we define the collapsed adjacency matrix, which counts the number of direct connections between entities.

\begin{definition}[Collapsed Adjacency Matrix]
\label{def:collapsed_adjacency_matrix}
Given a relational structure $\mathcal{R} = (\mathcal{S}, R_1, \ldots, R_k)$, the \emph{collapsed adjacency matrix} $\mathbf{A}^\text{\normalfont col}$ for the structure $\mathcal{R}$ is defined by:
\begin{equation}
\mathbf{A}_{\sigma, \tau}^\text{\normalfont col} = \sum_{i=1}^k \sum_{\boldsymbol{\xi} \in \mathcal{S}^{n_i - 1}} \mathbf{1}_{\{(\sigma, \boldsymbol{\xi}) \in R_i, \tau = \xi_j \text{ for some } j \in \{1, \ldots, n_i\}\}},\quad \sigma, \tau\in \mathcal{S}.
\end{equation}
\end{definition}
This matrix captures direct connections between entities through any relation, effectively collapsing the relational structure into a graph. Our proposed relational rewiring algorithm is as follows.

\begin{algorithm}[htbp]
\caption{Relational Rewiring Algorithm}
\label{alg:rewiring}
\begin{algorithmic}[1]
\Require Relational structure $\mathcal{R} = (\mathcal{S}, R_1, \ldots, R_k)$; graph rewiring algorithm \textsc{RewireAlgo}
\State Construct the collapsed adjacency matrix $\mathbf{A}^\text{col}$ (Definition~\ref{def:collapsed_adjacency_matrix})
\State Build the graph $\mathcal{G}^\text{col} = (\mathcal{S}, \mathbf{A}^\text{col})$
\State Apply \textsc{RewireAlgo} to $\mathcal{G}^\text{col}$ to obtain additional edges $E_{\text{new}}$
\State Define a new relation $R_{k+1} = E_{\text{new}}$
\State Update the relational structure: $\mathcal{R}^\prime = (\mathcal{S}, R_1, \ldots, R_k, R_{k+1})$
\end{algorithmic}
\end{algorithm}
Adding new connections ($E_{\text{new}}$) without removing existing ones improves the model capacity to capture long-range dependencies while preserving the original relational structure. We experimentally analyze the impact of relational rewiring using real-world and synthetic benchmarks in Sections~\ref{sec:experiments:realworld} and ~\ref{sec:experiments:synthetic}, respectively. Future work could explore rewiring algorithms that remove or reclassify edges, such as spectral pruning \citep{jamadandi2024spectral}, either by re-labeling the edges as new relations or by deletion. We report preliminary empirical results on pruning in Appendix~\ref{appendix:supplementary-analyses}.

\takeawaymessage{3 (Practical)}{
Graph rewiring techniques for improving information flow and mitigating oversquashing can be adapted to relational structures. This approach improves long-range connectivity and enhances message propagation while maintaining the integrity of the original connections.
}

\section{Experiments and Results}
\label{sec:experiments}

\subsection{Real-World Benchmark: Graph Classification}
\label{sec:experiments:realworld}

We run an empirical analysis with real-world datasets to compare the performance of different graph, relational graph, and simplicial message passing models, and the impact of relational rewiring on said models. We provide more details in Appendix~\ref{appendix:experiments}.

\textbf{Task and Datasets.} We use graph classification tasks ENZYMES, IMDB-B, MUTAG, NCI1, and PROTEINS from the TUDataset \citep{morris2020} for evaluation.

\textbf{Graph Lifting.} For the topological message passing models, graphs are separately treated as complexes with only graph nodes ($0$-dimensional simplices) and upper adjacencies, and also lifted into clique and ring complexes (see appendix).

\textbf{Models.} We evaluate three types of models: a) \textit{Graph message passing models}: SGC \citep{wu2019simplifying}, GCN \citep{kipf2017semisupervised}, and GIN \citep{xu2018how}; b) \textit{Relational graph message passing models}: RGCN \citep{schlichtkrull2018modeling} and RGIN; c) \textit{Topological message passing models}: SIN \citep{pmlr-v139-bodnar21a}, CIN \citep{bodnar2021weisfeiler}, and CIN++ \citep{giusti2024cinplusplus}.

\textbf{Relational Rewiring.} We apply relational rewiring for 40 iterations (Section~\ref{sec:rewiring}) using three choices for \textsc{RewireAlgo}: SDRF \citep{topping2022understanding}, FoSR \citep{karhadkar2023fosr}, and AFRC \citep{fesser2023mitigating}. Due to computational constraints, we run the three choices on all datasets except IMDB with Clique graph lifting, where we only run FoSR. We use fixed, dataset- and model-agnostic hyperparameters, diverging from prior work where hyperparameter sweeps are carried out. It is important to note that hyperparameter tuning can significantly impact performance on downstream tasks, as highlighted by, e.g., \citet{tori2024effectiveness}.

\textbf{Training.} Models are trained for up to 500 epochs, with early stopping and learning rate decay based on a validation set. Additional details can be found in Appendix~\ref{appendix:experiments}. Results are reported as \emph{mean $\pm$ standard error} over 10 trials.

\textbf{Results.} Table \ref{tab:test_acc_main} shows test accuracies for the TUDataset experiments.
Rewiring generally boosts performance for base graphs across models and datasets, and the impact of rewiring with our dataset- and model-agnostic choice of hyperparameters varies across datasets, with relational and topological models performance responding to rewiring similarly to graph models.

\begin{table}[htbp]
\centering
\resizebox{\textwidth}{!}{
\begin{tabular}{llllllllllll}
\toprule
 &  & \multicolumn{2}{c}{ENZYMES} & \multicolumn{2}{c}{IMDB-B} & \multicolumn{2}{c}{MUTAG} & \multicolumn{2}{c}{NCI1} & \multicolumn{2}{c}{PROTEINS} \\
Lift & Model & No Rew. & Best Rew. & No Rew. & FoSR & No Rew. & Best Rew. & No Rew. & Best Rew. & No Rew. & Best Rew. \\
\midrule
\multirow[c]{8}{*}{None}
    & SGC & $18.3 \pm 1.2$ & \increased{$21.5 \pm 1.6$} & $49.5 \pm 1.5$ & \increased{$50.0 \pm 1.8$} & $64.5 \pm 5.8$ & \increased{$70.0 \pm 2.6$} & $55.2 \pm 1.0$ & \decreased{$54.4 \pm 0.7$} & $62.2 \pm 1.4$ & \increased{$65.0 \pm 1.5$} \\
    & GCN & $32.2 \pm 2.0$ & \decreased{$30.7 \pm 1.5$} & $49.1 \pm 1.4$ & \decreased{$47.9 \pm 1.0$} & $71.0 \pm 3.8$ & \increased{$83.0 \pm 1.5$} & $48.3 \pm 0.6$ & \increased{$49.1 \pm 0.8$} & \cellcolor{gold}$73.1 \pm 1.2$ & \cellcolor{gold}\increased{$75.8 \pm 1.7$} \\
    & GIN & $47.2 \pm 1.9$ & \increased{$50.0 \pm 2.7$} & \cellcolor{gold}$71.7 \pm 1.5$ & \decreased{$67.1 \pm 1.5$} & $83.0 \pm 3.1$ & \increased{$88.0 \pm 2.4$} & $77.2 \pm 0.5$ & \increased{$77.9 \pm 0.5$} & $70.6 \pm 1.4$ & \increased{$72.2 \pm 0.8$} \\
    \cline{2-12}
    & RGCN & $33.8 \pm 1.6$ & \increased{$42.5 \pm 1.3$} & $47.6 \pm 1.4$ & \increased{$68.0 \pm 1.3$} & $72.5 \pm 2.5$ & \increased{$83.5 \pm 1.8$} & $53.2 \pm 0.7$ & \increased{$63.4 \pm 1.1$} & $71.9 \pm 1.6$ & \cellcolor{silver}\increased{$75.6 \pm 1.2$} \\
    & RGIN & $46.8 \pm 1.8$ & \increased{$49.8 \pm 2.0$} & $69.6 \pm 1.6$ & \decreased{$48.9 \pm 2.9$} & $81.5 \pm 1.7$ & \increased{$85.5 \pm 2.0$} & $76.8 \pm 1.1$ & \increased{$77.0 \pm 0.7$} & $70.8 \pm 1.2$ & \increased{$72.4 \pm 1.4$} \\
    \cline{2-12}
    & SIN & $47.5 \pm 2.3$ & \decreased{$46.8 \pm 2.1$} & $70.0 \pm 1.4$ & \decreased{$ 63.0 \pm 2.7$} & $88.5 \pm 3.0$ & \decreased{$85.5 \pm 1.7$} & $77.0 \pm 0.6$ & \decreased{$76.4 \pm 0.4$} & $70.2 \pm 1.3$ & \increased{$73.2 \pm 1.5$} \\
    & CIN & $50.0 \pm 1.9$ & \decreased{$49.0 \pm 2.0$} & $58.1 \pm 4.0$ & \increased{$58.4 \pm 2.7$} & $86.5 \pm 1.8$ & \increased{$87.0 \pm 2.4$} & $51.4 \pm 2.5$ & \increased{$66.2 \pm 2.0$} & $70.7 \pm 1.0$ & \increased{$71.0 \pm 1.4$} \\
    & CIN++ & $48.5 \pm 1.9$ & \increased{$51.0 \pm 1.5$} & $66.6 \pm 3.7$ & \decreased{$56.0 \pm 3.9$} & $85.0 \pm 3.4$ & \cellcolor{silver}\increased{$91.0 \pm 2.3$} & $60.8 \pm 3.8$ & \increased{$64.8 \pm 3.1$} & $67.9 \pm 1.9$ & \increased{$71.4 \pm 1.4$} \\
\midrule
\multirow[c]{8}{*}{Clique}
    & SGC & $14.5 \pm 1.4$ & \increased{$16.8 \pm 0.9$} & $48.7 \pm 2.2$ & \decreased{$47.8 \pm 1.6$} & $70.0 \pm 3.3$ & \decreased{$69.5 \pm 2.6$} & $50.0 \pm 1.3$ & \increased{$56.8 \pm 0.8$} & $59.9 \pm 1.8$ & \decreased{$59.1 \pm 1.4$} \\
    & GCN & $30.7 \pm 1.2$ & \decreased{$30.2 \pm 2.4$} & $64.0 \pm 3.1$ & \increased{$65.5 \pm 3.1$} & $67.0 \pm 3.5$ & \increased{$81.5 \pm 2.9$} & $48.4 \pm 0.4$ & \increased{$49.6 \pm 0.6$} & $69.9 \pm 0.6$ & \increased{$75.0 \pm 1.4$} \\
    & GIN & $44.0 \pm 1.7$ & \increased{$48.5 \pm 2.2$} & $69.1 \pm 1.2$ & \cellcolor{silver}\increased{$70.8 \pm 1.1$} & $83.0 \pm 2.8$ & \decreased{$82.5 \pm 2.6$} & $78.8 \pm 0.7$ & \decreased{$78.2 \pm 0.6$} & $68.7 \pm 1.4$ & \increased{$72.8 \pm 1.2$} \\
    \cline{2-12}
    & RGCN & $48.8 \pm 1.2$ & \decreased{$45.2 \pm 1.5$} & $71.0 \pm 1.0$ & \decreased{$69.7 \pm 1.5$} & $79.5 \pm 1.7$ & \increased{$81.5 \pm 3.8$} & $72.9 \pm 0.8$ & \increased{$75.0 \pm 0.9$} & \cellcolor{silver}$72.4 \pm 1.6$ & \increased{$74.2 \pm 1.2$} \\
    & RGIN & \cellcolor{silver}$50.8 \pm 1.5$ & \cellcolor{gold}\increased{$55.8 \pm 2.5$} & \cellcolor{silver}$71.6 \pm 0.9$ & \decreased{$69.0 \pm 1.4$} & $86.0 \pm 2.3$ & \decreased{$85.0 \pm 2.4$} & \cellcolor{gold}$79.2 \pm 0.6$ & \cellcolor{silver}\increased{$79.5 \pm 0.4$} & $71.5 \pm 1.5$ & \increased{$71.8 \pm 1.7$} \\
    \cline{2-12}
    & SIN & \cellcolor{gold}$51.0 \pm 2.4$ & \decreased{$46.5 \pm 1.2$} & $53.0 \pm 1.9$ & \increased{$64.0 \pm 2.3$} & $87.0 \pm 3.2$ & \decreased{$83.5 \pm 1.7$} & $76.6 \pm 1.3$ & \decreased{$75.4 \pm 0.7$} & $66.9 \pm 1.3$ & \increased{$70.4 \pm 1.2$} \\
    & CIN & $49.8 \pm 1.9$ & \decreased{$46.7 \pm 1.3$} & $52.6 \pm 2.4$ & \increased{$68.1 \pm 1.6$} & $85.5 \pm 2.8$ & \increased{$86.5 \pm 2.6$} & $51.8 \pm 2.3$ & \increased{$72.5 \pm 0.8$} & $70.7 \pm 1.2$ & \decreased{$70.3 \pm 0.8$} \\
    & CIN++ & $50.5 \pm 2.1$ & \cellcolor{silver}\increased{$52.7 \pm 1.6$} & $62.8 \pm 3.8$ & \increased{$64.7 \pm 1.5$} & \cellcolor{silver}$90.5 \pm 2.2$ & \decreased{$84.5 \pm 3.3$} & $61.5 \pm 4.6$ & \increased{$76.8 \pm 0.4$} & $68.3 \pm 1.3$ & \increased{$71.9 \pm 1.0$} \\
\midrule
\multirow[c]{8}{*}{Ring}
    & SGC & $16.5 \pm 1.6$ & \increased{$19.3 \pm 1.2$} & $50.1 \pm 1.9$ & \decreased{$49.9 \pm 1.9$} & $65.5 \pm 3.6$ & \increased{$75.0\pm5.7$} & $51.5\pm1.2$ & \decreased{$51.4\pm0.4$} & $44.8\pm2.3$ & \increased{$49.3\pm3.6$} \\
    & GCN & $34.8\pm1.3$ & \decreased{$32.0\pm1.4$} & $46.9\pm1.4$ & \increased{$48.0\pm1.2$} & $72.0\pm2.7$ & \increased{$77.5\pm2.4$} & $49.3\pm0.9$ & \increased{$49.4 \pm 0.6$} & $72.2\pm1.3$ & \increased{$72.7 \pm 1.2$} \\
    & GIN & $46.7\pm2.4$ & \increased{$47.0\pm1.6$} & $70.1\pm1.7$ & \cellcolor{gold}\increased{$73.1\pm1.1$} & $88.0\pm2.1$ & \increased{$89.0\pm1.9$} & \cellcolor{silver}$78.9 \pm 0.6$ & \decreased{$77.5\pm0.9$} & $69.8 \pm 1.4$ & \increased{$72.1 \pm 0.9$} \\
    \cline{2-12}
    & RGCN & $35.2\pm1.7$ & \increased{$45.7 \pm 1.5$} & $71.1 \pm 1.4$ & \decreased{$70.0\pm1.6$} & $83.5\pm2.7$ & \increased{$84.0\pm2.1$} & $73.9\pm0.5$ & \decreased{$73.5\pm0.5$} & $70.7\pm1.6$ & \increased{$71.3\pm1.2$} \\
    & RGIN & $45.3\pm1.3$ & \increased{$49.2\pm1.5$} & $68.6\pm1.2$ & \decreased{$67.2\pm1.8$} & $87.0\pm2.9$ & \increased{$87.5\pm2.4$} & $78.4 \pm 0.7$ & \cellcolor{gold}\increased{$79.8 \pm 0.7$} & $68.8\pm1.5$ & \increased{$71.3 \pm 1.5$} \\
    \cline{2-12}
    & SIN & $40.3\pm2.2$ & \increased{$48.0\pm2.0$} & $50.6\pm1.9$ & \increased{$60.9\pm2.1$} & $85.0 \pm 2.1$ & \increased{$88.5 \pm 2.5$} & $80.0\pm0.8$ & \decreased{$79.1\pm0.9$} & $70.6\pm1.1$ & \increased{$72.1\pm0.7$} \\
    & CIN & $47.5\pm2.0$ & \increased{$49.5\pm2.0$} & $48.6\pm1.6$ & \increased{$66.1 \pm 2.0$} & \cellcolor{gold}$93.5\pm2.1$ & \cellcolor{gold}\increased{$95.0\pm1.3$} & $51.6 \pm 3.2$ & \increased{$76.5\pm0.5$} & $68.7\pm1.4$ & \decreased{$68.5\pm1.6$} \\
    & CIN++ & $47.5\pm1.7$ & \decreased{$46.3\pm1.8$} & $66.0\pm1.4$ & \increased{$67.8\pm1.3$} & $85.5\pm2.0$ & \increased{$90.0\pm2.7$} & $56.8\pm4.5$ & \increased{$76.0 \pm 0.6$} & $68.1 \pm 1.2$ & \increased{$70.1 \pm 1.2$} \\
\bottomrule
\end{tabular}
}
\caption{Test accuracy for TUDataset experiments. Each value is presented as the mean $\pm$ standard error across ten trials. The best-performing result for each dataset is highlighted in gold, while the second-best is in silver. The results after rewiring are shown with green text if the mean increased and red text if the mean decreased.}
\label{tab:test_acc_main}

\end{table}

\subsection{Synthetic Benchmark: \textsc{RingTransfer}}
\label{sec:experiments:synthetic}

We confirm the theoretical results from Section~\ref{sec:sensitivity_analysis} using the \textsc{RingTransfer} benchmark, a graph feature transfer task designed to tease out the effect of long-range dependencies in message-passing models using rings of growing size. We follow the experimental setup of \citet{karhadkar2023fosr} and \citet{pmlr-v202-di-giovanni23a}, and provide more details in Appendix~\ref{appendix:experiments:ringtransfer}. We test the impact of neural network hidden dimensions (Section~\ref{subsec:impact-of-hidden-dimensions}), relational structure depth (Section~\ref{subsec:impact-of-depth}), and relational structure local geometry (Sections~\ref{subsec:impact-of-hidden-dimensions} and \ref{sec:rewiring}) on task performance by varying the hidden dimensions, ring sizes, and rewiring iterations. The results, consistent with the theory, demonstrate that increasing network hidden dimensions improves performance up to a point, after which it declines, potentially due to overfitting. Larger ring sizes lead to performance deterioration, as the effects of long-range dependencies and bottlenecks start to take over. At the same time, rewiring improves performance by facilitating communication between distant nodes and mitigating oversquashing. As illustrated in Figure~\ref{fig:ring_transfer}, message passing on graphs and simplicial complexes demonstrate similar trends, consistent with our theoretical predictions.

\begin{figure}[htbp]
    \centering
    \includegraphics[width=.4\textwidth]{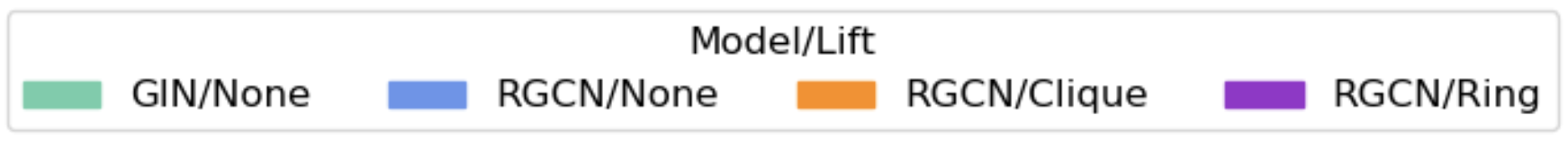}\\
    \begin{subfigure}[b]{0.32\textwidth}
        \centering
        \includegraphics[width=\textwidth]{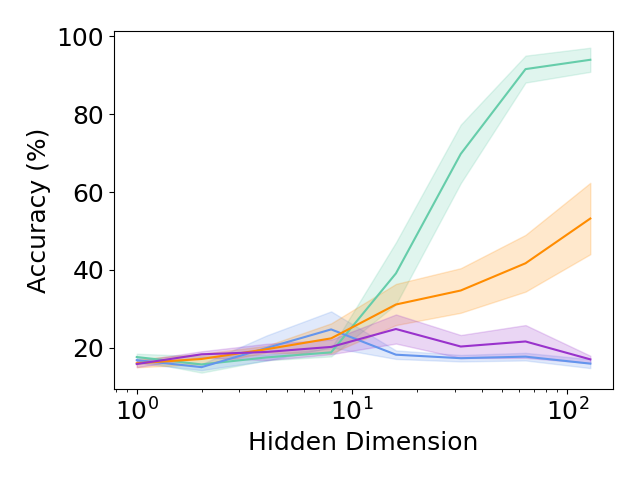}
        \caption{Hidden Dimensions}
    \end{subfigure}
    \begin{subfigure}[b]{0.32\textwidth}
        \centering
        \includegraphics[width=\textwidth]{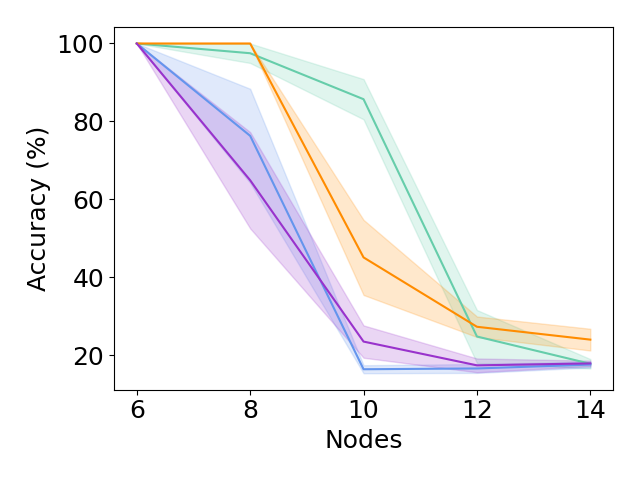}
        \caption{Size}
    \end{subfigure}
    \begin{subfigure}[b]{0.32\textwidth}
        \centering
        \includegraphics[width=\textwidth]{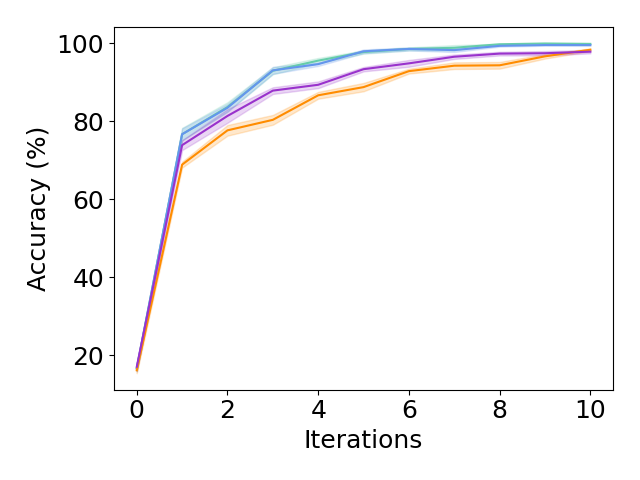}
        \caption{Rewiring}
    \end{subfigure}
    \caption{Performance on \textsc{RingTransfer} obtained by varying model hidden dimensions (left), ring size (middle), and number of rewiring iterations (right). 
    }
    \label{fig:ring_transfer}
\end{figure}

\subsection{Additional Experiments and Analyses}

We report additional analyses in Appendix~\ref{appendix-graph-lifting-curvature} and Appendix~\ref{appendix:graph_lift_wc}. There, we visualize the curvature of relational structures for dumbbell graphs and their corresponding clique complexes. We also reports a statistically significant linear relationship between the weighted curvature of graphs and their lifted clique complexes. These 
interesting patterns merit further investigation. 
We also present the following additional experiments: (1) neighbors match for path of cliques and tree datasets in Appendix~\ref{appendix:neighbors-match}, (2) graph regression for ZINC in Appendix~\ref{appendix:zinc}, (3) node classification for CORNELL, WISCONSIN, TEXAS, CORA, and CITESEER in Appendix~\ref{appendix:node-classification}, (4) simplex pruning on the MUTAG dataset in Appendix~\ref{apdx:prune}, and (5) full TUDataset results in Appendix~\ref{appendix:tudataset-full}.

\section{Discussion and Conclusions}
\label{sec:discussion-and-conclusions}

This work addresses pressing questions about oversquashing in topological networks and higher-order generalizations of rewiring algorithms raised by the TDL community (Questions 2 and 9 of \citealp{pmlr-v235-papamarkou24a}).
We introduce a theoretical framework for unifying graph and topological message passing via relational structures, extending key graph-theoretic results on oversquashing and rewiring strategies to higher-order networks such as simplicial complexes via \emph{influence graphs} that capture the aggregated message passing dynamical structure on relational structures.
Our approach applies broadly to other message-passing schemes, including relational GNNs, high-order GNNs, and CW networks, providing a foundation for future theoretical and empirical research. Empirical results on real-world datasets show that simplicial networks respond to rewiring similarly to graph networks, and synthetic benchmarks further confirm our theoretical findings.

Certain aspects are worthy of further investigation. 
In particular, we compare message passing on graphs and their clique complexes through proxies (e.g., performance on tasks), as the significant differences in size and structure make direct empirical comparisons, e.g., of curvature, less theoretically rigorous. While we observe statistically significant patterns when comparing weighted curvatures, further theoretical and empirical investigation is needed. Furthermore, the rewiring algorithms we applied our relational rewiring heuristic to were not originally designed with weighted directed influence graphs in mind.  
Potentially, further improvements could be obtained by implementing algorithms specifically tailored for rewiring weighted directed graphs. 

For future work, exploring global geometric properties of relational structures, studying oversmoothing, and empirically analyzing more relational message-passing schemes are promising directions. 
Developing theoretical tools and tailored rewiring heuristics for weighted directed graphs
will be crucial, as will be tools for direct comparisons of message-passing across different relational structures.
Furthermore, systematically and empirically assessing our framework's higher-order extensions of more state-of-the-art (SoTA) graph rewiring solutions is essential.
By unifying topological message passing into message passing on relational structures and generalizing graph-based analysis to this setting, we hope that the present work can aid in both rigorous analysis and direct comparison between different higher-order message passing schemes.

Lastly, for practitioners, we recommend topological message passing as yet another relational learning tool with relational rewiring as a preprocessing step.

\paragraph{Reproducibility Statement}
The code for replicating our experiments is available at
\url{https://github.com/chapman20j/Simplicial-Oversquashing}.
Experimental settings and implementation details are described in Section~\ref{sec:experiments}, and Appendices~\ref{appendix:experiments} and ~\ref{appendix:implementation-details}.

\subsubsection*{Acknowledgments}
GM and JC have been supported by NSF CCF-2212520. GM also acknowledges support from NSF DMS-2145630, DFG SPP 2298 project 464109215, 
and BMBF in DAAD project 57616814. 

\bibliography{references}

\begin{thebibliography}{87}
\providecommand{\natexlab}[1]{#1}
\providecommand{\url}[1]{\texttt{#1}}
\expandafter\ifx\csname urlstyle\endcsname\relax
  \providecommand{\doi}[1]{doi: #1}\else
  \providecommand{\doi}{doi: \begingroup \urlstyle{rm}\Url}\fi

\bibitem[Alon \& Yahav(2021)Alon and Yahav]{alon2021on}
Uri Alon and Eran Yahav.
\newblock On the bottleneck of graph neural networks and its practical implications.
\newblock In \emph{International Conference on Learning Representations}, 2021.
\newblock URL \url{https://openreview.net/forum?id=i80OPhOCVH2}.

\bibitem[Antelmi et~al.(2023)Antelmi, Cordasco, Polato, Scarano, Spagnuolo, and Yang]{antelmi2023survey}
Alessia Antelmi, Gennaro Cordasco, Mirko Polato, Vittorio Scarano, Carmine Spagnuolo, and Dingqi Yang.
\newblock A survey on hypergraph representation learning.
\newblock \emph{ACM Computing Surveys}, 56\penalty0 (1):\penalty0 1--38, 2023.

\bibitem[Attali et~al.(2024)Attali, Buscaldi, and Pernelle]{attalidelaunay}
Hugo Attali, Davide Buscaldi, and Nathalie Pernelle.
\newblock Delaunay graph: Addressing over-squashing and over-smoothing using delaunay triangulation.
\newblock In \emph{Forty-first International Conference on Machine Learning}, 2024.

\bibitem[Banerjee et~al.(2022)Banerjee, Karhadkar, Wang, Alon, and Montúfar]{9929363}
Pradeep~Kr. Banerjee, Kedar Karhadkar, Yu~Guang Wang, Uri Alon, and Guido Montúfar.
\newblock Oversquashing in {GNNs} through the lens of information contraction and graph expansion.
\newblock In \emph{2022 58th Annual Allerton Conference on Communication, Control, and Computing (Allerton)}, pp.\  1--8, 2022.
\newblock \doi{10.1109/Allerton49937.2022.9929363}.

\bibitem[Barbero et~al.(2022)Barbero, Bodnar, de~Oc{\'a}riz~Borde, and Lio]{barbero2022sheaf}
Federico Barbero, Cristian Bodnar, Haitz~S{\'a}ez de~Oc{\'a}riz~Borde, and Pietro Lio.
\newblock Sheaf attention networks.
\newblock In \emph{NeurIPS 2022 Workshop on Symmetry and Geometry in Neural Representations}, 2022.
\newblock URL \url{https://openreview.net/forum?id=LIDvgVjpkZr}.

\bibitem[Bartlett et~al.(2017)Bartlett, Foster, and Telgarsky]{bartlett2017spectrally}
Peter~L Bartlett, Dylan~J Foster, and Matus~J Telgarsky.
\newblock Spectrally-normalized margin bounds for neural networks.
\newblock In \emph{Advances in Neural Information Processing Systems}, volume~30. Curran Associates, Inc., 2017.
\newblock URL \url{https://proceedings.neurips.cc/paper_files/paper/2017/file/b22b257ad0519d4500539da3c8bcf4dd-Paper.pdf}.

\bibitem[Black et~al.(2023)Black, Wan, Nayyeri, and Wang]{pmlr-v202-black23a}
Mitchell Black, Zhengchao Wan, Amir Nayyeri, and Yusu Wang.
\newblock Understanding oversquashing in {GNN}s through the lens of effective resistance.
\newblock In \emph{Proceedings of the 40th International Conference on Machine Learning}, volume 202 of \emph{Proceedings of Machine Learning Research}, pp.\  2528--2547. PMLR, 2023.
\newblock URL \url{https://proceedings.mlr.press/v202/black23a.html}.

\bibitem[Bodnar et~al.(2021{\natexlab{a}})Bodnar, Frasca, Otter, Wang, Li{\`o}, Mont\'ufar, and Bronstein]{bodnar2021weisfeiler}
Cristian Bodnar, Fabrizio Frasca, Nina Otter, Yu~Guang Wang, Pietro Li{\`o}, Guido Mont\'ufar, and Michael~M. Bronstein.
\newblock Weisfeiler and {L}ehman go cellular: {CW} networks.
\newblock In \emph{Advances in Neural Information Processing Systems}, 2021{\natexlab{a}}.
\newblock URL \url{https://openreview.net/forum?id=uVPZCMVtsSG}.

\bibitem[Bodnar et~al.(2021{\natexlab{b}})Bodnar, Frasca, Wang, Otter, Mont\'ufar, Li{\'o}, and Bronstein]{pmlr-v139-bodnar21a}
Cristian Bodnar, Fabrizio Frasca, Yuguang Wang, Nina Otter, Guido Mont\'ufar, Pietro Li{\'o}, and Michael Bronstein.
\newblock Weisfeiler and {L}ehman go topological: Message passing simplicial networks.
\newblock In \emph{Proceedings of the 38th International Conference on Machine Learning}, volume 139 of \emph{Proceedings of Machine Learning Research}, pp.\  1026--1037. PMLR, 2021{\natexlab{b}}.
\newblock URL \url{https://proceedings.mlr.press/v139/bodnar21a.html}.

\bibitem[Bodnar et~al.(2022)Bodnar, Giovanni, Chamberlain, Lio, and Bronstein]{bodnar2022neural}
Cristian Bodnar, Francesco~Di Giovanni, Benjamin~Paul Chamberlain, Pietro Lio, and Michael~M. Bronstein.
\newblock Neural sheaf diffusion: A topological perspective on heterophily and oversmoothing in {GNN}s.
\newblock In Alice~H. Oh, Alekh Agarwal, Danielle Belgrave, and Kyunghyun Cho (eds.), \emph{Advances in Neural Information Processing Systems}, 2022.
\newblock URL \url{https://openreview.net/forum?id=vbPsD-BhOZ}.

\bibitem[Bronstein et~al.(2021)Bronstein, Bruna, Cohen, and Veličković]{bronstein2021geometric}
Michael~M. Bronstein, Joan Bruna, Taco Cohen, and Petar Veličković.
\newblock Geometric deep learning: Grids, groups, graphs, geodesics, and gauges, 2021.
\newblock URL \url{https://arxiv.org/abs/2104.13478}.

\bibitem[Chamberlain et~al.(2021)Chamberlain, Rowbottom, Eynard, Di~Giovanni, Dong, and Bronstein]{chamberlain2021beltrami}
Benjamin Chamberlain, James Rowbottom, Davide Eynard, Francesco Di~Giovanni, Xiaowen Dong, and Michael Bronstein.
\newblock Beltrami flow and neural diffusion on graphs.
\newblock In \emph{Advances in Neural Information Processing Systems}, volume~34, pp.\  1594--1609. Curran Associates, Inc., 2021.
\newblock URL \url{https://proceedings.neurips.cc/paper_files/paper/2021/file/0cbed40c0d920b94126eaf5e707be1f5-Paper.pdf}.

\bibitem[Chen et~al.(2019)Chen, Villar, Chen, and Bruna]{chen2019equivalence}
Zhengdao Chen, Soledad Villar, Lei Chen, and Joan Bruna.
\newblock On the equivalence between graph isomorphism testing and function approximation with {GNNs}.
\newblock In \emph{Advances in Neural Information Processing Systems}, volume~32. Curran Associates, Inc., 2019.
\newblock URL \url{https://proceedings.neurips.cc/paper_files/paper/2019/file/71ee911dd06428a96c143a0b135041a4-Paper.pdf}.

\bibitem[Dasoulas et~al.(2021)Dasoulas, Lutzeyer, and Vazirgiannis]{dasoulas2021learning}
George Dasoulas, Johannes~F. Lutzeyer, and Michalis Vazirgiannis.
\newblock Learning parametrised graph shift operators.
\newblock In \emph{International Conference on Learning Representations}, 2021.
\newblock URL \url{https://openreview.net/forum?id=0OlrLvrsHwQ}.

\bibitem[Di~Giovanni et~al.(2023)Di~Giovanni, Giusti, Barbero, Luise, Liò, and Bronstein]{pmlr-v202-di-giovanni23a}
Francesco Di~Giovanni, Lorenzo Giusti, Federico Barbero, Giulia Luise, Pietro Liò, and Michael~M. Bronstein.
\newblock On over-squashing in message passing neural networks: The impact of width, depth, and topology.
\newblock In \emph{Proceedings of the 40th International Conference on Machine Learning}, volume 202 of \emph{Proceedings of Machine Learning Research}, pp.\  7865--7885. PMLR, 2023.
\newblock URL \url{https://proceedings.mlr.press/v202/di-giovanni23a.html}.

\bibitem[Dwivedi \& Bresson(2020)Dwivedi and Bresson]{dwivedi2020generalization}
Vijay~Prakash Dwivedi and Xavier Bresson.
\newblock A generalization of transformer networks to graphs.
\newblock \emph{AAAI 2021 Workshop on Deep Learning on Graphs: Methods and Applications (DLG-AAAI 2021)}, 2020.
\newblock URL \url{https://arxiv.org/abs/2012.09699}.

\bibitem[Ebli et~al.(2020)Ebli, Defferrard, and Spreemann]{ebli2020simplicial}
Stefania Ebli, Micha{\"e}l Defferrard, and Gard Spreemann.
\newblock Simplicial neural networks.
\newblock In \emph{TDA {\&} Beyond}, 2020.
\newblock URL \url{https://openreview.net/forum?id=nPCt39DVIfk}.

\bibitem[Eidi \& Jost(2020)Eidi and Jost]{Eidi}
Marzieh Eidi and J{\"u}rgen Jost.
\newblock Ollivier {R}icci curvature of directed hypergraphs.
\newblock \emph{Scientific Reports}, 10\penalty0 (1):\penalty0 12466, 2020.
\newblock URL \url{https://doi.org/10.1038/s41598-020-68619-6}.

\bibitem[Eitan et~al.(2024)Eitan, Gelberg, Bar-Shalom, Frasca, Bronstein, and Maron]{eitan2024topological}
Yam Eitan, Yoav Gelberg, Guy Bar-Shalom, Fabrizio Frasca, Michael Bronstein, and Haggai Maron.
\newblock Topological blind spots: Understanding and extending topological deep learning through the lens of expressivity.
\newblock \emph{arXiv preprint arXiv:2408.05486}, 2024.

\bibitem[Fatemi et~al.(2023)Fatemi, Taslakian, Vazquez, and Poole]{fatemi2023knowledge}
Bahare Fatemi, Perouz Taslakian, David Vazquez, and David Poole.
\newblock Knowledge hypergraph embedding meets relational algebra.
\newblock \emph{Journal of Machine Learning Research}, 24\penalty0 (105):\penalty0 1--34, 2023.

\bibitem[Fesser \& Weber(2023)Fesser and Weber]{fesser2023mitigating}
Lukas Fesser and Melanie Weber.
\newblock Mitigating over-smoothing and over-squashing using augmentations of {F}orman-{R}icci curvature.
\newblock In \emph{The Second Learning on Graphs Conference}, 2023.
\newblock URL \url{https://openreview.net/forum?id=bKTkZMRtfC}.

\bibitem[Fesser \& Weber(2024)Fesser and Weber]{fesser2024}
Lukas Fesser and Melanie Weber.
\newblock Mitigating over-smoothing and over-squashing using augmentations of forman-ricci curvature, 2024.
\newblock URL \url{https://arxiv.org/abs/2309.09384}.

\bibitem[Fey \& Lenssen(2019)Fey and Lenssen]{fey2019fast}
Matthias Fey and Jan~Eric Lenssen.
\newblock Fast graph representation learning with pytorch geometric.
\newblock \emph{arXiv preprint arXiv:1903.02428}, 2019.

\bibitem[Flamary et~al.(2021)Flamary, Courty, Gramfort, Alaya, Boisbunon, Chambon, Chapel, Corenflos, Fatras, Fournier, Gautheron, Gayraud, Janati, Rakotomamonjy, Redko, Rolet, Schutz, Seguy, Sutherland, Tavenard, Tong, and Vayer]{JMLR:v22:20-451}
R{\'e}mi Flamary, Nicolas Courty, Alexandre Gramfort, Mokhtar~Z. Alaya, Aur{\'e}lie Boisbunon, Stanislas Chambon, Laetitia Chapel, Adrien Corenflos, Kilian Fatras, Nemo Fournier, L{\'e}o Gautheron, Nathalie~T.H. Gayraud, Hicham Janati, Alain Rakotomamonjy, Ievgen Redko, Antoine Rolet, Antony Schutz, Vivien Seguy, Danica~J. Sutherland, Romain Tavenard, Alexander Tong, and Titouan Vayer.
\newblock Pot: Python optimal transport.
\newblock \emph{Journal of Machine Learning Research}, 22\penalty0 (78):\penalty0 1--8, 2021.
\newblock URL \url{http://jmlr.org/papers/v22/20-451.html}.

\bibitem[Gama et~al.(2020)Gama, Ribeiro, and Bruna]{gama2020stability}
Fernando Gama, Alejandro Ribeiro, and Joan Bruna.
\newblock Stability of graph neural networks to relative perturbations.
\newblock In \emph{ICASSP 2020-2020 IEEE International Conference on Acoustics, Speech and Signal Processing (ICASSP)}, pp.\  9070--9074. IEEE, 2020.

\bibitem[Giusti(2024)]{giusti2024awesome}
Lorenzo Giusti.
\newblock Awesome topological deep learning, 2024.
\newblock URL \url{https://github.com/lrnzgiusti/awesome-topological-deep-learning}.
\newblock Accessed: Sep. 28, 2024.

\bibitem[Giusti et~al.(2023)Giusti, Battiloro, Testa, Di~Lorenzo, Sardellitti, and Barbarossa]{giusti2023cell}
Lorenzo Giusti, Claudio Battiloro, Lucia Testa, Paolo Di~Lorenzo, Stefania Sardellitti, and Sergio Barbarossa.
\newblock Cell attention networks.
\newblock In \emph{2023 International Joint Conference on Neural Networks (IJCNN)}, pp.\  1--8. IEEE, 2023.

\bibitem[Giusti et~al.(2024)Giusti, Reu, Ceccarelli, Bodnar, and Liò]{giusti2024cinplusplus}
Lorenzo Giusti, Teodora Reu, Francesco Ceccarelli, Cristian Bodnar, and Pietro Liò.
\newblock Topological message passing for higher - order and long - range interactions.
\newblock In \emph{2024 International Joint Conference on Neural Networks (IJCNN)}, pp.\  1--8, 2024.
\newblock \doi{10.1109/IJCNN60899.2024.10650343}.

\bibitem[Goh et~al.(2022)Goh, Bodnar, and Lio]{goh2022simplicial}
Christopher Wei~Jin Goh, Cristian Bodnar, and Pietro Lio.
\newblock Simplicial attention networks.
\newblock In \emph{ICLR 2022 Workshop on Geometrical and Topological Representation Learning}, 2022.
\newblock URL \url{https://openreview.net/forum?id=ScfRNWkpec}.

\bibitem[Hagberg et~al.(2008)Hagberg, Swart, and Schult]{osti_960616}
Aric Hagberg, Pieter~J Swart, and Daniel~A Schult.
\newblock Exploring network structure, dynamics, and function using networkx.
\newblock Technical report, Los Alamos National Laboratory (LANL), Los Alamos, NM (United States), 2008.

\bibitem[Hajij et~al.(2020)Hajij, Istvan, and Zamzmi]{hajij2020cell}
Mustafa Hajij, Kyle Istvan, and Ghada Zamzmi.
\newblock Cell complex neural networks.
\newblock In \emph{TDA {\&} Beyond}, 2020.
\newblock URL \url{https://openreview.net/forum?id=6Tq18ySFpGU}.

\bibitem[Hajij et~al.(2022)Hajij, Zamzmi, Papamarkou, Miolane, Guzm{\'a}n-S{\'a}enz, and Ramamurthy]{hajij2022higher}
Mustafa Hajij, Ghada Zamzmi, Theodore Papamarkou, Nina Miolane, Aldo Guzm{\'a}n-S{\'a}enz, and Karthikeyan~Natesan Ramamurthy.
\newblock Higher-order attention networks.
\newblock \emph{arXiv preprint arXiv:2206.00606}, 2\penalty0 (3):\penalty0 4, 2022.

\bibitem[Hajij et~al.(2023)Hajij, Zamzmi, Papamarkou, Miolane, Guzmán-Sáenz, Ramamurthy, Birdal, Dey, Mukherjee, Samaga, Livesay, Walters, Rosen, and Schaub]{hajij2022topological}
Mustafa Hajij, Ghada Zamzmi, Theodore Papamarkou, Nina Miolane, Aldo Guzmán-Sáenz, Karthikeyan~Natesan Ramamurthy, Tolga Birdal, Tamal~K. Dey, Soham Mukherjee, Shreyas~N. Samaga, Neal Livesay, Robin Walters, Paul Rosen, and Michael~T. Schaub.
\newblock Topological deep learning: Going beyond graph data, 2023.
\newblock URL \url{https://arxiv.org/abs/2206.00606}.

\bibitem[Hansen \& Ghrist(2019)Hansen and Ghrist]{hansen2019toward}
Jakob Hansen and Robert Ghrist.
\newblock Toward a spectral theory of cellular sheaves.
\newblock \emph{Journal of Applied and Computational Topology}, 3\penalty0 (4):\penalty0 315--358, 2019.

\bibitem[Harris et~al.(2020)Harris, Millman, van~der Walt, Gommers, Virtanen, Cournapeau, Wieser, Taylor, Berg, Smith, Kern, Picus, Hoyer, van Kerkwijk, Brett, Haldane, del R{\'{i}}o, Wiebe, Peterson, G{\'{e}}rard-Marchant, Sheppard, Reddy, Weckesser, Abbasi, Gohlke, and Oliphant]{harris2020array}
Charles~R. Harris, K.~Jarrod Millman, St{\'{e}}fan~J. van~der Walt, Ralf Gommers, Pauli Virtanen, David Cournapeau, Eric Wieser, Julian Taylor, Sebastian Berg, Nathaniel~J. Smith, Robert Kern, Matti Picus, Stephan Hoyer, Marten~H. van Kerkwijk, Matthew Brett, Allan Haldane, Jaime~Fern{\'{a}}ndez del R{\'{i}}o, Mark Wiebe, Pearu Peterson, Pierre G{\'{e}}rard-Marchant, Kevin Sheppard, Tyler Reddy, Warren Weckesser, Hameer Abbasi, Christoph Gohlke, and Travis~E. Oliphant.
\newblock Array programming with {NumPy}.
\newblock \emph{Nature}, 585\penalty0 (7825):\penalty0 357--362, September 2020.
\newblock \doi{10.1038/s41586-020-2649-2}.
\newblock URL \url{https://doi.org/10.1038/s41586-020-2649-2}.

\bibitem[Hodges(1993)]{hodges1993model}
Wilfrid Hodges.
\newblock \emph{Model theory}.
\newblock Cambridge University Press, 1993.

\bibitem[Huang et~al.(2024{\natexlab{a}})Huang, Wang, Li, and Lio]{huang2024how}
Keke Huang, Yu~Guang Wang, Ming Li, and Pietro Lio.
\newblock How universal polynomial bases enhance spectral graph neural networks: Heterophily, over-smoothing, and over-squashing.
\newblock In \emph{Forty-first International Conference on Machine Learning}, 2024{\natexlab{a}}.
\newblock URL \url{https://openreview.net/forum?id=Z2LH6Va7L2}.

\bibitem[Huang et~al.(2024{\natexlab{b}})Huang, Orth, Barcel{\'o}, Bronstein, and Ceylan]{huang2024link}
Xingyue Huang, Miguel~Romero Orth, Pablo Barcel{\'o}, Michael~M Bronstein, and {\.I}smail~{\.I}lkan Ceylan.
\newblock Link prediction with relational hypergraphs.
\newblock \emph{arXiv preprint arXiv:2402.04062}, 2024{\natexlab{b}}.

\bibitem[Hunter(2007)]{Hunter:2007}
J.~D. Hunter.
\newblock Matplotlib: A 2d graphics environment.
\newblock \emph{Computing in Science \& Engineering}, 9\penalty0 (3):\penalty0 90--95, 2007.
\newblock \doi{10.1109/MCSE.2007.55}.

\bibitem[Jamadandi et~al.(2024)Jamadandi, Rubio-Madrigal, and Burkholz]{jamadandi2024spectral}
Adarsh Jamadandi, Celia Rubio-Madrigal, and Rebekka Burkholz.
\newblock Spectral graph pruning against over-squashing and over-smoothing, 2024.
\newblock URL \url{https://arxiv.org/abs/2404.04612}.

\bibitem[Karhadkar et~al.(2023)Karhadkar, Banerjee, and Mont\'ufar]{karhadkar2023fosr}
Kedar Karhadkar, Pradeep~Kr. Banerjee, and Guido Mont\'ufar.
\newblock Fo{SR}: First-order spectral rewiring for addressing oversquashing in {GNN}s.
\newblock In \emph{The Eleventh International Conference on Learning Representations}, 2023.
\newblock URL \url{https://openreview.net/forum?id=3YjQfCLdrzz}.

\bibitem[Kipf \& Welling(2017)Kipf and Welling]{kipf2017semisupervised}
Thomas~N. Kipf and Max Welling.
\newblock Semi-supervised classification with graph convolutional networks.
\newblock In \emph{International Conference on Learning Representations}, 2017.
\newblock URL \url{https://openreview.net/forum?id=SJU4ayYgl}.

\bibitem[Lam et~al.(2015)Lam, Pitrou, and Seibert]{10.1145/2833157.2833162}
Siu~Kwan Lam, Antoine Pitrou, and Stanley Seibert.
\newblock Numba: a llvm-based python jit compiler.
\newblock In \emph{Proceedings of the Second Workshop on the LLVM Compiler Infrastructure in HPC}, LLVM '15, New York, NY, USA, 2015. Association for Computing Machinery.
\newblock ISBN 9781450340052.
\newblock \doi{10.1145/2833157.2833162}.
\newblock URL \url{https://doi.org/10.1145/2833157.2833162}.

\bibitem[Li et~al.(2024)Li, Li, Zhang, Ouyang, and Rong]{li2024addressing}
Hao Li, Chen Li, Jianfei Zhang, Yuanxin Ouyang, and Wenge Rong.
\newblock Addressing over-squashing in gnns with graph rewiring and ordered neurons.
\newblock In \emph{2024 International Joint Conference on Neural Networks (IJCNN)}, pp.\  1--8. IEEE, 2024.

\bibitem[Li et~al.(2018)Li, Han, and Wu]{li2018deeper}
Qimai Li, Zhichao Han, and Xiao-Ming Wu.
\newblock Deeper insights into graph convolutional networks for semi-supervised learning.
\newblock In \emph{Proceedings of the AAAI conference on artificial intelligence}, volume~32, 2018.
\newblock URL \url{https://doi.org/10.1609/aaai.v32i1.11604}.

\bibitem[Liaw et~al.(2018)Liaw, Liang, Nishihara, Moritz, Gonzalez, and Stoica]{liaw2018tune}
Richard Liaw, Eric Liang, Robert Nishihara, Philipp Moritz, Joseph~E Gonzalez, and Ion Stoica.
\newblock Tune: A research platform for distributed model selection and training.
\newblock \emph{arXiv preprint arXiv:1807.05118}, 2018.

\bibitem[Liu et~al.(2023)Liu, Zhou, Pan, Wu, Li, Chen, and Zhang]{liu2023curvdrop}
Yang Liu, Chuan Zhou, Shirui Pan, Jia Wu, Zhao Li, Hongyang Chen, and Peng Zhang.
\newblock Curvdrop: A ricci curvature based approach to prevent graph neural networks from over-smoothing and over-squashing.
\newblock In \emph{Proceedings of the ACM Web Conference 2023}, pp.\  221--230, 2023.

\bibitem[Liu et~al.(2020)Liu, Yao, and Li]{liu2020generalizing}
Yu~Liu, Quanming Yao, and Yong Li.
\newblock Generalizing tensor decomposition for n-ary relational knowledge bases.
\newblock In \emph{Proceedings of the web conference 2020}, pp.\  1104--1114, 2020.

\bibitem[Maron et~al.(2019)Maron, Ben-Hamu, Serviansky, and Lipman]{maron2019provably}
Haggai Maron, Heli Ben-Hamu, Hadar Serviansky, and Yaron Lipman.
\newblock Provably powerful graph networks.
\newblock In \emph{Advances in Neural Information Processing Systems}, volume~32. Curran Associates, Inc., 2019.
\newblock URL \url{https://proceedings.neurips.cc/paper_files/paper/2019/file/bb04af0f7ecaee4aae62035497da1387-Paper.pdf}.

\bibitem[Mateos et~al.(2019)Mateos, Segarra, Marques, and Ribeiro]{mateos2019connecting}
Gonzalo Mateos, Santiago Segarra, Antonio~G Marques, and Alejandro Ribeiro.
\newblock Connecting the dots: Identifying network structure via graph signal processing.
\newblock \emph{IEEE Signal Processing Magazine}, 36\penalty0 (3):\penalty0 16--43, 2019.

\bibitem[Moritz et~al.(2018)Moritz, Nishihara, Wang, Tumanov, Liaw, Liang, Elibol, Yang, Paul, Jordan, et~al.]{moritz2018ray}
Philipp Moritz, Robert Nishihara, Stephanie Wang, Alexey Tumanov, Richard Liaw, Eric Liang, Melih Elibol, Zongheng Yang, William Paul, Michael~I Jordan, et~al.
\newblock Ray: A distributed framework for emerging $\{$AI$\}$ applications.
\newblock In \emph{13th USENIX symposium on operating systems design and implementation (OSDI 18)}, pp.\  561--577, 2018.

\bibitem[Morris et~al.(2019)Morris, Ritzert, Fey, Hamilton, Lenssen, Rattan, and Grohe]{morris2019weisfeiler}
Christopher Morris, Martin Ritzert, Matthias Fey, William~L Hamilton, Jan~Eric Lenssen, Gaurav Rattan, and Martin Grohe.
\newblock Weisfeiler and {L}eman go neural: Higher-order graph neural networks.
\newblock In \emph{Proceedings of the AAAI conference on artificial intelligence}, volume~33, pp.\  4602--4609, 2019.
\newblock URL \url{https://doi.org/10.1609/aaai.v33i01.33014602}.

\bibitem[Morris et~al.(2020)Morris, Kriege, Bause, Kersting, Mutzel, and Neumann]{morris2020}
Christopher Morris, Nils~M. Kriege, Franka Bause, Kristian Kersting, Petra Mutzel, and Marion Neumann.
\newblock {TUD}ataset: A collection of benchmark datasets for learning with graphs.
\newblock In \emph{ICML 2020 Workshop on Graph Representation Learning and Beyond (GRL+ 2020)}, 2020.
\newblock URL \url{http://eprints.cs.univie.ac.at/6469/}.

\bibitem[Münch(2022)]{münch2022olliviercurvaturebetweennesscentrality}
Florentin Münch.
\newblock Ollivier curvature, betweenness centrality and average distance, 2022.
\newblock URL \url{https://arxiv.org/abs/2209.15564}.

\bibitem[Nanda(2021)]{nanda2021computational}
V.~Nanda.
\newblock Computational algebraic topology lecture notes, 2021.
\newblock Available online.

\bibitem[Nguyen et~al.(2023)Nguyen, Hieu, Nguyen, Ho, Osher, and Nguyen]{pmlr-v202-nguyen23c}
Khang Nguyen, Nong~Minh Hieu, Vinh~Duc Nguyen, Nhat Ho, Stanley Osher, and Tan~Minh Nguyen.
\newblock Revisiting over-smoothing and over-squashing using {O}llivier-{R}icci curvature.
\newblock In \emph{Proceedings of the 40th International Conference on Machine Learning}, volume 202 of \emph{Proceedings of Machine Learning Research}, pp.\  25956--25979. PMLR, 2023.
\newblock URL \url{https://proceedings.mlr.press/v202/nguyen23c.html}.

\bibitem[Nickel et~al.(2015)Nickel, Murphy, Tresp, and Gabrilovich]{nickel2015review}
Maximilian Nickel, Kevin Murphy, Volker Tresp, and Evgeniy Gabrilovich.
\newblock A review of relational machine learning for knowledge graphs.
\newblock \emph{Proceedings of the IEEE}, 104\penalty0 (1):\penalty0 11--33, 2015.

\bibitem[Nt \& Maehara(2019)Nt and Maehara]{nt2019revisiting}
Hoang Nt and Takanori Maehara.
\newblock Revisiting graph neural networks: All we have is low-pass filters.
\newblock \emph{arXiv preprint arXiv:1905.09550}, 2019.
\newblock URL \url{https://arxiv.org/abs/1905.09550}.

\bibitem[Oono \& Suzuki(2020)Oono and Suzuki]{oono2019graph}
Kenta Oono and Taiji Suzuki.
\newblock Graph neural networks exponentially lose expressive power for node classification.
\newblock In \emph{International Conference on Learning Representations}, 2020.
\newblock URL \url{https://openreview.net/forum?id=S1ldO2EFPr}.

\bibitem[pandas~development team(2020)]{reback2020pandas}
The pandas~development team.
\newblock pandas-dev/pandas: Pandas, February 2020.
\newblock URL \url{https://doi.org/10.5281/zenodo.3509134}.

\bibitem[Papamarkou et~al.(2024)Papamarkou, Birdal, Bronstein, Carlsson, Curry, Gao, Hajij, Kwitt, Lio, Di~Lorenzo, Maroulas, Miolane, Nasrin, Natesan~Ramamurthy, Rieck, Scardapane, Schaub, Veli\v{c}kovi\'{c}, Wang, Wang, Wei, and Zamzmi]{pmlr-v235-papamarkou24a}
Theodore Papamarkou, Tolga Birdal, Michael~M. Bronstein, Gunnar~E. Carlsson, Justin Curry, Yue Gao, Mustafa Hajij, Roland Kwitt, Pietro Lio, Paolo Di~Lorenzo, Vasileios Maroulas, Nina Miolane, Farzana Nasrin, Karthikeyan Natesan~Ramamurthy, Bastian Rieck, Simone Scardapane, Michael~T Schaub, Petar Veli\v{c}kovi\'{c}, Bei Wang, Yusu Wang, Guowei Wei, and Ghada Zamzmi.
\newblock Position: Topological deep learning is the new frontier for relational learning.
\newblock In \emph{Proceedings of the 41st International Conference on Machine Learning}, volume 235 of \emph{Proceedings of Machine Learning Research}, pp.\  39529--39555. PMLR, 2024.
\newblock URL \url{https://proceedings.mlr.press/v235/papamarkou24a.html}.

\bibitem[Papillon et~al.(2023{\natexlab{a}})Papillon, Hajij, Myers, Jenne, Mathe, Papamarkou, Guzm\'{a}n-S\'{a}enz, Livesay, Dey, Rabinowitz, Brent, Salatiello, Nikitin, Zia, Battiloro, Gavrilev, Magai, Bazhenov, Bernardez, Spinelli, Agerberg, Nadimpalli, Telyatninkov, Scofano, Testa, Lecha, Yang, Hassanin, Gardaa, Zaghen, Hausner, Snopoff, Ballester, Barikbin, Escalera, Fiorellino, Kvinge, Ramamurthy, Rosen, Walters, Samaga, Mukherjee, Sanborn, Emerson, Doster, Birdal, Khamis, Scardapane, Singh, Malygina, Yue, and Miolane]{icml2024challenge}
Mathilde Papillon, Mustafa Hajij, Audun Myers, Helen Jenne, Johan Mathe, Theodore Papamarkou, Aldo Guzm\'{a}n-S\'{a}enz, Neal Livesay, Tamal Dey, Abraham Rabinowitz, Aiden Brent, Alessandro Salatiello, Alexander Nikitin, Ali Zia, Claudio Battiloro, Dmitrii Gavrilev, German Magai, Gleb Bazhenov, Guillermo Bernardez, Indro Spinelli, Jens Agerberg, Kalyan Nadimpalli, Lev Telyatninkov, Luca Scofano, Lucia Testa, Manuel Lecha, Maosheng Yang, Mohammed Hassanin, Odin~Hoff Gardaa, Olga Zaghen, Paul Hausner, Paul Snopoff, Rub\'{e}n Ballester, Sadrodin Barikbin, Sergio Escalera, Simone Fiorellino, Henry Kvinge, Karthikeyan~Natesan Ramamurthy, Paul Rosen, Robin Walters, Shreyas~N. Samaga, Soham Mukherjee, Sophia Sanborn, Tegan Emerson, Timothy Doster, Tolga Birdal, Abdelwahed Khamis, Simone Scardapane, Suraj Singh, Tatiana Malygina, Yixiao Yue, and Nina Miolane.
\newblock {ICML 2023 Topological Deep Learning Challenge}: Design and results.
\newblock In \emph{Proceedings of 2nd Annual Workshop on Topology, Algebra, and Geometry in Machine Learning (TAG-ML)}, volume 221 of \emph{Proceedings of Machine Learning Research}, pp.\  3--8. PMLR, 2023{\natexlab{a}}.
\newblock URL \url{https://proceedings.mlr.press/v221/papillon23a.html}.

\bibitem[Papillon et~al.(2023{\natexlab{b}})Papillon, Sanborn, Hajij, and Miolane]{papillon2023architectures}
Mathilde Papillon, Sophia Sanborn, Mustafa Hajij, and Nina Miolane.
\newblock Architectures of topological deep learning: A survey of message-passing topological neural networks.
\newblock \emph{arXiv preprint arXiv:2304.10031}, 2023{\natexlab{b}}.

\bibitem[Papillon et~al.(2024)Papillon, Bern{\'a}rdez, Battiloro, and Miolane]{papillon2024topotune}
Mathilde Papillon, Guillermo Bern{\'a}rdez, Claudio Battiloro, and Nina Miolane.
\newblock Topotune: A framework for generalized combinatorial complex neural networks.
\newblock \emph{arXiv preprint arXiv:2410.06530}, 2024.

\bibitem[Paszke et~al.(2019)Paszke, Gross, Massa, Lerer, Bradbury, Chanan, Killeen, Lin, Gimelshein, Antiga, et~al.]{paszke2019pytorch}
Adam Paszke, Sam Gross, Francisco Massa, Adam Lerer, James Bradbury, Gregory Chanan, Trevor Killeen, Zeming Lin, Natalia Gimelshein, Luca Antiga, et~al.
\newblock Pytorch: An imperative style, high-performance deep learning library.
\newblock \emph{Advances in neural information processing systems}, 32, 2019.

\bibitem[Pedregosa et~al.(2011)Pedregosa, Varoquaux, Gramfort, Michel, Thirion, Grisel, Blondel, Prettenhofer, Weiss, Dubourg, Vanderplas, Passos, Cournapeau, Brucher, Perrot, and Duchesnay]{scikit-learn}
F.~Pedregosa, G.~Varoquaux, A.~Gramfort, V.~Michel, B.~Thirion, O.~Grisel, M.~Blondel, P.~Prettenhofer, R.~Weiss, V.~Dubourg, J.~Vanderplas, A.~Passos, D.~Cournapeau, M.~Brucher, M.~Perrot, and E.~Duchesnay.
\newblock Scikit-learn: Machine learning in {P}ython.
\newblock \emph{Journal of Machine Learning Research}, 12:\penalty0 2825--2830, 2011.

\bibitem[Pei et~al.(2024)Pei, Li, Deng, Hai, Wang, Ma, Tao, Xiong, and Guan]{peimulti}
Hongbin Pei, Yu~Li, Huiqi Deng, Jingxin Hai, Pinghui Wang, Jie Ma, Jing Tao, Yuheng Xiong, and Xiaohong Guan.
\newblock Multi-track message passing: Tackling oversmoothing and oversquashing in graph learning via preventing heterophily mixing.
\newblock In \emph{Forty-first International Conference on Machine Learning}, 2024.

\bibitem[Qian et~al.(2023)Qian, Manolache, Ahmed, Zeng, Van~den Broeck, Niepert, and Morris]{qian2023probabilistic}
Chendi Qian, Andrei Manolache, Kareem Ahmed, Zhe Zeng, Guy Van~den Broeck, Mathias Niepert, and Christopher Morris.
\newblock Probabilistic task-adaptive graph rewiring.
\newblock In \emph{ICML 2023 Workshop on Differentiable Almost Everything: Differentiable Relaxations, Algorithms, Operators, and Simulators}, 2023.

\bibitem[Robinson et~al.(2024)Robinson, Ranjan, Hu, Huang, Han, Dobles, Fey, Lenssen, Yuan, Zhang, et~al.]{robinson2024relational}
Joshua Robinson, Rishabh Ranjan, Weihua Hu, Kexin Huang, Jiaqi Han, Alejandro Dobles, Matthias Fey, Jan~Eric Lenssen, Yiwen Yuan, Zecheng Zhang, et~al.
\newblock Relational deep learning: Graph representation learning on relational databases.
\newblock In \emph{NeurIPS 2024 Third Table Representation Learning Workshop}, 2024.

\bibitem[Rusch et~al.(2023)Rusch, Bronstein, and Mishra]{rusch2023survey}
T.~Konstantin Rusch, Michael~M. Bronstein, and Siddhartha Mishra.
\newblock A survey on oversmoothing in graph neural networks.
\newblock \emph{arXiv preprint arXiv.2303.10993}, 2023.
\newblock URL \url{https://doi.org/10.48550/arXiv.2303.10993}.

\bibitem[Schlichtkrull et~al.(2018)Schlichtkrull, Kipf, Bloem, Van Den~Berg, Titov, and Welling]{schlichtkrull2018modeling}
Michael Schlichtkrull, Thomas~N Kipf, Peter Bloem, Rianne Van Den~Berg, Ivan Titov, and Max Welling.
\newblock Modeling relational data with graph convolutional networks.
\newblock In \emph{The semantic web: 15th international conference, ESWC 2018, Heraklion, Crete, Greece, June 3--7, 2018, proceedings 15}, pp.\  593--607. Springer, 2018.

\bibitem[Shen et~al.(2024{\natexlab{a}})Shen, Qin, Zhang, Zhu, and Xiong]{shen2024handling}
Dazhong Shen, Chuan Qin, Qi~Zhang, Hengshu Zhu, and Hui Xiong.
\newblock Handling over-smoothing and over-squashing in graph convolution with maximization operation.
\newblock \emph{IEEE Transactions on Neural Networks and Learning Systems}, 2024{\natexlab{a}}.

\bibitem[Shen et~al.(2024{\natexlab{b}})Shen, Lio, Yang, Yuan, Zhang, and Peng]{shen2024graph}
Xu~Shen, Pietro Lio, Lintao Yang, Ru~Yuan, Yuyang Zhang, and Chengbin Peng.
\newblock Graph rewiring and preprocessing for graph neural networks based on effective resistance.
\newblock \emph{IEEE Transactions on Knowledge and Data Engineering}, 2024{\natexlab{b}}.

\bibitem[Shi et~al.(2023)Shi, Han, Lin, Guo, and Gao]{shi2023exposition}
Dai Shi, Andi Han, Lequan Lin, Yi~Guo, and Junbin Gao.
\newblock Exposition on over-squashing problem on {GNNs}: Current methods, benchmarks and challenges.
\newblock \emph{arXiv preprint arXiv.2311.07073}, 2023.
\newblock URL \url{https://doi.org/10.48550/arXiv.2311.07073}.

\bibitem[Shi et~al.(2024)Shi, Han, Lin, Guo, Wang, and Gao]{shi2024design}
Dai Shi, Andi Han, Lequan Lin, Yi~Guo, Zhiyong Wang, and Junbin Gao.
\newblock Design your own universe: A physics-informed agnostic method for enhancing graph neural networks.
\newblock \emph{International Journal of Machine Learning and Cybernetics}, pp.\  1--16, 2024.

\bibitem[Stanovic et~al.(2024)Stanovic, Ga{\"u}z{\`e}re, and Brun]{stanovic2024graph}
Stevan Stanovic, Benoit Ga{\"u}z{\`e}re, and Luc Brun.
\newblock Graph neural networks with maximal independent set-based pooling: Mitigating over-smoothing and over-squashing.
\newblock \emph{Pattern Recognition Letters}, 2024.

\bibitem[Suk et~al.(2022)Suk, Giusti, Hemo, Lopez, Barmpas, and Bodnar]{suk2022surfing}
Julian Suk, Lorenzo Giusti, Tamir Hemo, Miguel Lopez, Konstantinos Barmpas, and Cristian Bodnar.
\newblock Surfing on the neural sheaf.
\newblock In \emph{NeurIPS 2022 Workshop on Symmetry and Geometry in Neural Representations}, 2022.
\newblock URL \url{https://openreview.net/forum?id=xOXFkyRzTlu}.

\bibitem[Sun et~al.(2023)Sun, Huang, Wu, Ye, Peng, Yu, and Philip]{sun2023deepricci}
Li~Sun, Zhenhao Huang, Hua Wu, Junda Ye, Hao Peng, Zhengtao Yu, and S~Yu Philip.
\newblock Deepricci: Self-supervised graph structure-feature co-refinement for alleviating over-squashing.
\newblock In \emph{2023 IEEE International Conference on Data Mining (ICDM)}, pp.\  558--567. IEEE, 2023.

\bibitem[{The GUDHI Project}(2020)]{gudhi:urm}
{The GUDHI Project}.
\newblock \emph{{GUDHI} User and Reference Manual}.
\newblock {GUDHI Editorial Board}, {3.1.1} edition, 2020.
\newblock URL \url{https://gudhi.inria.fr/doc/3.1.1/}.

\bibitem[Topping et~al.(2022)Topping, Giovanni, Chamberlain, Dong, and Bronstein]{topping2022understanding}
Jake Topping, Francesco~Di Giovanni, Benjamin~Paul Chamberlain, Xiaowen Dong, and Michael~M. Bronstein.
\newblock Understanding over-squashing and bottlenecks on graphs via curvature.
\newblock In \emph{International Conference on Learning Representations}, 2022.
\newblock URL \url{https://openreview.net/forum?id=7UmjRGzp-A}.

\bibitem[Tori et~al.(2024)Tori, Holst, and Ginis]{tori2024effectiveness}
Floriano Tori, Vincent Holst, and Vincent Ginis.
\newblock The effectiveness of curvature-based rewiring and the role of hyperparameters in gnns revisited.
\newblock \emph{arXiv preprint arXiv:2407.09381}, 2024.

\bibitem[Vashishth et~al.(2019)Vashishth, Sanyal, Nitin, and Talukdar]{vashishth2019composition}
Shikhar Vashishth, Soumya Sanyal, Dasgupta~S Nitin, and Partha Talukdar.
\newblock Composition-based multi-relational graph convolutional networks.
\newblock In \emph{International Conference on Learning Representations}, 2019.

\bibitem[Veli{\v{c}}kovi{\'c} et~al.(2018)Veli{\v{c}}kovi{\'c}, Cucurull, Casanova, Romero, Li{\`o}, and Bengio]{velivckovic2018graph}
Petar Veli{\v{c}}kovi{\'c}, Guillem Cucurull, Arantxa Casanova, Adriana Romero, Pietro Li{\`o}, and Yoshua Bengio.
\newblock Graph attention networks.
\newblock In \emph{International Conference on Learning Representations}, 2018.

\bibitem[Waskom(2021)]{Waskom2021}
Michael~L. Waskom.
\newblock seaborn: statistical data visualization.
\newblock \emph{Journal of Open Source Software}, 6\penalty0 (60):\penalty0 3021, 2021.
\newblock \doi{10.21105/joss.03021}.
\newblock URL \url{https://doi.org/10.21105/joss.03021}.

\bibitem[Wu et~al.(2019)Wu, Souza, Zhang, Fifty, Yu, and Weinberger]{wu2019simplifying}
Felix Wu, Amauri Souza, Tianyi Zhang, Christopher Fifty, Tao Yu, and Kilian Weinberger.
\newblock Simplifying graph convolutional networks.
\newblock In \emph{Proceedings of the 36th International Conference on Machine Learning}, volume~97 of \emph{Proceedings of Machine Learning Research}, pp.\  6861--6871. PMLR, 2019.
\newblock URL \url{https://proceedings.mlr.press/v97/wu19e.html}.

\bibitem[Xu et~al.(2019)Xu, Hu, Leskovec, and Jegelka]{xu2018how}
Keyulu Xu, Weihua Hu, Jure Leskovec, and Stefanie Jegelka.
\newblock How powerful are graph neural networks?
\newblock In \emph{International Conference on Learning Representations}, 2019.
\newblock URL \url{https://openreview.net/forum?id=ryGs6iA5Km}.

\bibitem[Yadati(2020)]{yadati2020neural}
Naganand Yadati.
\newblock Neural message passing for multi-relational ordered and recursive hypergraphs.
\newblock \emph{Advances in Neural Information Processing Systems}, 33:\penalty0 3275--3289, 2020.

\end{thebibliography}
\bibliographystyle{iclr2025_conference}

\newpage 
\appendix

\addcontentsline{toc}{section}{Appendix} 
\part{Appendix} 
\parttoc

\section{Related Work}
\label{appendix:related-work}

\paragraph{Topological Networks} Topological deep learning (TDL) integrates algebraic topology with neural networks to create message-passing schemes that are more expressive than graph neural networks. An early contribution was the simplicial Weisfeiler-Lehman (SWL) test introduced by \citet{bodnar2021weisfeiler}, which extends the Weisfeiler-Lehman (WL) test from graphs to simplicial complexes. Simplicial neural networks (SNNs) built on SWL generalize graph isomorphism networks (GIN) \citep{xu2018how}, offering provably stronger expressiveness. CW networks (CWNs) \citep{pmlr-v139-bodnar21a} extend message passing to cell complexes, achieving greater power than the traditional WL test and exceeding the $3$-WL test. These hierarchical, geometrically-grounded representations enable effective handling of higher-order interactions. \citet{hajij2020cell} proposed a general message-passing scheme for cell complexes, though it lacks formal analysis of expressiveness and complexity. In contrast, \citet{bodnar2022neural} and \citet{suk2022surfing} introduced neural sheaf diffusion models, which learn sheaf structures over graphs, particularly excelling in heterophilic graph tasks. Attention mechanisms have been integrated into topological message passing in works such as \citet{goh2022simplicial} for simplicial complexes, \citet{barbero2022sheaf} for sheaves, and \citet{giusti2023cell} for cellular complexes. Additionally, \citet{hajij2022higher} extends message-passing to combinatorial complexes. Our work represents a first step toward extending the theory of oversquashing and oversmoothing, widely studied in graph neural networks, to these topological networks.

We note that our relational structures approach generalizes a powerful perspective in which simplicial complexes and similar constructs are treated as augmented Hasse diagrams---a viewpoint that has shown both practical and theoretical advantages. Recent works have leveraged this perspective: \citet{hajij2022topological} provided a general description of topological message passing schemes, \citet{eitan2024topological} explored the expressivity limits of topological message passing, and \citet{papillon2024topotune} introduced a framework for systematically transforming any graph neural network into a topological analog. Our work aligns with and contributes to this growing body of research.

For detailed surveys of topological deep learning architectures, we refer to \citet{papillon2023architectures} and \citet{giusti2024awesome}. We also refer the reader to the recent position paper by \citet{pmlr-v235-papamarkou24a} on open problems in TDL.

\paragraph{Oversquashing, Oversmoothing, and Graph Rewiring} 
Oversquashing refers to the phenomenon where information from distant nodes is compressed into fixed-size vectors during message passing, limiting the ability of a model to capture long-range dependencies. In GNNs, this has been extensively studied, with works such as those of \citet{alon2021on}, \citet{topping2022understanding}, and \citet{pmlr-v202-di-giovanni23a} naming and relating the phenomenon to the geometric properties of graphs. 
To address this issue, numerous techniques have been proposed, including spatial and curvature-based rewiring methods such as SDRF \citep{topping2022understanding}, BORF \citep{pmlr-v202-nguyen23c}, and AFRC \citep{fesser2023mitigating}, which modify the graph structure to improve connectivity and alleviate oversquashing. 
Spectral rewiring approaches, such as FOSR \citep{karhadkar2023fosr}, optimize the graph spectral gap to enhance long-range message propagation, while implicit methods such as Graph Beltrami Diffusion \citep{chamberlain2021beltrami} and Graph Transformers \citep{dwivedi2020generalization} allow information flow across a fully connected graph without explicitly modifying the topology.

Related to oversquashing, another significant challenge in GNNs is oversmoothing, where node features become indistinguishable as they are excessively aggregated through layers of message passing. \citet{li2018deeper}, \citet{oono2019graph}, and \citet{nt2019revisiting} have identified and theoretically analyzed oversmoothing, showing that it limits the effectiveness of deep GNNs. Generally, there is a trade-off between mitigating oversquashing and avoiding oversmoothing, and various rewiring techniques aim to balance these competing objectives.

Following foundational work on understanding and mitigating oversquashing and oversmoothing in graph neural networks, the field has rapidly evolved with various innovative and sophisticated approaches in recent years. For example, \citet{liu2023curvdrop} introduced CurvDrop, a Ricci curvature-based, topology-aware method to address both failure modes of message passing, while \citet{sun2023deepricci} proposed DeepRicci, a self-supervised Riemannian model designed to alleviate oversquashing. \citet{qian2023probabilistic} developed a probabilistic framework leveraging differentiable $k$-subset sampling, and \citet{shen2024handling} introduced a technique based on maximization operations in graph convolution. \citet{shen2024graph} used effective resistance for graph rewiring and preprocessing. \citet{li2024addressing} combined graph rewiring with ordered neurons, while \citet{stanovic2024graph} proposed maximal independent set-based pooling to mitigate both oversquashing and oversmoothing. \citet{shi2024design} presented a physics-informed, agnostic method targeting both failure modes, and \citet{peimulti} addressed them through multi-track message passing. \citet{attalidelaunay} employed Delaunay triangulation of features, while \citet{huang2024how} proposed particular spectral filters to improve spectral GNNs. Collectively, these works demonstrate the breadth and depth of recent advancements, highlighting a vibrant and rapidly evolving research landscape.

However, the study of oversquashing and oversmoothing in topological message passing remains largely unexplored. While works such as those of \citet{pmlr-v139-bodnar21a}, \citet{bodnar2021weisfeiler}, and \cite{giusti2024cinplusplus} have alluded to the potential of topological message passing to capture group interactions and long-range dependencies, a theoretical analysis of oversquashing and oversmoothing in higher-order structures like simplicial or cellular complexes is, to the best of our knowledge, absent from the literature. 
Our work aims to fill this gap, providing a rigorous first step in studying oversquashing in topological message passing. Moreover, our framework for extending graph rewiring to relational rewiring (Section~\ref{sec:rewiring}), along with future innovations from the topological deep learning community, promises a wealth of topological and relational rewiring techniques for further assessment and refinement.

For more exhaustive expositions on oversquashing and oversmoothing, we refer the reader to the excellent recent surveys of \citet{shi2023exposition} and \citet{rusch2023survey}. 

\paragraph{Relational Learning} Relational graph neural networks (R-GNNs) extend traditional GNNs to handle multi-relational data, particularly in the context of knowledge graphs, where nodes represent entities and edges capture diverse types of relationships between them. Knowledge graphs, such as those used in knowledge representation and reasoning tasks \citep{nickel2015review}, are inherently multi-relational and benefit significantly from RGNNs, which model the varying nature of relations explicitly. The relational graph convolutional network (R-GCN) \citep{schlichtkrull2018modeling} is a foundational approach that assigns distinct transformations to each relation type, allowing for efficient representation learning on multi-relational graphs. Extensions such as Relational Graph Attention Networks (RGAT) \citep{velivckovic2018graph} incorporate attention mechanisms, enabling the model to focus on the most important relations during message passing. Additionally, Composition-based Graph Convolutional Networks (CompGCN) \citep{vashishth2019composition} apply compositional operators to better capture interactions in knowledge graphs.

Recent works in relational learning have introduced various extensions for modeling multi-relational and higher-order interactions. Hypergraph-based approaches, such as \citet{fatemi2023knowledge} and \citet{huang2024link}, unify relational reasoning with hypergraph structures, enabling the modeling of higher-order relationships. Message-passing frameworks like the one proposed by \citet{yadati2020neural} extend traditional GNN paradigms to ordered and recursive hypergraphs. Additionally, tensor decomposition methods such as GETD \citep{liu2020generalizing} represent hypergraphs as high-dimensional tensors, allowing efficient encoding of hyper-relational data. We refer the reader to the excellent survey by \citet{antelmi2023survey} for a comprehensive overview of relational hypergraphs and their representation learning techniques. \citet{robinson2024relational} recently extended relational learning to relational databases consisting of data laid out across multiple tables. Lastly, topological deep learning, which is one focus of this work, is a new frontier in relational learning \citep{pmlr-v235-papamarkou24a}.

Our work unifies relational graph neural networks and topological neural networks by viewing complexes as relational structures, bridging the gap between the two fields.

\paragraph{Graph Lifting} Graph lifting transforms a graph into a higher-dimensional structure to enable more expressive message-passing schemes. For example, higher-order graph neural networks ($k$-GNNs) \citep{morris2019weisfeiler} lift graphs by representing $k$-node subgraphs as entities, capturing more complex relationships between nodes. Similarly, \citet{chen2019equivalence} and \citet{maron2019provably} proposed lifting graphs to higher-order structures to improve graph isomorphism testing. In topological message passing, graphs are lifted into structures like clique complexes to capture interactions beyond pairwise relationships. Recently, this topic has gained increased attention, as highlighted by the ICML Topological Deep Learning Challenge 2024 \citep{icml2024challenge}, which emphasized the development of topological lifting techniques across various data structures, including graphs, hypergraphs, and simplicial complexes. Among these, high-order graphs, simplicial complexes, and cellular complexes fit naturally within the relational structure framework and can be analyzed uniformly through this lens.

\section{Additional Remarks}
\label{appendix:additional-remarks}

\subsection{Remarks for Section~\ref{subsec:impact-of-local-geometry}}
\label{appendix-remarks-subsec=impact-of-local-geometry}

\begin{remark}[Ollivier-Ricci Curvature for Weighted Directed Graphs]
\label{remark:orc-weighted-directed-definition}
We recall the definition of \emph{Ollivier-Ricci curvature} (ORC) for weighted directed graphs from \citet{Eidi}. Consider a weighted directed graph $\mathcal{G} = (\mathcal{S}, \mathcal{E}, w)$, where $\mathcal{S}$ is the set of entities (nodes), $\mathcal{E}$ is the set of directed edges, and $w: \mathcal{E} \to \mathbb{R}_{\geq 0}$ assigns non-negative weights to edges. We abuse notation, and for any two entities $\xi, \eta \in \mathcal{S}$ without an edge $(\xi \to \eta) \in \mathcal{E}$, we write $w_{\xi \to \eta} = 0$. For each entity $\sigma \in \mathcal{S}$, denote the \emph{weighted out-degree} and \emph{weighted in-degree}, as in Definition~\ref{definition:motif-counts}, by:
\[
w_{\sigma}^{\mathrm{out}} = \sum_{(\sigma \to \eta) \in \mathcal{E}} w_{\sigma \to \eta}, \quad
w_{\sigma}^{\mathrm{in}} = \sum_{(\xi \to \sigma) \in \mathcal{E}} w_{\xi \to \sigma}.
\]
For a directed edge $(\sigma \to \tau) \in \mathcal{E}$ with non-zero weight $w_{\sigma \to \tau} > 0$, we define the probability measures: 
\[
\mu_{\tau}^{\mathrm{out}}(\xi) = \frac{w_{\tau \to \xi}}{w_{\tau}^{\mathrm{out}}}, \quad \mu_{\sigma}^{\mathrm{in}}(\xi) = \frac{w_{\xi \to \sigma}}{w_{\sigma}^{\mathrm{in}}}, \quad \xi \in \mathcal{S}.
\]
The \emph{Ollivier-Ricci curvature} of an edge $(\sigma \to \tau)$ with non-zero weight $w_{\sigma \to \tau}$ is then defined as:
\[
k(\sigma, \tau) = 1 - \frac{W\left(\mu_{\sigma}^{\mathrm{in}},\, \mu_{\tau}^{\mathrm{out}}\right)}{w_{\sigma \to \tau}},
\]
where $W\left(\mu_{\sigma}^{\mathrm{in}},\, \mu_{\tau}^{\mathrm{out}}\right)$ is the (directed) Wasserstein distance between the measures $\mu_{\sigma}^{\mathrm{in}}$ and $\mu_{\tau}^{\mathrm{out}}$, defined by:
\[
W\left(\mu_{\sigma}^{\mathrm{in}},\, \mu_{\tau}^{\mathrm{out}}\right) = \inf_{\pi \in \Pi(\mu_{\sigma}^{\mathrm{in}}, \mu_{\tau}^{\mathrm{out}})} \sum_{\xi, \eta \in \mathcal{S}} \pi(\xi, \eta) d(\xi, \eta),
\]
where $\Pi(\mu_{\sigma}^{\mathrm{in}}, \mu_{\tau}^{\mathrm{out}})$ is the set of all joint probability measures on $\mathcal{S} \times \mathcal{S}$ with marginals $\mu_{\sigma}^{\mathrm{in}}$ and $\mu_{\tau}^{\mathrm{out}}$, and $d(\xi, \eta)$ is the distance from $\xi$ (incoming neighbor of $\sigma$) to $\eta$ (outgoing neighbor of $\tau$) in the graph $\mathcal{G}$ (the shortest directed path distance based on edge weights).
\end{remark}

\begin{remark}
\label{remark:orc-sensitivity}
Assume that the edge $(\sigma \to \tau) \in \mathcal{E}$ exists and has non-zero weight $w_{\sigma \to \tau} > 0$ in the influence graph $\mathcal{G}(\mathcal{S}, \mathbf{B})$. Then, the sensitivity bound for information flow from $\tau$ to $\sigma$ from Lemma~\ref{lemma: jacobian bound higher order} can be related to the Ollivier-Ricci curvature $k(\sigma, \tau)$ of the edge $(\sigma \to \tau)$ as follows:
\[
\left\Vert \frac{\partial \mathbf{h}_{\sigma}^{(2)}}{\partial \mathbf{h}_{\tau}^{(0)}} \right\Vert_{1} \leq \left( \prod_{\ell=0}^{1} \alpha^{(\ell)} \beta^{(\ell)} \right) w_{\tau}^{\mathrm{out}} w_{\sigma}^{\mathrm{in}} \left(1 - \dfrac{w_{\sigma \to \tau}}{w_{\mathrm{max}}^3} (1 - k(\sigma, \tau))\right),
\]
where $w_{\mathrm{max}}^3$ is the maximum weighted $3$-step path from incoming neighbors of $\sigma$ to outgoing neighbors of $\tau$. This result indicates that lower Ollivier-Ricci curvature (smaller $k(\sigma, \tau)$) leads to reduced sensitivity, thereby contributing to oversquashing. The connection leverages the existence of the reversed edge $(\sigma \to \tau)$ in addition to the edge $(\tau \to \sigma)$ over which information flows, which is not guaranteed to be the case in a generic relational message passing scheme, but is the case for undirected graphs where $k(\sigma, \tau)= k(\tau, \sigma)$. This result aligns with Theorem~4.5 from \citet{pmlr-v202-nguyen23c} but without requiring their assumption of linearity in the message and update functions. Analyzing oversquashing for connections beyond $2$-steps requires stronger assumptions, such as the influence graph $\mathcal{G}(\mathcal{S}, \mathbf{B})$ being strongly connected. These assumptions present obstructions to analyzing oversquashing in general relational structures with Ollivier-Ricci curvature.
\end{remark}

We provide the proof in Appendix~\ref{appendix-proofs-local-geometry}.

\subsection{Remarks for Section~\ref{subsec:impact-of-hidden-dimensions}}
\label{appendix-remarks-subsec-impact-of-hidden-dimensions}

\begin{remark}[Lipschitz Constants for MLP Message and Update Functions]
\label{remark:bounded_jacobian_special_cases}
Consider the following message and update functions at layer $t$:

\textbf{Message function}:
\begin{equation*}
\boldsymbol{\psi}_{i}^{(t)}\left( \mathbf{h}_\sigma^{(t)}, \mathbf{h}_{\xi_1}^{(t)}, \dots, \mathbf{h}_{\xi_{n_i-1}}^{(t)} \right) = \mathbf{W}_{i}^{(t)} \begin{bmatrix}
\mathbf{h}_\sigma^{(t)} \\ \mathbf{h}_{\xi_1}^{(t)} \\ \vdots \\ \mathbf{h}_{\xi_{n_i-1}}^{(t)}
\end{bmatrix},
\end{equation*}
where $\mathbf{W}_{i}^{(t)}$ is a weight matrix of appropriate dimensions, and $\begin{bmatrix} \cdot \end{bmatrix}$ denotes column-wise concatenation.

\textbf{Update function}:
\begin{equation*}
\boldsymbol{\phi}^{(t)}\left( \mathbf{m}_{\sigma,1}^{(t)}, \dots, \mathbf{m}_{\sigma,k}^{(t)} \right) = \mathbf{f}\left( \mathbf{W}^{(t)} \begin{bmatrix}
\mathbf{g}\left( \mathbf{m}_{\sigma,1}^{(t)} \right) \\ \vdots \\ \mathbf{g}\left( \mathbf{m}_{\sigma,k}^{(t)} \right)
\end{bmatrix} \right),
\end{equation*}
where $\mathbf{f}$ and $\mathbf{g}$ are the component-wise applications of non-linear functions $f$ and $g$ with bounded derivatives $C_f$ and $C_g$, respectively. I.e., $f, g: \mathbb{R} \to \mathbb{R}$, and $|f^\prime(x)| \leq C_f$ and $|g^\prime(x)| \leq C_g$ for all $x$.

Assume that the entries of all weight matrices $\mathbf{W}_{i}^{(t)}$ and $\mathbf{W}^{(t)}$ are bounded in absolute value by a constant $C_w > 0$. Then, the Lipschitz constants $\beta_{i}^{(t)}$ and $\alpha^{(t)}$ satisfy:
\begin{equation*}
\beta_{i}^{(t)} \leq C_w p_{i,t},
\end{equation*}
and
\begin{equation*}
\alpha^{(t)} \leq C_w C_f C_g p_{t+1} ,
\end{equation*}
where $p_{i,t}$ is the dimension of the message vector $\mathbf{m}_{\sigma,i}^{(t)}$, and $p_{t+1}$ is the output dimension of the update function $\boldsymbol{\phi}_t$.
\end{remark} 

We provide the proof in Appendix~\ref{appendix-proof-hidden-dimensions}.

\section{Proofs}
\label{appendix:proofs}

\subsection{Proofs for Section~\ref{subsec:sensitivity-analysis}}
\label{appendix-subsec-proofs-sensitivity-analysis}

\begin{proof}[Proof of Lemma~\ref{lemma: jacobian bound higher order}]
First, we compute the Jacobian of Equation~\ref{eq: update rule higher order}:
\begin{align*}
\frac{\partial \mathbf{h}_\sigma^{(s+1)}}{\partial \mathbf{h}_\tau^{(0)}}
&= \sum_{i=1}^k \frac{\partial \boldsymbol{\phi}^{(s)}}{\partial \mathbf{m}_i^{(s)}} \frac{\partial \mathbf{m}_{\sigma,i}^{(s)}}{\partial \mathbf{h}_\tau^{(0)}} \\
&= \sum_{i=1}^k \left( \frac{\partial \boldsymbol{\phi}^{(s)}}{\partial \mathbf{m}_i^{(s)}} \right) \sum_{\boldsymbol{\xi} \in \mathcal{S}^{n_i - 1}} \mathbf{A}^{R_i}_{\sigma, \boldsymbol{\xi}} \left[ \frac{\partial \boldsymbol{\psi}_{i}^{(s)}}{\partial \mathbf{h}_\sigma^{(s)}} \frac{\partial \mathbf{h}_\sigma^{(s)}}{\partial \mathbf{h}_\tau^{(0)}} + \sum_{j=1}^{n_i-1} \frac{\partial \boldsymbol{\psi}_{i}^{(s)}}{\partial \mathbf{h}_{\xi_j}^{(s)}} \frac{\partial \mathbf{h}_{\xi_j}^{(s)}}{\partial \mathbf{h}_\tau^{(0)}} \right].
\end{align*}
By the submultiplicative and additive properties of the induced 1-norm (maximum absolute column sum) and the boundedness of the Jacobians of the functions $\boldsymbol{\phi}^{(s)}$ and $\boldsymbol{\psi}_{i}^{(s)}$ (Assumption~\ref{assump:bounded_jacobian}), we get:
\begin{align*}
\left\Vert \frac{\partial \mathbf{h}_\sigma^{(s+1)}}{\partial \mathbf{h}_\tau^{(0)}} \right\Vert_1
&\leq \sum_{i=1}^k \left\Vert \frac{\partial \boldsymbol{\phi}^{(s)}}{\partial \mathbf{m}_i^{(s)}} \right\Vert_1 \sum_{\boldsymbol{\xi} \in \mathcal{S}^{n_i - 1}} \mathbf{A}^{R_i}_{\sigma, \boldsymbol{\xi}} \left[ \left\Vert \frac{\partial \boldsymbol{\psi}_{i}^{(s)}}{\partial \mathbf{h}_\sigma^{(s)}} \right\Vert_1 \left\Vert \frac{\partial \mathbf{h}_\sigma^{(s)}}{\partial \mathbf{h}_\tau^{(0)}} \right\Vert_1 + \sum_{j=1}^{n_i-1} \left\Vert \frac{\partial \boldsymbol{\psi}_{i}^{(s)}}{\partial \mathbf{h}_{\xi_j}^{(s)}} \right\Vert_1 \left\Vert \frac{\partial \mathbf{h}_{\xi_j}^{(s)}}{\partial \mathbf{h}_\tau^{(0)}} \right\Vert_1 \right] \\
&\leq \alpha^{(s)} \sum_{i=1}^k \beta_{i}^{(s)} \sum_{\boldsymbol{\xi} \in \mathcal{S}^{n_i - 1}} \mathbf{A}^{R_i}_{\sigma, \boldsymbol{\xi}} \left[ \left\Vert \frac{\partial \mathbf{h}_\sigma^{(s)}}{\partial \mathbf{h}_\tau^{(0)}} \right\Vert_1 + \sum_{j=1}^{n_i-1} \left\Vert \frac{\partial \mathbf{h}_{\xi_j}^{(s)}}{\partial \mathbf{h}_\tau^{(0)}} \right\Vert_1 \right] \\
&\leq \alpha^{(s)} \beta^{(s)} \sum_{i=1}^k \sum_{\boldsymbol{\xi} \in \mathcal{S}^{n_i - 1}} \mathbf{A}^{R_i}_{\sigma, \boldsymbol{\xi}} \left[ \left\Vert \frac{\partial \mathbf{h}_\sigma^{(s)}}{\partial \mathbf{h}_\tau^{(0)}} \right\Vert_1 + \sum_{j=1}^{n_i-1} \left\Vert \frac{\partial \mathbf{h}_{\xi_j}^{(s)}}{\partial \mathbf{h}_\tau^{(0)}} \right\Vert_1 \right] \\
&\leq \alpha^{(s)} \beta^{(s)} \left[ \gamma \left\Vert \frac{\partial \mathbf{h}_\sigma^{(s)}}{\partial \mathbf{h}_\tau^{(0)}} \right\Vert_1 + \sum_{i=1}^k \sum_{\boldsymbol{\xi} \in \mathcal{S}^{n_i - 1}} \mathbf{A}^{R_i}_{\sigma, \boldsymbol{\xi}} \sum_{j=1}^{n_i-1} \left\Vert \frac{\partial \mathbf{h}_{\xi_j}^{(s)}}{\partial \mathbf{h}_\tau^{(0)}} \right\Vert_1 \right].
\end{align*}
Here, we used that the entries of $\mathbf{A}^{R_i}$ are nonnegative, and $\sum_{i=1}^k \sum_{\boldsymbol{\xi} \in \mathcal{S}^{n_i - 1}} \mathbf{A}^{R_i}_{\sigma, \boldsymbol{\xi}} \leq \gamma$.

We now prove the lemma using induction on the layer $t$. For the base case $t = 1$, we get
\begin{align*}
\left\Vert \frac{\partial \mathbf{h}_\sigma^{(1)}}{\partial \mathbf{h}_\tau^{(0)}} \right\Vert_1
&\leq \alpha^{(0)} \beta^{(0)} \left[ \gamma \left\Vert \frac{\partial \mathbf{h}_\sigma^{(0)}}{\partial \mathbf{h}_\tau^{(0)}} \right\Vert_1 + \sum_{i=1}^k \sum_{\boldsymbol{\xi} \in \mathcal{S}^{n_i - 1}} \mathbf{A}^{R_i}_{\sigma, \boldsymbol{\xi}} \sum_{j=1}^{n_i-1} \left\Vert \frac{\partial \mathbf{h}_{\xi_j}^{(0)}}{\partial \mathbf{h}_\tau^{(0)}} \right\Vert_1 \right] \\
&= \alpha^{(0)} \beta^{(0)} \left[ \gamma \mathbf{I}_{\sigma,\tau} + \sum_{i=1}^k \sum_{\boldsymbol{\xi} \in \mathcal{S}^{n_i - 1}} \mathbf{A}^{R_i}_{\sigma, \boldsymbol{\xi}} \sum_{j=1}^{n_i-1} \mathbf{I}_{\xi_j,\tau} \right] \\ 
&= \alpha^{(0)} \beta^{(0)} \left[ \gamma \mathbf{I}_{\sigma,\tau} + \sum_{i=1}^k \tilde{\mathbf{A}}_{\sigma,\tau}^{R_i} \right] \\
&= \alpha^{(0)} \beta^{(0)} \left( \gamma \mathbf{I} + \sum_{i=1}^k \tilde{\mathbf{A}}^{R_i} \right)_{\sigma, \tau} \\
&= \alpha^{(0)} \beta^{(0)} \left(\mathbf{B}^1\right)_{\sigma,\tau},
\end{align*}
where we used that $\tilde{\mathbf{A}}^{R_i}_{\sigma, \tau} = \sum_{j=1}^{n_i-1} \sum_{\boldsymbol{\xi} \in \mathcal{S}^{n_i - 2}} \mathbf{A}^{R_i}_{\sigma, \xi_1, ..., \xi_{j-1}, \tau, \xi_j, ..., \xi_{n_i-2}}$. This proves the base case.

For the induction step, assume the bound holds for $t$. We now compute:
\begin{align*}
\left\Vert \frac{\partial \mathbf{h}_\sigma^{(t+1)}}{\partial \mathbf{h}_\tau^{(0)}} \right\Vert_1
&\leq \alpha^{(t)} \beta^{(t)} \left[ \gamma \left\Vert \frac{\partial \mathbf{h}_\sigma^{(t)}}{\partial \mathbf{h}_\tau^{(0)}} \right\Vert_1 + \sum_{i=1}^k \sum_{\boldsymbol{\xi} \in \mathcal{S}^{n_i - 1}} \mathbf{A}^{R_i}_{\sigma, \boldsymbol{\xi}} \sum_{j=1}^{n_i-1} \left\Vert \frac{\partial \mathbf{h}_{\xi_j}^{(t)}}{\partial \mathbf{h}_\tau^{(0)}} \right\Vert_1 \right] \\
&\leq \left(\prod_{\ell=0}^t \alpha^{(\ell)}\beta^{(\ell)}\right) \left[ \gamma \left(\mathbf{B}^{t} \right)_{\sigma,\tau} + \sum_{i=1}^k \sum_{\boldsymbol{\xi} \in \mathcal{S}^{n_i - 1}} \mathbf{A}^{R_i}_{\sigma, \boldsymbol{\xi}} \sum_{j=1}^{n_i-1} \left( \mathbf{B}^{t} \right)_{\xi_j,\tau} \right] \\
&= \left(\prod_{\ell=0}^t \alpha^{(\ell)}\beta^{(\ell)}\right) \left[ \left(\gamma \mathbf{I}\mathbf{B}^{t} \right)_{\sigma,\tau} + \left( \left( \sum_{i=1}^k \tilde{\mathbf{A}}^{R_i} \right) \mathbf{B}^{t} \right)_{\sigma, \tau} \right] \\
&= \left(\prod_{\ell=0}^t \alpha^{(\ell)}\beta^{(\ell)}\right) \left(\left(\gamma \mathbf{I} + \sum_{i=1}^k \tilde{\mathbf{A}}^{R_i} \right) \mathbf{B}^{t} \right)_{\sigma, \tau} \\
&= \left(\prod_{\ell=0}^t \alpha^{(\ell)}\beta^{(\ell)}\right) \left( \mathbf{B}^{t+1} \right)_{\sigma, \tau} .
\end{align*}
To see this, note that:
\begin{align*}
\left( \left( \sum_{i=1}^k \tilde{\mathbf{A}}^{R_i} \right) \mathbf{B}^t \right)_{\sigma, \tau}
&= \sum_{i=1}^k \sum_{\nu \in \mathcal{S}} \tilde{\mathbf{A}}^{R_i}_{\sigma, \nu} \left(\mathbf{B}^t\right)_{\nu, \tau} \\
&= \sum_{i=1}^k \sum_{j=1}^{n_i-1} \sum_{\boldsymbol{\xi} \in \mathcal{S}^{n_i - 2}} \sum_{\nu \in \mathcal{S}} \mathbf{A}^{R_i}_{\sigma, \xi_1, ..., \xi_{j-1}, \nu, \xi_j, ..., \xi_{n_i-2}} \left(\mathbf{B}^t\right)_{\nu, \tau} \\
&= \sum_{i=1}^k \sum_{\boldsymbol{\xi} \in \mathcal{S}^{n_i - 1}} \mathbf{A}^{R_i}_{\sigma, \boldsymbol{\xi}} \sum_{j=1}^{n_i-1} \left(\mathbf{B}^t\right)_{\xi_j, \tau}.
\end{align*}
This completes the proof.
\end{proof}

\subsection{Proofs for Section~\ref{subsec:impact-of-local-geometry} and Appendix~\ref{appendix-remarks-subsec=impact-of-local-geometry}}
\label{appendix-subsec-proofs-impact-of-local-geometry}
\label{appendix-proofs-local-geometry}

\begin{proof}[Proof of Proposition~\ref{prop:local-geometry}]
In the influence graph $\mathcal{G}(\mathcal{S}, \mathbf{B})$, the edge weights correspond to the entries of $\mathbf{B}$, that is, $w_{\tau \to \sigma} = \mathbf{B}_{\sigma, \tau}$. Therefore, the $(\sigma, \tau)$ entry of $\mathbf{B}^2$ is
\[
(\mathbf{B}^2)_{\sigma, \tau} = \sum_{\xi \in \mathcal{S}} \mathbf{B}_{\sigma, \xi} \mathbf{B}_{\xi, \tau} = \sum_{\xi \in \mathcal{S}} w_{\xi \to \sigma} \cdot w_{\tau \to \xi} = w_T.
\]
From Lemma~\ref{lemma: jacobian bound higher order} with $t = 2$, we have
\[
\left\Vert \frac{\partial \mathbf{h}_{\sigma}^{(2)}}{\partial \mathbf{h}_{\tau}^{(0)}} \right\Vert_{1} \leq \left( \prod_{\ell=0}^{1} \alpha^{(\ell)} \beta^{(\ell)} \right) (\mathbf{B}^2)_{\sigma, \tau} = \left( \prod_{\ell=0}^{1} \alpha^{(\ell)} \beta^{(\ell)} \right) w_T.
\]
Rewriting the curvature formula to solve for $w_T$, and since $w_F \geq 0$, we get
\begin{align*}
w_T &= \frac{1}{3} \left( \mathrm{EFC}_{\mathcal{G}}(\tau, \sigma) + w_{\tau}^{\mathrm{out}} + w_{\sigma}^{\mathrm{in}} - 4 - 2 w_F \right) \\
    &\leq \frac{1}{3} \left( \mathrm{EFC}_{\mathcal{G}}(\tau, \sigma) + w_{\tau}^{\mathrm{out}} + w_{\sigma}^{\mathrm{in}} - 4 \right).
\end{align*}
Substituting back, we obtain the desired inequality. This completes the proof.
\end{proof}

\begin{proof}[Proof of Remark~\ref{remark:orc-sensitivity}]
Consider the influence graph $\mathcal{G}(\mathcal{S}, \mathbf{B}) = (\mathcal{S}, \mathcal{E}, w)$ derived from the matrix $\mathbf{B}$, where the edge weights correspond to the entries of $\mathbf{B}$: $w_{\xi \to \eta} = \mathbf{B}_{\eta, \xi}$. Assume that the reversed edge $(\sigma \to \tau) \in \mathcal{E}$ exists with weight $w_{\sigma \to \tau} > 0$. From Lemma~\ref{lemma: jacobian bound higher order} with $t = 2$, we have:
\[
\left\Vert \frac{\partial \mathbf{h}_{\sigma}^{(2)}}{\partial \mathbf{h}_{\tau}^{(0)}} \right\Vert_{1} \leq \left( \prod_{\ell=0}^{1} \alpha^{(\ell)} \beta^{(\ell)} \right) (\mathbf{B}^2)_{\sigma, \tau} = \left( \prod_{\ell=0}^{1} \alpha^{(\ell)} \beta^{(\ell)} \right) \sum_{\xi \in \mathcal{S}} w_{\xi \to \sigma} \cdot w_{\tau \to \xi} = \left( \prod_{\ell=0}^{1} \alpha^{(\ell)} \beta^{(\ell)} \right) w_T,
\]
where $w_T = \sum_{\xi} w_{\tau \to \xi} \cdot w_{\xi \to \sigma}$.

We construct a \emph{transference plan} to transport mass from $\mu_{\sigma}^{\mathrm{in}}$ to $\mu_{\tau}^{\mathrm{out}}$: We do not move $\frac{w_T}{w_\tau^\mathrm{out} w_\sigma^\mathrm{in}}$, but move the rest of the mass with a cost at most $w_\mathrm{max}^3$ from the incoming neighbors of $\sigma$ to the outgoing neighbors of $\tau$, and we get:
\[
W\left(\mu_{\sigma}^{\mathrm{in}},\, \mu_{\tau}^{\mathrm{out}}\right) \leq \left(1 - \frac{w_T}{w_{\sigma}^{\mathrm{in}} w_{\tau}^{\mathrm{out}}}\right) w_{\mathrm{max}}^3.
\]
The Ollivier-Ricci curvature of the edge $(\sigma \to \tau)$ is:
\[
k(\sigma, \tau) = 1 - \frac{W\left(\mu_{\sigma}^{\mathrm{in}},\, \mu_{\tau}^{\mathrm{out}}\right)}{w_{\sigma \to \tau}}.
\]
Substituting and rearranging, we obtain:
\[
w_T \leq w_{\sigma}^{\mathrm{in}} w_{\tau}^{\mathrm{out}} \left(1 - \dfrac{w_{\sigma \to \tau}}{w_{\mathrm{max}}^3} (1 - k(\sigma, \tau))\right).
\]
We thus get:
\[
\left\Vert \dfrac{\partial \mathbf{h}_{\sigma}^{(2)}}{\partial \mathbf{h}_{\tau}^{(0)}} \right\Vert_{1} \leq \left( \prod_{\ell=0}^{1} \alpha^{(\ell)} \beta^{(\ell)} \right) w_T \leq \left( \prod_{\ell=0}^{1} \alpha^{(\ell)} \beta^{(\ell)} \right) w_{\tau}^{\mathrm{out}} w_{\sigma}^{\mathrm{in}} \left(1 - \dfrac{w_{\sigma \to \tau}}{w_{\mathrm{max}}^3} (1 - k(\sigma, \tau))\right).
\]
This completes the proof.
\end{proof}

\subsection{Proofs for Section~\ref{subsec:impact-of-depth}}
\label{appendix-subsec-proofs-impact-of-depth}

\begin{proof}[Proof of Theorem~\ref{theorem: impact of depth}]
We start with the bound from Lemma \ref{lemma: jacobian bound higher order} and expand the right-hand side using the binomial theorem:

\begin{eqnarray*}
\left\|\frac{\partial \mathbf{h}_\sigma^{(r+m)}}{\partial \mathbf{h}_\tau^{(0)}}\right\|_1 &\leq& \left(\prod_{\ell=0}^{r+m-1} \alpha^{(\ell)} \beta^{(\ell)}\right)(\gamma \mathbf{I} + \tilde{\mathbf{A}})^{r+m}_{\sigma,\tau} \\
&=& \left(\prod_{\ell=0}^{r+m-1} \alpha^{(\ell)} \beta^{(\ell)}\right) \sum_{i=0}^{r+m} \binom{r+m}{i} \gamma^{r+m-i} \left(\tilde{\mathbf{A}}^i\right)_{\sigma,\tau}.
\end{eqnarray*}

Since the combinatorial distance from $\tau$ to $\sigma$ in the graph $\mathcal{G}(\mathcal{S}, \tilde{\mathbf{A}})$ is $r$, the first $r-1$ terms of the sum vanish:
\begin{equation*}
\left\|\frac{\partial \mathbf{h}_\sigma^{(r+m)}}{\partial \mathbf{h}_\tau^{(0)}}\right\|_1 \leq \left(\prod_{\ell=0}^{r+m-1} \alpha^{(\ell)} \beta^{(\ell)}\right) \sum_{i=r}^{r+m} \binom{r+m}{i} \gamma^{r+m-i} \left(\tilde{\mathbf{A}}^i\right)_{\sigma,\tau}.
\end{equation*}

Using $\tilde{\mathbf{A}}^i_{\sigma,\tau} \leq \omega_i(\sigma,\tau) M^i \leq \omega_{r+m}(\sigma,\tau) M^i$ and letting $q = i-r$:

\begin{eqnarray*}
\left\|\frac{\partial \mathbf{h}_\sigma^{(r+m)}}{\partial \mathbf{h}_\tau^{(0)}}\right\|_1 &\leq& \left(\prod_{\ell=0}^{r+m-1} \alpha^{(\ell)} \beta^{(\ell)}\right) M^r \omega_{r+m}(\sigma,\tau) \sum_{q=0}^m \binom{r+m}{r+q} \gamma^{m-q} M^q.
\end{eqnarray*}

We can bound $\binom{r+m}{r+q}$ as follows:

\begin{align*}
    \binom{r+m}{r+q} &= \frac{(r+m)(r-1+m)\cdots(1+m)}{(r+q)(r-1+q)\cdots(1+q)}\binom{m}{q} \leq \frac{(r+m)(r-1+m)\cdots(1+m)}{r!}\binom{m}{q}  \\
    &\leq \left(1 + \frac{m}{r}\right)\cdots \left(1+\frac{m}{1}\right) \binom{m}{q} \leq \left(1+\frac{m}{m+1}\right)^{r-m}(1+m)^{m}\binom{m}{q}.
\end{align*}

Substituting this bound, we get:

\begin{eqnarray*}
\left\|\frac{\partial \mathbf{h}_\sigma^{(r+m)}}{\partial \mathbf{h}_\tau^{(0)}}\right\|_1
&\leq& \left(\prod_{\ell=0}^{r+m-1} \alpha^{(\ell)} \beta^{(\ell)}\right) M^r \omega_{r+m}(\sigma,\tau) \left(1+\frac{m}{m+1}\right)^{r-m}(1+m)^m \sum_{q=0}^m \binom{m}{q} \gamma^{m-q} M^q \\
&=& \left(\prod_{\ell=0}^{r+m-1} \alpha^{(\ell)} \beta^{(\ell)}\right) M^r \omega_{r+m}(\sigma,\tau) \left(1+\frac{m}{m+1}\right)^{r-m}(1+m)^m (\gamma+M)^m \\
&=& \left(\prod_{\ell=r}^{r+m-1} \alpha^{(\ell)} \beta^{(\ell)}\right) \left(\frac{(1+m)^2}{2m+1} (\gamma + M)\right)^m \omega_{r+m}(\sigma,\tau)  \left(\prod_{\ell=0}^{r-1} \alpha^{(\ell)} \beta^{(\ell)}\right) \left(\left(1 + \frac{m}{m+1}\right) M\right)^r.
\end{eqnarray*}

Using $1+\frac{m}{m+1} \leq 2$, $M \leq k$, and $\gamma \leq k$:

\begin{equation*}
\left\|\frac{\partial \mathbf{h}_\sigma^{(r+m)}}{\partial \mathbf{h}_\tau^{(0)}}\right\|_1 \leq \left(\prod_{\ell=r}^{r+m-1} \alpha^{(\ell)} \beta^{(\ell)}\right) (2k(1+m))^m \omega_{r+m}(\sigma,\tau) \left(\prod_{\ell=0}^{r-1} \alpha^{(\ell)} \beta^{(\ell)}\right) \left(2M\right)^r.
\end{equation*}

Define $C = (\alpha_{\text{max}} \beta_{\text{max}})^m (2k(1+m))^m$. This depends only on $\alpha_{\text{max}}$, $\beta_{\text{max}}$, $k$, and $m$.

Finally, we can write:
\begin{equation}
\left\|\frac{\partial \mathbf{h}_\sigma^{(r+m)}}{\partial \mathbf{h}_\tau^{(0)}}\right\|_1 \leq C \omega_{r+m}(\sigma,\tau) (\alpha_{\text{max}} \beta_{\text{max}})^r (2M)^r.
\end{equation}

This completes the proof.
\end{proof}

\subsection{Proofs for Section~\ref{subsec:impact-of-hidden-dimensions} and Appendix~\ref{appendix-remarks-subsec-impact-of-hidden-dimensions}}
\label{appendix-subsec-proofs-impact-of-hidden-dimensions}
\label{appendix-proof-hidden-dimensions}

\begin{proof}[Proof of Remark \ref{remark:bounded_jacobian_special_cases}]
We derive the bounds for $\beta_{i}^{(t)}$ and $\alpha^{(t)}$ separately.

\textbf{Derivation of $\beta_{i}^{(t)}$:}

The message function is linear:
\begin{equation*}
\boldsymbol{\psi}_{i}^{(t)}\left( \mathbf{h}_\sigma^{(t)}, \mathbf{h}_{\xi_1}^{(t)}, \dots, \mathbf{h}_{\xi_{n_i-1}}^{(t)} \right) = \mathbf{W}_{i}^{(t)} \mathbf{H}_{\sigma,\xi}^{(t)},
\end{equation*}
where $\mathbf{H}_{\sigma,\xi}^{(t)} = \begin{bmatrix}
\mathbf{h}_\sigma^{(t)} \\ \mathbf{h}_{\xi_1}^{(t)} \\ \vdots \\ \mathbf{h}_{\xi_{n_i-1}}^{(t)}
\end{bmatrix}$ is the concatenated feature vector.

We need to compute the Lipschitz constant $\beta_{i}^{(t)}$ of the message function $\boldsymbol{\psi}_{i}^{(t)}$ with respect to each neighbor's feature vector $\mathbf{h}_{\xi_j}^{(t)}$.

The Jacobian of $\boldsymbol{\psi}_{i}^{(t)}$ with respect to $ \mathbf{h}_{\xi_j}^{(t)}$ is:
\begin{equation*}
\frac{\partial \boldsymbol{\psi}_{i}^{(t)}}{\partial \mathbf{h}_{\xi_j}^{(t)}} = \mathbf{W}_{i}^{(t)(:,\, \mathcal{I}_j)},
\end{equation*}
where $\mathbf{W}_{i}^{(t)(:,\, \mathcal{I}_j)}$ denotes the columns of $\mathbf{W}_{i}^{(t)}$ corresponding to $\mathbf{h}_{\xi_j}^{(t)}$.

Since the entries of $\mathbf{W}_{i}^{(t)}$ are bounded by $C_w$, and $\mathbf{h}_{\xi_j}^{(t)} \in \mathbb{R}^{p_t}$, the matrix $\frac{\partial \boldsymbol{\psi}_{i}^{(t)}}{\partial \mathbf{h}_{\xi_j}^{(t)}}$ is of size $p_{i,t} \times p_t$ with entries bounded by $C_w$.

Using the induced matrix 1-norm:
\begin{equation*}
\left\| \frac{\partial \boldsymbol{\psi}_{i}^{(t)}}{\partial \mathbf{h}_{\xi_j}^{(t)}} \right\|_1 = \max_{1 \leq l \leq p_t} \sum_{k=1}^{p_{i,t}} \left| \left( \frac{\partial \boldsymbol{\psi}_{i}^{(t)}}{\partial \mathbf{h}_{\xi_j}^{(t)}} \right)_{k,l} \right|.
\end{equation*}

Since each entry $\left| \left( \frac{\partial \boldsymbol{\psi}_{i}^{(t)}}{\partial \mathbf{h}_{\xi_j}^{(t)}} \right)_{k,l} \right| \leq C_w$, the sum over $k$ is bounded by $C_w p_{i,t}$. Therefore,
\begin{equation*}
\left\| \frac{\partial \boldsymbol{\psi}_{i}^{(t)}}{\partial \mathbf{h}_{\xi_j}^{(t)}} \right\|_1 \leq C_w p_{i,t}.
\end{equation*}

Thus, the Lipschitz constant $\beta_{i}^{(t)}$ satisfies:
\begin{equation*}
\beta_{i}^{(t)} \leq C_w p_{i,t}.
\end{equation*}

\textbf{Derivation of $\alpha^{(t)}$:}

The update function is given by:
\begin{equation*}
\boldsymbol{\phi}^{(t)}\left( \mathbf{m}_{\sigma,1}^{(t)}, \dots, \mathbf{m}_{\sigma,k}^{(t)} \right) = \mathbf{f}\left( \mathbf{W}^{(t)} \mathbf{M}_\sigma^{(t)} \right),
\end{equation*}
where $\mathbf{M}_\sigma^{(t)} = \begin{bmatrix}
\mathbf{g}\left( \mathbf{m}_{\sigma,1}^{(t)} \right) \\ \vdots \\ \mathbf{g}\left( \mathbf{m}_{\sigma,k}^{(t)} \right)
\end{bmatrix}$.

We need to compute the Lipschitz constant $\alpha^{(t)}$ of the update function $\boldsymbol{\phi}^{(t)}$ with respect to each input message $\mathbf{m}_{\sigma,i}^{(t)}$.

First, compute the Jacobian of $\boldsymbol{\phi}^{(t)}$ with respect to $\mathbf{m}_{\sigma,i}^{(t)}$:
\begin{equation*}
\frac{\partial \boldsymbol{\phi}^{(t)}}{\partial \mathbf{m}_{\sigma,i}^{(t)}} = \frac{\partial \boldsymbol{\phi}^{(t)}}{\partial \mathbf{M}_\sigma^{(t)}} \frac{\partial \mathbf{M}_\sigma^{(t)}}{\partial \mathbf{m}_{\sigma,i}^{(t)}}.
\end{equation*}

Compute $\frac{\partial \mathbf{M}_\sigma^{(t)}}{\partial \mathbf{m}_{\sigma,i}^{(t)}}$:
\begin{equation*}
\frac{\partial \mathbf{M}_\sigma^{(t)}}{\partial \mathbf{m}_{\sigma,i}^{(t)}} = \begin{bmatrix}
\mathbf{0} \\ \vdots \\ \operatorname{diag}\left( g^\prime \left( \mathbf{m}_{\sigma,i}^{(t)} \right) \right) \\ \vdots \\ \mathbf{0}
\end{bmatrix},
\end{equation*}
where the non-zero block $\operatorname{diag}\left( g^\prime\left( \mathbf{m}_{\sigma,i}^{(t)} \right) \right)$ is at position $i$.

Compute $\frac{\partial \boldsymbol{\phi}^{(t)}}{\partial \mathbf{M}_\sigma^{(t)}}$:
\begin{equation*}
\frac{\partial \boldsymbol{\phi}^{(t)}}{\partial \mathbf{M}_\sigma^{(t)}} = \operatorname{diag}\left( f^\prime\left( \mathbf{W}^{(t)} \mathbf{M}_\sigma^{(t)} \right) \right) \mathbf{W}^{(t)}.
\end{equation*}

Therefore,
\begin{equation*}
\frac{\partial \boldsymbol{\phi}^{(t)}}{\partial \mathbf{m}_{\sigma,i}^{(t)}} = \operatorname{diag}\left( f^\prime\left( \mathbf{W}^{(t)} \mathbf{M}_\sigma^{(t)} \right) \right) \mathbf{W}^{(t)(:,\, \mathcal{I}_i)} \operatorname{diag}\left( g^\prime\left( \mathbf{m}_{\sigma,i}^{(t)} \right) \right),
\end{equation*}
where $\mathbf{W}^{(t)(:,\, \mathcal{I}_i)}$ denotes the columns of $\mathbf{W}^{(t)}$ corresponding to $\mathbf{m}_{\sigma,i}^{(t)}$.

The matrix $\frac{\partial \boldsymbol{\phi}^{(t)}}{\partial \mathbf{m}_{\sigma,i}^{(t)}}$ is of size $p_{t+1} \times p_{i,t}$.

Since $|f^\prime(x)| \leq C_f$, $|g^\prime(x)| \leq C_g$, and $|(\mathbf{W}^{(t)})_{k,l}| \leq C_w$, the entries of $\frac{\partial \boldsymbol{\phi}^{(t)}}{\partial \mathbf{m}_{\sigma,i}^{(t)}}$ are bounded by $C_w C_f C_g$.

Using the induced matrix 1-norm:
\begin{equation*}
\left\| \frac{\partial \boldsymbol{\phi}^{(t)}}{\partial \mathbf{m}_{\sigma,i}^{(t)}} \right\|_1 = \max_{1 \leq l \leq p_{i,t}} \sum_{k=1}^{p_{t+1}} \left| \left( \frac{\partial \boldsymbol{\phi}^{(t)}}{\partial \mathbf{m}_{\sigma,i}^{(t)}} \right)_{k,l} \right|.
\end{equation*}

Each column $l$ sums over $k$ up to $p_{t+1}$ entries, each bounded by $C_w C_f C_g$. Therefore,
\begin{equation*}
\left\| \frac{\partial \boldsymbol{\phi}^{(t)}}{\partial \mathbf{m}_{\sigma,i}^{(t)}} \right\|_1 \leq C_w C_f C_g p_{t+1}.
\end{equation*}

Thus, the Lipschitz constant $\alpha^{(t)}$ satisfies:
\begin{equation*}
\alpha^{(t)} \leq C_w C_f C_g p_{t+1}.
\end{equation*}

\end{proof}

\section{Supplementary Analyses}
\label{appendix:supplementary-analyses}

\subsection{Graph Lifting Example and Curvature}
\label{appendix-graph-lifting-curvature}

As shown in Section~\ref{sec:experiments:curvature_distribution}, the edge curvature distribution is non-trivially impacted by the graph lifting procedure. We explore this more in this section. 

One possible explanation for the general positive distribution shift is the widening of bottleneck regions as well as the addition of many nodes and edges in densely connected regions. We observe this qualitatively in Figure~\ref{fig:long_dumbbell_complex}. We can see that the narrow path is now widened into two nodes instead of one. 
We also see that the clique regions gain a lot more nodes and edges than the path. The Ollivier-Ricci curvature becomes much more red, while the balanced Forman curvature and augmented Forman curvature maintain a higher number of small, negatively curved edges. This example is consistent with the observations from Figure~\ref{fig:kde_summary}. 

In the augmented Forman curvature plot, we can see that the curvature for edge \texttt{DI} becomes more negative. While this edge may not propagate information very well due to its negative curvature, there are now many edges for information to flow around this edge. This provides some qualitative evidence that incorporating global structure into analysis of graph lifting could present an important future direction.

\begin{figure}
    \centering
    \begin{subfigure}[b]{\textwidth}
        \centering
        \includegraphics[scale=.4]{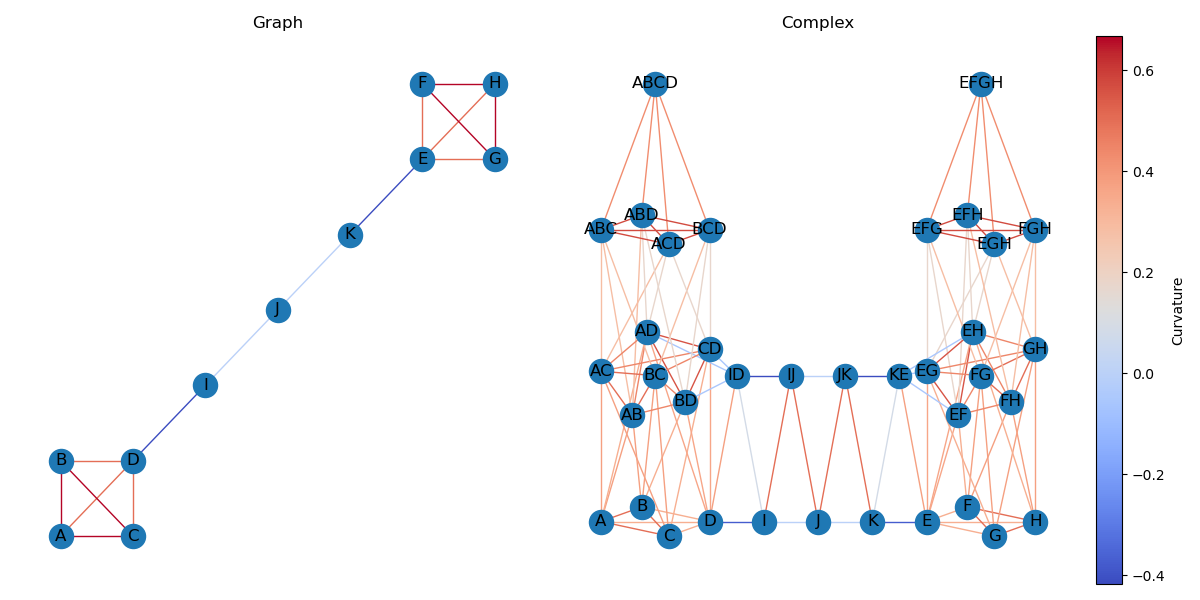}
        \caption{ORC}
    \end{subfigure}
    \begin{subfigure}[b]{\textwidth}
        \centering
        \includegraphics[scale=.4]{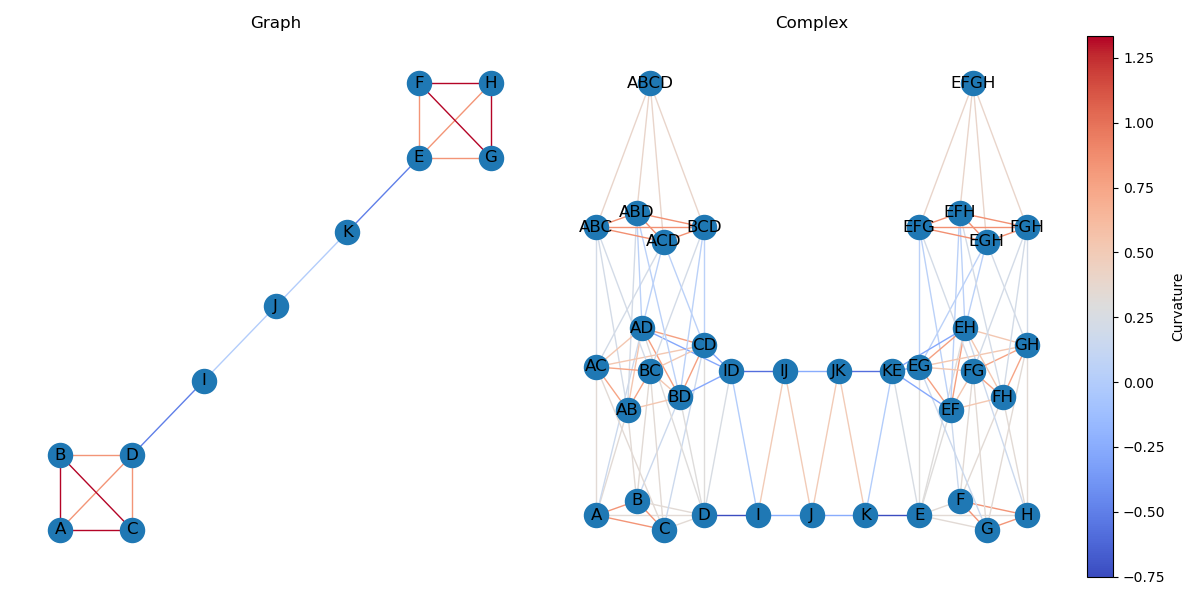}
        \caption{BFC}
    \end{subfigure}
    \begin{subfigure}[b]{\textwidth}
        \centering
        \includegraphics[scale=.4]{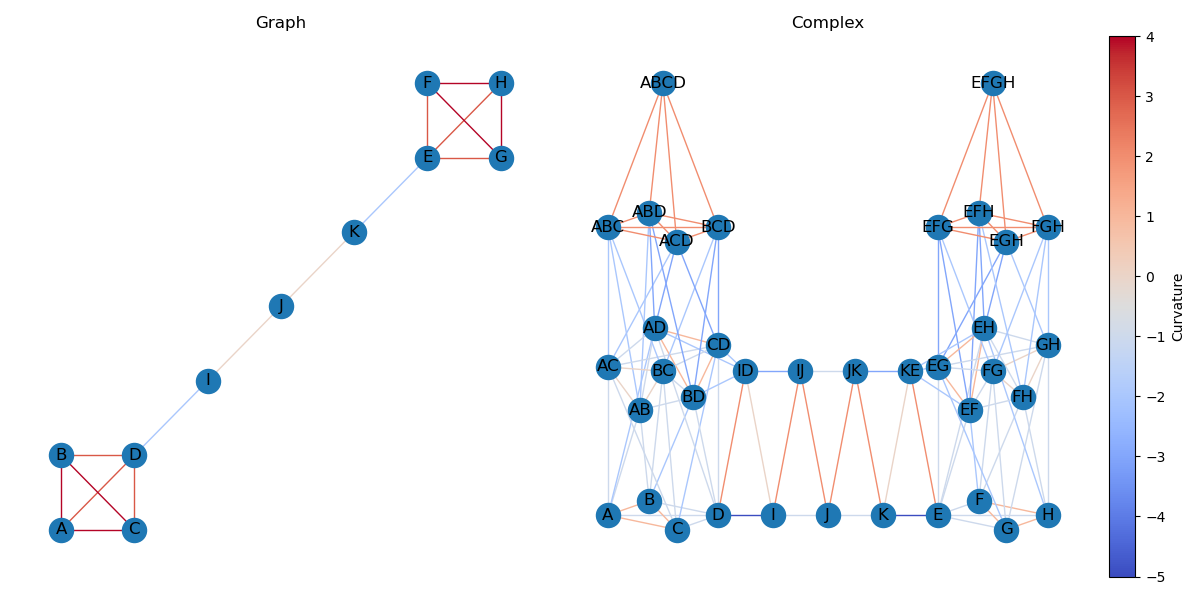}
        \caption{AFC}
    \end{subfigure}
    
    \caption{Long dumbbell graph before and after lifting to its clique complex. Edges are colored based on their curvature: Ollivier Ricci curvature (top), balanced Forman curvature (middle), and augmented Forman curvature (bottom). Note that this uses Boundary, Co-boundary, Lower, and Upper relations. These are presented as edges for visual clarity. 
    }
    \label{fig:long_dumbbell_complex}
\end{figure}

\subsection{Graph Lifting and Weighted Curvature}
\label{appendix:graph_lift_wc}

The addition of many nodes and edges, as well as the shift in the curvature distribution, makes direct comparisons between graphs and their corresponding lifts more challenging. While this widened bottleneck may alleviate oversquashing, measures like algebraic connectivity may not measure this effect since widened bottlenecks are counteracted by the addition of many nodes. Similarly, average curvature could be biased by the addition of many positively curved edges located in densely connected regions. We instead propose the \emph{betweenness weighted curvature}
\begin{equation}
    \text{wc} = \sum_{e\in E(G)} \text{bc}(e)\ \text{curv}(e),
\end{equation}
where bc denotes betwenness centrality and curv denotes the Ollivier-Ricci curvature. This measure, originally in
introduced by \citet{münch2022olliviercurvaturebetweennesscentrality}, places more weight on bottleneck edges which gives a weighted average of the curvature in the graph. This measure has the added benefit that it captures both local information propagation through curvature as well as global information propagation through betweenness centrality. To account for the influence of more positively weighted edges, we also consider the \textit{negative betweenness weighted curvature}
\begin{equation}
    \text{nwc} = \sum_{e\in E(G)} \text{bc}(e)\ \text{curv}(e) \chi_{\text{curv}(e) < 0},
\end{equation}
where $\chi$ is the indicator function. 

In Figure~\ref{fig:weighted_curvature}, we create a scatter plot of the weighted curvature of a graph and its corresponding clique complex: $(\text{wc}(G), \text{wc}(C))$. We can see that the weighted curvature generally becomes less negative after complex construction for Ollivier-Ricci and balanced Forman curvature, which may suggest that oversquashing is alleviated. However, the augmented Forman curvature makes the weighted curvature more negative for all graphs in the MUTAG dataset. This is consistent with Figure~\ref{fig:kde_summary} which produced a small number of negatively curved edges for augmented Forman curvature. Interestingly, there is a strong linear correlation in each of the scatter plots. Further study is required to understand this trend. 

\begin{figure}
    \centering
    \begin{subfigure}[b]{0.45\textwidth}
        \centering
        \includegraphics[width=\textwidth]{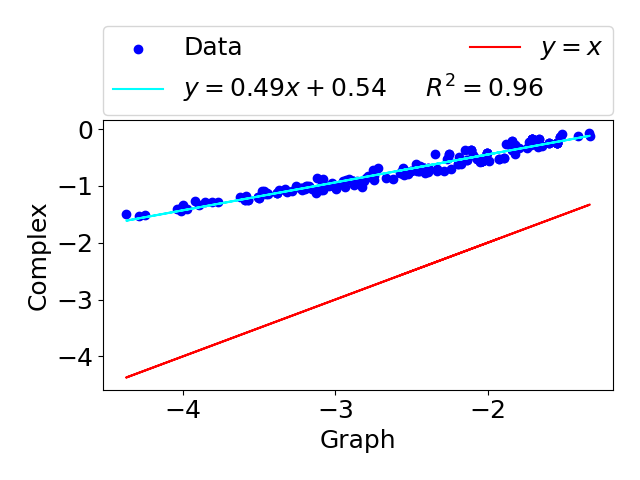}
        \caption{ORC WC}
    \end{subfigure}
    \begin{subfigure}[b]{0.45\textwidth}
        \centering
        \includegraphics[width=\textwidth]{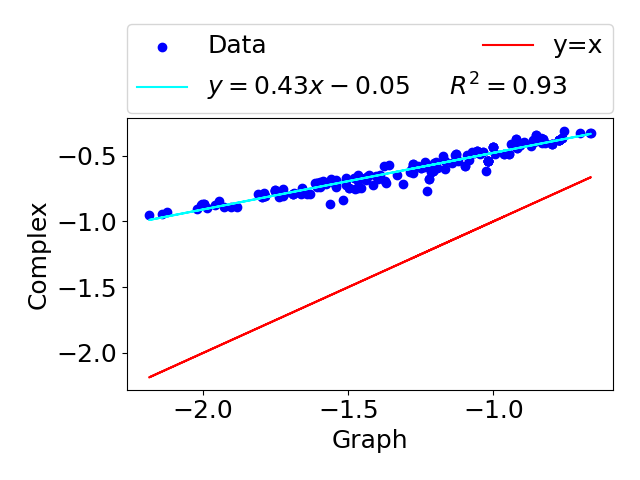}
        \caption{ORC NWC}
    \end{subfigure}
    \begin{subfigure}[b]{0.45\textwidth}
        \centering
        \includegraphics[width=\textwidth]{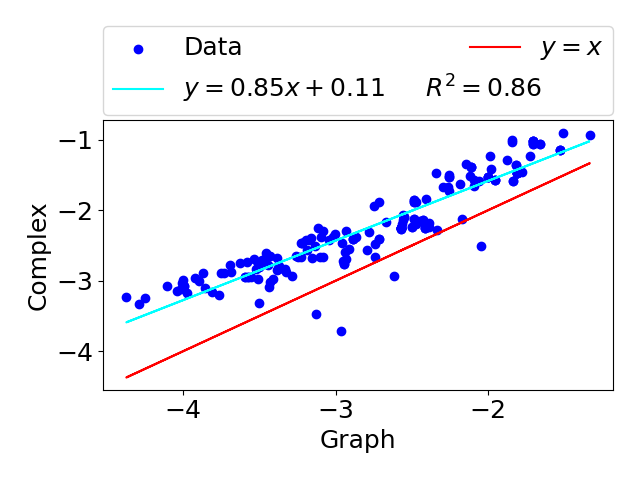}
        \caption{BFC WC}
    \end{subfigure}
    \begin{subfigure}[b]{0.45\textwidth}
        \centering
        \includegraphics[width=\textwidth]{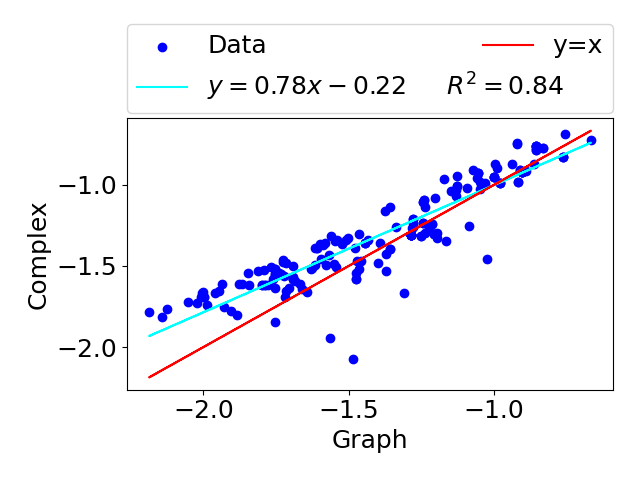}
        \caption{BFC NWC}
    \end{subfigure}
    \begin{subfigure}[b]{0.45\textwidth}
        \centering
        \includegraphics[width=\textwidth]{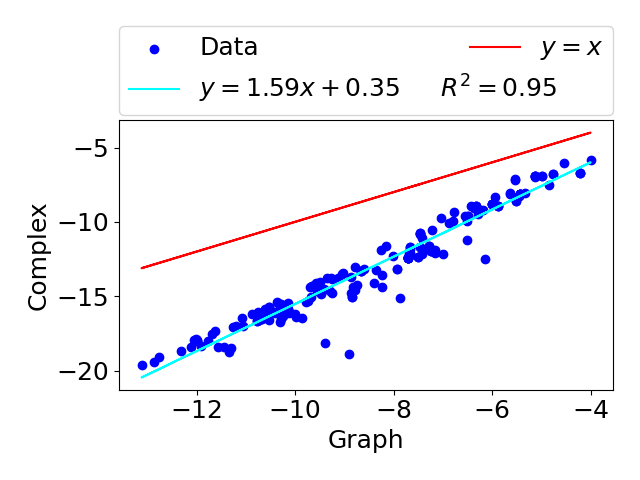}
        \caption{AFC WC}
    \end{subfigure}
    \begin{subfigure}[b]{0.45\textwidth}
        \centering
        \includegraphics[width=\textwidth]{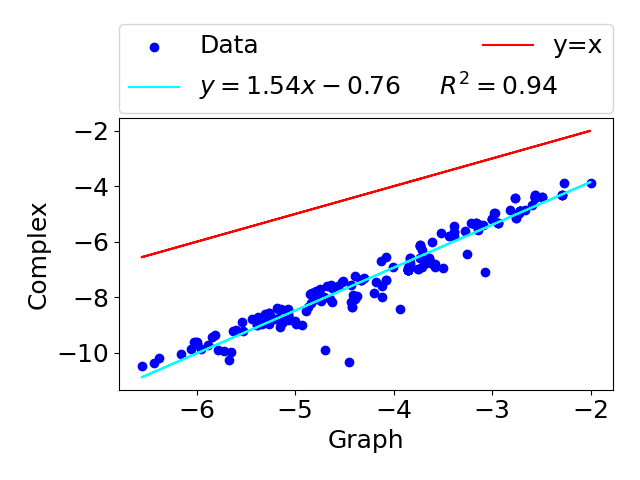}
        \caption{AFC NWC}
    \end{subfigure}
    \caption{Weighted curvature scatter plots for MUTAG: Ollivier Ricci curvature (top), balanced Forman curvature (middle), and augmented Forman curvature (bottom); betweeness weighted curvature (left) and negative betweenness weighted curvature (right).}
    \label{fig:weighted_curvature}
\end{figure}

\subsection{Graph Lifting and Edge Curvature Distribution}
\label{sec:experiments:curvature_distribution}

\begin{figure}[htbp]
    \centering
    \begin{subfigure}[b]{0.32\textwidth}
        \centering
        \includegraphics[width=\textwidth]{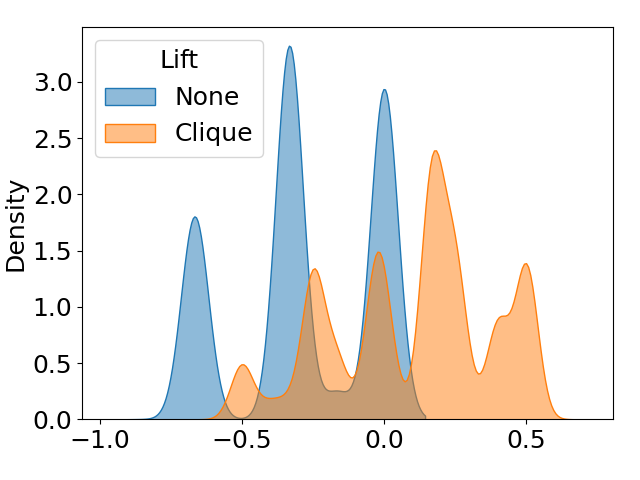}
        \caption{ORC}
    \end{subfigure}
    \begin{subfigure}[b]{0.32\textwidth}
        \centering
        \includegraphics[width=\textwidth]{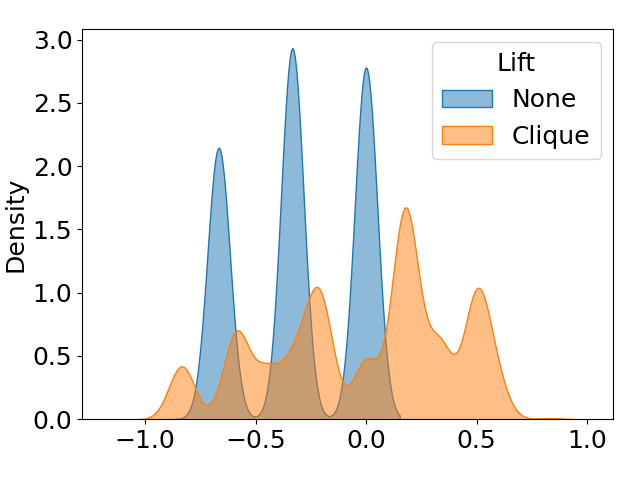}
        \caption{BFC}
    \end{subfigure}
    \begin{subfigure}[b]{0.32\textwidth}
        \centering
        \includegraphics[width=\textwidth]{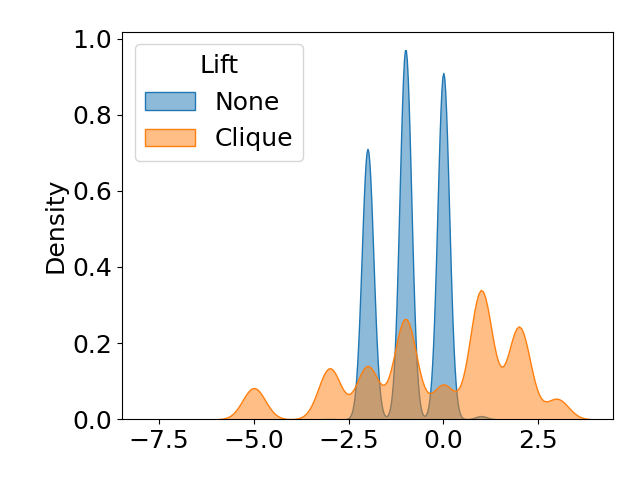}
        \caption{AFC}
    \end{subfigure}
    \caption{Edge curvature distribution across all graphs in the MUTAG dataset: Ollivier-Ricci curvature (left), Balanced Forman curvature (middle), and Augmented Forman curvature (right).}
    \label{fig:kde_summary}
\end{figure}

Understanding the impact on information propagation of lifting a graph to a clique complex is crucial. Graph lifting often adds numerous edges and nodes, complicating theoretical analysis, as the structural differences between the relational structures corresponding to a graph and its clique complex can be significant. As a first step, we analyze the edge curvature distribution to gain empirical insight into the clique graph lift. Figure~\ref{fig:kde_summary} shows the kernel density estimate (KDE) of edge curvature on the MUTAG dataset. Ollivier-Ricci curvature \citep{pmlr-v202-nguyen23c} increases uniformly, while balanced Forman curvature \citep{topping2022understanding} generally increases with some edges becoming slightly more negative. A more pronounced shift is observed with augmented Forman curvature \citep{fesser2023mitigating}, but the overall effect on information propagation remains unclear, as global graph structure is not captured in this analysis.

\subsection{Synthetic Benchmark: \textsc{NeighborsMatch}}
\label{appendix:neighbors-match}

The neighbors match experiment is a graph transfer task introduced by \cite{alon2021on} to test oversquashing in GNNs. We build on the implementation of \cite{karhadkar2023fosr} which adapts this benchmark to test rewiring. Each graph in the dataset is a path of cliques as shown in Figure~\ref{fig:nmatch_graph}. Each of the green nodes is assigned a distinct random label which is represented by a one-hot encoding. The root node, colored red, is also assigned a random label which is encoded using a one-hot encoding. The goal is for the neural network to learn which green node has the same one-hot encoding as the root node based solely on the output of the neural network at the root node. 

\begin{figure}
    \centering
    \includegraphics[width=0.5\linewidth]{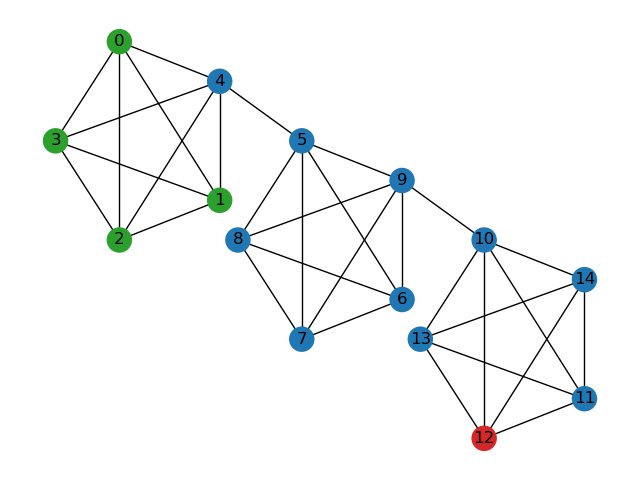}
    \caption{NeighborsMatch Graph with 3 cliques of 5 nodes. }
    \label{fig:nmatch_graph}
\end{figure}

In this experiment, we choose the number of nodes in each clique to be 5 to avoid excessively large clique complexes when performing graph lifting. Note that we see similar trends as in the RingTransfer experiment. However, this creates a much larger clique complex which may present a greater challenge for rewiring methods. We can see in Figure~\ref{fig:nmatch} that the Clique and Ring lifts require more rewiring iterations to achieve the same performance as their graph counterparts. This is likely related to the lifts having many more nodes and edges than the original graph. The original graph has 15 nodes and 32 edges, whereas the clique lift has 77 nodes and 764 edges and the ring lift has 77 nodes and 854 edges.

\begin{figure}[h]
    \centering
    \includegraphics[width=0.5\linewidth]{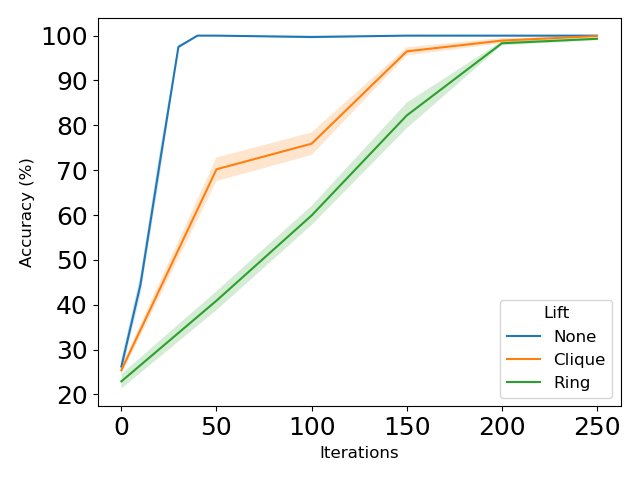}
    \caption{Neighbors Match Rewiring experiment. }
    \label{fig:nmatch}
\end{figure}

{
We also benchmark trees and trees with cycles attached to the leaves on the NeighborsMatch benchmark (see Figure~\ref{fig:trees}). We note that a tree as an acyclic connected graph has Betti numbers $b_0 = 1$ (number of connected components) and $b_1 = 0$ (number of independent loops or cycles). On the other hand, a tree with $n$ attached cycles has Betti numbers $b_0 = 1$ and $b_1 = n$. As such, the two structures are topologically very distinct when considered as 1-simplicial complexes: the tree is contractible and has a trivial fundamental group, while the tree with attached cycles is not contractible and has a non-trivial fundamental group corresponding to a free group with $n$ generators. However, their relational message passing structures are quite similar, and they demonstrate almost identical prototypical oversquashing trends (see Figure~\ref{fig:trees-oversquashing}).}

\begin{figure}[ht]
    \centering
    \begin{subfigure}[b]{0.45\textwidth}
        \centering
        \includegraphics[width=\textwidth]{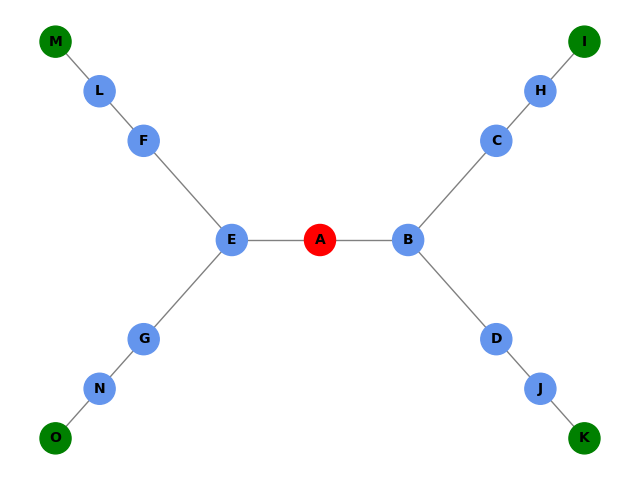}
        \caption{Tree with no attached cycles}
        \label{fig:subfig1}
    \end{subfigure}
    \hfill
    \begin{subfigure}[b]{0.45\textwidth}
        \centering
        \includegraphics[width=\textwidth]{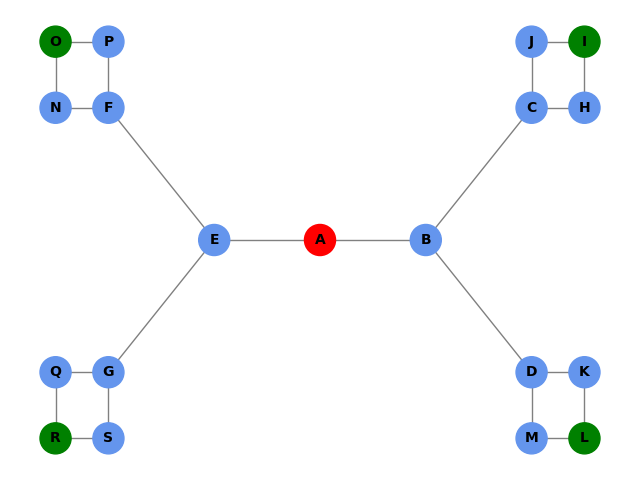}
        \caption{Tree with attached cycles}
        \label{fig:subfig2}
    \end{subfigure}
    \caption{Tree without and with attached cycles have very distinct topologies.}
    \label{fig:trees}
\end{figure}

\begin{figure}[ht]
    \centering
    \includegraphics[width=0.5\linewidth]{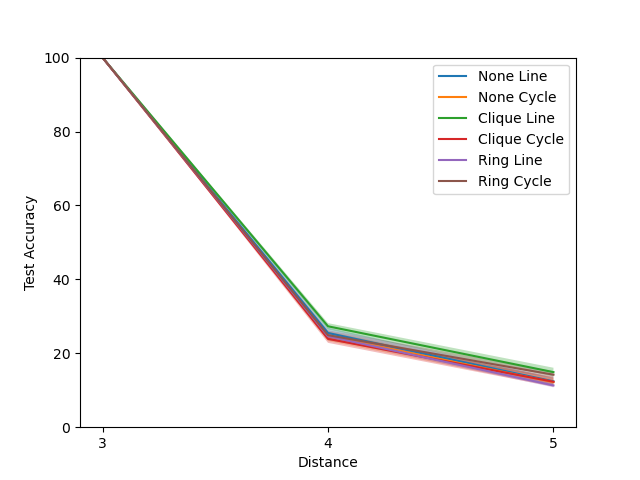}
    \caption{Neighbors Match Rewiring experiment on trees and trees with cycles, with increasing tree depths. ``None'', ``Clique'', and ``Ring'' indicate the graph lifting applied. ``Line'' and ``Cycle'' indicate whether lines or cycles are attached to the leaves.}
    \label{fig:trees-oversquashing}
\end{figure}

\subsection{Real-World Benchmark: Graph Regression (Zinc)}

\begin{table}[ht]
\begin{tabular}{llllll}
\toprule
 &  & GIN & R-GIN & SIN & CIN++\\
Complex & Rewiring &  \\
\midrule
\multirow[c]{2}{*}{None} 
    & None & $0.405 \pm 0.001$ & $0.368 \pm 0.003$ & $0.370 \pm 0.003$ & $0.339 \pm 0.005$\\
    & FoSR & $1.056 \pm 0.010$ & $1.485 \pm 1.045$ & $0.377 \pm 0.005$ & $0.374 \pm 0.008$ \\
 \midrule
\multirow[c]{2}{*}{Clique} 
& None & $0.409 \pm 0.006$ & $0.326 \pm 0.003$ & $0.331 \pm 0.004$ & $0.320 \pm 0.006$\\
& FoSR & $0.844 \pm 0.011$ & $0.420 \pm 0.002$ & $0.369 \pm 0.003$ & $0.379 \pm 0.008$\\
 \midrule
\multirow[c]{2}{*}{Ring} 
& None & \cellcolor{gold}$0.332 \pm 0.006$ & \cellcolor{gold}$0.215 \pm 0.003$ & $0.223 \pm 0.004$ & $0.214 \pm 0.003$\\
& FoSR & $0.641 \pm 0.011$ & $0.232 \pm 0.002$ & \cellcolor{gold}$0.212 \pm 0.003$ & \cellcolor{gold}$0.211 \pm 0.002$\\
\bottomrule
\end{tabular}
\caption{ZINC (MAE). Lower values are better. }
\label{tab:zinc}
\end{table}

{
\label{appendix:zinc}
In this section we test our methods on the ZINC dataset. Due to increased computational demands, we do not perform a hyperparameter sweep for this experiment and instead use the hyperparameters from the TUDataset experiments. We also test this on the ring lifting procedure. The ring lift consists of adding cycles in the graph as 2-cells in the complex as in \cite{bodnar2021weisfeiler}. In this work, we restrict the rings to have size at most 7. Structures like carbon rings are important in molecule datasets and prediction tasks can benefit from encoding larger rings in the lifting procedure. The results are presented in Table~\ref{tab:zinc}, where we see a clear improvement from the Ring lift across all models and rewirings tested for the ZINC dataset. 
}

\subsection{Real-World Benchmark: Node Classification}
\label{appendix:node-classification}
In this section, we test the lifting and rewiring methods on node classification tasks. The results of this experiment are shown in Table~\ref{tab:node_class}. The rewirings considered included FoSR and the pruning algorithm introduced in Appendix~\ref{apdx:prune}. The best results for each dataset and model are highlighted in gold. Notably, GIN and CIN++ benefited from lifting on all datasets tested. 

\begin{table}[ht]
\centering
\scriptsize
\begin{subtable}[h]{\textwidth}
\centering
\caption{\textbf{GIN}}
\label{tab:node_class_gin}
\begin{tabular}{lllllll}
\toprule
 &  & CORNELL & TEXAS & WISCONSIN & CITESEER & CORA\\
Complex & Rewiring &  &  &  \\
\midrule
\multirow[c]{3}{*}{None} & FoSR & $41.6 \pm 3.5$ & $57.9 \pm 2.7$ & $49.6 \pm 2.5$ & $69.3 \pm 0.8$ & $82.4 \pm 0.7$\\
 & Prune & $45.8 \pm 4.0$ & $55.3 \pm 4.9$ & $48.1 \pm 3.9$ & $70.6 \pm 0.4^*$ & $84.4 \pm 0.8^*$ \\
 & None & $33.7 \pm 3.3$ & $60.0 \pm 5.2$ & $49.2 \pm 2.9$ & $69.9 \pm 0.8$ & $82.7 \pm 0.9$\\
\midrule
\multirow[c]{3}{*}{Clique} & FoSR & $40.5 \pm 3.5$ & $64.7 \pm 5.3$ & \cellcolor{gold}$56.5 \pm 2.6$ & $68.6 \pm 0.8$ & \cellcolor{gold}$84.4 \pm 0.6$\\
 & Prune & $36.8 \pm 2.6$ & $61.6 \pm 3.3$ & $47.3 \pm 3.7$ & \cellcolor{gold}$72.0 \pm 0.8^*$ & $82.4 \pm 0.5^*$ \\
 & None & $44.7 \pm 3.6$ & $58.9 \pm 2.3$ & $55.8 \pm 2.2$ & $69.0 \pm 0.8$ & $83.2 \pm 0.6$\\
 \midrule
\multirow[c]{3}{*}{Ring} & FoSR & $45.3 \pm 2.1$ & $61.1 \pm 3.3$ & $55.0 \pm 3.3$ & & \\
 & Prune & \cellcolor{gold}$47.4 \pm 4.1$ & \cellcolor{gold}$65.3 \pm 4.5$ & $54.2 \pm 3.2^*$ & & \\
 & None  & $46.3 \pm 3.1$ & $58.9 \pm 3.3$ & $54.6 \pm 3.0$ &  & \\
\bottomrule
\end{tabular}
\end{subtable}

\begin{subtable}[h]{\textwidth}
\centering
\caption{\textbf{RGIN}}
\label{tab:node_class_rgin}
\begin{tabular}{lllllll}
\toprule
 &  & CORNELL & TEXAS & WISCONSIN & CITESEER & CORA\\
Complex & Rewiring &  &  &  \\
\midrule
\multirow[c]{3}{*}{None} & FoSR & $54.2 \pm 3.3$ & $51.1 \pm 3.3$ & $71.2 \pm 2.4$ & $73.2 \pm 0.6$ & $81.9 \pm 0.8$\\
& Prune & $52.6 \pm 3.5$ & $65.3 \pm 3.2$ & $68.5 \pm 2.1$ & $73.2 \pm 0.6^*$ & $85.6 \pm 0.7^*$ \\
 & None & \cellcolor{gold}$59.5 \pm 5.0$ & \cellcolor{gold}$67.9 \pm 3.2$ & $69.2 \pm 2.8$ & $72.2 \pm 1.0$ & $86.4 \pm 0.7$\\
\midrule
\multirow[c]{3}{*}{Clique} & FoSR & $51.1 \pm 3.8$ & $62.1 \pm 3.7$ & $68.5 \pm 3.3$ & \cellcolor{gold}$73.4 \pm 1.0$ & $85.7 \pm 0.6$ \\
 & Prune & $47.9 \pm 2.3$ & $64.2 \pm 3.6$ & $60.4 \pm 2.0$ & $73.3 \pm 0.6^*$ & $85.9 \pm 0.6^*$\\
 & None & $50.5 \pm 3.5$ & $67.4 \pm 3.4$ & $68.5 \pm 3.8$ & $72.1 \pm 0.9$ & \cellcolor{gold}$86.6 \pm 0.6$\\
 \midrule
\multirow[c]{3}{*}{Ring} & FoSR & $57.4 \pm 2.3$ & $69.5 \pm 3.6$ & $66.5 \pm 2.4$ & & \\
 & Prune & $54.2 \pm 4.1$ & $59.5 \pm 4.7$ & \cellcolor{gold}$74.6 \pm 2.5^*$ & & \\
 & None  & $56.3 \pm 2.7$ & $62.1 \pm 4.1$ & $67.7 \pm 2.2$ & & \\

\bottomrule
\end{tabular}
\end{subtable}

\begin{subtable}[h]{\textwidth}
\centering
\caption{\textbf{SIN}}
\label{tab:node_class_sin}
\begin{tabular}{lllllll}
\toprule
 &  & CORNELL & TEXAS & WISCONSIN & CITESEER & CORA\\
Complex & Rewiring &  &  &  \\
\midrule
\multirow[c]{3}{*}{None} & FoSR & $30.5 \pm 3.5$ & $34.7 \pm 4.1$ & $30.8 \pm 3.2$ & $54.3 \pm 1.1$ & $50.6 \pm 2.1$\\
 & Prune & \cellcolor{gold}$44.7 \pm 4.2$ & $55.8 \pm 4.9$ & $53.1 \pm 2.9$ & $66.9 \pm 0.6^*$ & $81.7 \pm 0.8^*$\\
 & None & $42.6 \pm 2.9$ & $52.6 \pm 4.3$ & $50.4 \pm 2.8$ & \cellcolor{gold}$66.9 \pm 0.8$ & \cellcolor{gold}$81.6 \pm 0.8$\\
\midrule
\multirow[c]{3}{*}{Clique} & FoSR & $36.3 \pm 2.7$ & $49.5 \pm 3.0$ & $47.3 \pm 3.0$ & $60.2 \pm 1.2$ & $68.2 \pm 0.9$\\
 & Prune & $40.5 \pm 2.9$ & $56.3 \pm 3.0$ & $51.5 \pm 3.1$ & $64.8 \pm 0.9^*$ & $75.6 \pm 1.2^*$\\
 & None & $36.8 \pm 2.9$ & $55.8 \pm 2.7$ & $45.8 \pm 3.8$ & $63.2 \pm 0.8$ & $77.7 \pm 0.7$\\
 \midrule
\multirow[c]{3}{*}{Ring} & FoSR & $33.7 \pm 4.1$ & $51.1 \pm 4.4$ & $53.1 \pm 3.5$ & & \\
 & Prune & $42.1 \pm 3.2$ & \cellcolor{gold}$58.4 \pm 2.8$ & $55.0 \pm 3.3^*$ & & \\
 & None  & $38.9 \pm 3.3$ & $54.7 \pm 2.2$ & \cellcolor{gold}$61.2 \pm 3.5$ & & \\
\bottomrule
\end{tabular}
\end{subtable}

\begin{subtable}[h]{\textwidth}
\centering
\caption{\textbf{CIN++}}
\label{tab:node_class_cin++}
\begin{tabular}{lllllll}
\toprule
 &  & CORNELL & TEXAS & WISCONSIN & CITESEER & CORA\\
Complex & Rewiring &  &  &  \\
\midrule
\multirow[c]{3}{*}{None} & FoSR & $41.6 \pm 2.0$ & $55.3 \pm 2.7$ & $51.9 \pm 4.1$ & $44.9 \pm 1.1$ & $64.0 \pm 1.8$\\
 & Prune & $42.6 \pm 3.5$ & $49.5 \pm 4.7$ & $61.2 \pm 2.3$ & $45.2 \pm 1.4^*$ & $55.7 \pm 1.8^*$ \\
 & None & $43.2 \pm 4.4$ & $59.5 \pm 2.6$ & $54.2 \pm 3.8$ & $44.7 \pm 1.7$ & $63.6 \pm 2.0$ \\
\midrule
\multirow[c]{3}{*}{Clique} & FoSR & $47.9 \pm 4.5$ & \cellcolor{gold}$66.3 \pm 3.7$ & $64.6 \pm 1.6$ & \cellcolor{gold}$64.4 \pm 1.0$ & $78.6 \pm 0.8$\\
 & Prune & \cellcolor{gold}$52.6 \pm 3.5$ & $57.9 \pm 3.2$ & $63.8 \pm 3.5$ & $60.5 \pm 0.7^*$ & $77.0 \pm 1.0^*$ \\
 & None & $50.5 \pm 2.7$ & $65.3 \pm 2.6$ & \cellcolor{gold}$66.2 \pm 2.8$ & $63.2 \pm 0.9$ & \cellcolor{gold}$79.4 \pm 0.7$\\
\midrule
\multirow[c]{3}{*}{Ring} & FoSR & $42.1 \pm 4.1$ & $65.3 \pm 3.3$ & $62.3 \pm 1.7$ & & \\
 & Prune & $47.9 \pm 3.7$ & $63.2 \pm 4.4$ & $56.9 \pm 2.5^*$ & & \\
 & None  & $54.2 \pm 4.2$ & $64.2 \pm 4.1$ & $63.1 \pm 3.0$ &  & \\

\bottomrule
\end{tabular}
\end{subtable}

\caption{Node classification experiments. Experiments with $*$ used 1d curvature for pruning. }
\label{tab:node_class}
\end{table}

\subsection{Ablation Tests}

We perform ablation studies to evaluate the effect of increasing the number of layers and rewiring iterations, as well as the effect of increasing the hidden dimensions and rewiring iterations, on the performance of CIN++ on the graph classification MUTAG dataset with FoSR rewiring. The results are presented in Tables~\ref{tab:ablation-layers} and ~\ref{tab:ablation-hidden}. This analysis complements the experiments in Section~\ref{sec:experiments:synthetic} by testing on real world data. 
We note that the results of the ablation studies for the MUTAG dataset do not show a clear pattern across all rewiring iterations. However, we broadly see that rewiring too much tends to hurt performance as they begin to oversmooth information. We also generally see that increasing the hidden dimension tends to improve performance. Due to time constraints and limited compute, we could only test one dataset and perform 10 trials for each set of hyperparamters. With more trials, the standard error will decrease and make the patterns more clear. \cite{pmlr-v202-nguyen23c} and \cite{fesser2024} performed similar ablations and considered 100 trials for each set of hyperparameters. As in the experiments from Section~\ref{sec:experiments} and the results reported in \cite{rusch2023survey} and \cite{pmlr-v202-nguyen23c}, we expect that increasing the hidden dimension will improve classification accuracy. Similarly, increasing layers will initially improve performance until the model becomes too deep and experiences oversmoothing. 

\begin{table}[ht]
    \centering
    \begin{tabular}{llllll}
\toprule
 & Rewire Iterations & 0 & 10 & 20 & 40 \\
Lift & Layers &  &  &  &  \\
\midrule
\multirow[t]{3}{*}{None} & 2 & $86.0 \pm 2.9$ & $85.0 \pm 2.4$ & $75.5 \pm 1.7$ & $85.5 \pm 3.1$ \\
 & 4 & $87.5 \pm 2.3$ & $81.5 \pm 2.6$ & $81.0 \pm 3.2$ & $77.5 \pm 2.9$ \\
 & 6 & $89.0 \pm 2.7$ & $82.5 \pm 2.8$ & $78.5 \pm 3.7$ & $80.0 \pm 2.8$ \\
\cline{1-6}
\multirow[t]{3}{*}{Clique} & 2 & $85.5 \pm 2.2$ & $84.0 \pm 1.8$ & $81.0 \pm 2.2$ & $83.5 \pm 3.0$ \\
 & 4 & $86.0 \pm 1.9$ & $78.5 \pm 3.6$ & $83.5 \pm 2.9$ & $79.0 \pm 2.1$ \\
 & 6 & $84.5 \pm 2.7$ & $80.5 \pm 2.0$ & $82.5 \pm 3.1$ & $80.5 \pm 2.3$ \\
\cline{1-6}
\multirow[t]{3}{*}{Ring} 
 & 2 & $84.5 \pm 1.9$ & $86.5 \pm 2.4$ & $86.5 \pm 2.6$ & $87.0 \pm 2.5$ \\
 & 4 & $86.5 \pm 2.6$ & $84.5 \pm 2.4$ & $80.5 \pm 3.0$ & $86.5 \pm 1.7$ \\
 & 6 & $82.0 \pm 1.9$ & $86.0 \pm 1.8$ & $86.0 \pm 2.7$ & $88.0 \pm 2.5$ \\
\bottomrule
\end{tabular}

    \caption{Impact of increasing the number of layers and rewiring iterations for FoSR on the MUTAG dataset with CIN++.}
    \label{tab:ablation-layers}
\end{table}

\begin{table}[ht]
    \centering
    \begin{tabular}{llllll}
\toprule
 & Rewire Iterations & 0 & 10 & 20 & 40 \\
Lift & Hidden Dimension &  &  &  &  \\
\midrule
\multirow[t]{3}{*}{None} & 16 & $81.5 \pm 2.5$ & $77.5 \pm 3.0$ & $78.5 \pm 3.7$ & $74.0 \pm 1.8$ \\
 & 32 & $83.0 \pm 2.9$ & $77.5 \pm 2.8$ & $84.5 \pm 2.6$ & $82.0 \pm 4.2$ \\
 & 64 & $85.0 \pm 1.8$ & $82.0 \pm 2.1$ & $83.0 \pm 1.5$ & $82.5 \pm 2.4$ \\
\cline{1-6}
\multirow[t]{3}{*}{Clique} & 16 & $87.5 \pm 3.0$ & $79.5 \pm 3.6$ & $77.0 \pm 2.6$ & $79.0 \pm 1.6$ \\
 & 32 & $83.5 \pm 2.2$ & $81.0 \pm 2.2$ & $80.5 \pm 4.1$ & $78.5 \pm 3.4$ \\
 & 64 & $89.0 \pm 2.6$ & $81.0 \pm 2.8$ & $75.5 \pm 3.8$ & $83.5 \pm 2.8$ \\
\cline{1-6}
\multirow[t]{3}{*}{Ring} 
 & 16 & $87.0 \pm 1.7$ & $89.0 \pm 2.4$ & $85.0 \pm 1.5$ & $85.0 \pm 1.1$ \\
 & 32 & $88.5 \pm 2.8$ & $83.5 \pm 1.8$ & $89.0 \pm 2.2$ & $83.5 \pm 1.3$ \\
 & 64 & $85.5 \pm 3.0$ & $85.0 \pm 3.2$ & $86.0 \pm 1.8$ & $84.5 \pm 1.9$ \\
\bottomrule
\end{tabular}

    \caption{Impact of increasing the hidden dimensions and rewiring iterations for FoSR on the MUTAG dataset with CIN++.}
    \label{tab:ablation-hidden}
\end{table}

\subsection{Simplex Pruning}
\label{apdx:prune}
In this section we analyze the impact of pruning edges in the influence graph. For the pruning algorithm, we consider removing the edge with the highest balanced Forman curvature. More precisely, we first convert the simplex into the influence graph and collapse multi-edges. Then we compute the balanced Forman curvature of each edge in the graph. Finally, we remove the edge with the highest curvature. This process is repeated 40 times. When removing an edge with higher arity $e = (\sigma, \tau, \delta)$, the entire higher-order edge is removed from the simplex. When an edge $(\sigma, \tau)$ with higher multiplicity is selected, all instances of that edge are removed from the graph. 

The results of this experiment are shown in Table~\ref{tab:mutag_prune}. In eight out of twelve experiments, the performance improves with some amount of pruning. Notably, GIN and CIN++ attain their highest performance with 40 pruning iterations. This suggests that rewiring algorithms that implement edge pruning and edge addition could be the most promising candidates for simplicial rewiring algorithms. This is in part due to the addition of many nodes and edges in the graph lifting procedure. 

\begin{table}[ht]
    \centering
    \begin{tabular}{llllll}
    \toprule
     & Prune Iterations & 0 & 10 & 20 & 40 \\
    Complex & Model & & &  &  \\
    \midrule
    \multirow[t]{4}{*}{None}
     & GIN & $83.0 \pm 3.1$ & $87.5 \pm 3.1$ & \cellcolor{gold}$89.0 \pm 3.1$ & $86.0 \pm 2.1$ \\
     & RGIN & $81.5 \pm 1.7$ & $84.0 \pm 2.8$ & \cellcolor{gold}$91.0 \pm 2.2$ & $86.0 \pm 1.2$ \\
     & SIN & \cellcolor{gold}$88.5 \pm 3.0$ & $83.0 \pm 2.5$ & $86.0 \pm 2.1$ & $86.5 \pm 2.5$ \\
      & CIN++ & $85.0 \pm 3.4$& \cellcolor{gold}$91.5 \pm 1.8$ & $88.0 \pm 2.0$ & $84.5 \pm 2.5$ \\
    \midrule
    \multirow[t]{4}{*}{Clique} & GIN &$83.0 \pm 2.8$ & $83.5 \pm 1.7$ & $82.5 \pm 3.1$ & \cellcolor{gold}$90.5 \pm 1.4$ \\
     & RGIN & \cellcolor{gold}$86.0 \pm 2.3$ & $81.0 \pm 2.4$ & $82.5 \pm 2.5$ & $85.0 \pm 1.7$ \\
     & SIN & \cellcolor{gold}$87.0\pm 3.2$ & $79.5 \pm 2.2$ & $83.5 \pm 2.2$ & $84.5 \pm 3.4$ \\
     & CIN++ & $90.5 \pm 2.2$ & $84.5 \pm 3.7$ & $78.5 \pm 3.2$ & \cellcolor{gold}$91.0 \pm 1.8$ \\
     \midrule
    \multirow[t]{4}{*}{Ring} 
     & GIN & \cellcolor{gold}$90.0 \pm 2.2$ & $87.5 \pm 3.2$ & $88.0 \pm 1.7$ & $85.5 \pm 1.6$ \\
     & RGIN & $87.0 \pm 2.3$ & \cellcolor{gold}$87.5 \pm 2.6$ & $86.5 \pm 2.8$ & $87.0 \pm 2.4$ \\
     & SIN & $87.5 \pm 1.5$ & $87.0 \pm 2.9$ & $84.0 \pm 2.4$ & \cellcolor{gold}$88.0 \pm 2.3$ \\
     & CIN++ & $87.5 \pm 2.6$ & $80.5 \pm 3.3$ & \cellcolor{gold}$88.0 \pm 2.0$ & $81.5 \pm 1.5$ \\
    \bottomrule
    \end{tabular}
    \caption{MUTAG pruning experiment. Removing edges with highest balanced Forman Curvature. }
    \label{tab:mutag_prune}
\end{table}

\section{Additional Experimental Details}
\label{appendix:experiments}

For all experiments, we split the dataset randomly using $80\%$ for training, $10\%$ for validation, and $10\%$ for testing. For classificiation, we use the cross entropy loss. Each experiment is run with early stopping using the validation set. Also, if there is no improvement for 10 iterations, the learning rate is decreased. For all rewiring methods, we only consider adding edges and no edge removal. Each rewiring method is run for 40 iterations. 

\subsection{Real-World Benchmark: Graph Classification on TUDataset}
\label{appendix:tudataset-full}
To train and evaluate these method on real world data, we use the TUDataset \citep{morris2020}. These graph classification tasks are commonly used to benchmark message passing neural networks. We have included summary statistics for MUTAG, NCI1, ENZYMES, PROTEINS, and IMDB-BINARY in Table~\ref{table:dataset_statistics}. 

\begin{table}[htbp]
\centering
\begin{tabular}{|l|r|r|r|r|r|}
\hline
\textbf{Dataset}     & \textbf{\# Graphs} & \textbf{Classes} & \textbf{Avg. Nodes} & \textbf{Avg. Edges} & \textbf{Features} \\ \hline
MUTAG       & 188       & 2 & 17.93 & 19.79 & 7  \\ \hline
NCI1        & 4110      & 2 & 29.87 & 32.30 & 37 \\  \hline
ENZYMES     & 600       & 6 & 32.63 & 62.14 & 3 \\  \hline
PROTEINS    & 1113      & 2 & 39.06 & 72.82 & 3  \\ \hline
IMDB-BINARY & 1000      & 2 & 19.77 & 96.53 & 0  \\ \hline
\end{tabular}
\caption{TUDataset Statistics. }
\label{table:dataset_statistics}
\end{table}

In Tables~\ref{tab:all-graph-models}, ~\ref{tab:all-relational-models} and ~\ref{tab:all-simplicial-models}, we provide all the results for Table~\ref{tab:test_acc_main} from Section~\ref{sec:experiments}. We note that in Table~\ref{tab:test_acc_main}, we only report the Lift=None results for FoSR, and not AFR4 or SDRF, for fair comparison with Lift=Clique and Lift=Ring, where we only report the results for FoSR. We note that for $24/32 = 75\%$ of the results on ENZYME, MUTAG, NCI1, and PROTEINS, the best performing rewiring algorithms for Lift=None and Lift=Clique agree. Similarly, the best rewiring algorithms for Lift=None and Lift=Ring agree in $17/32 = 53\%$ of the experiments. In Figure~\ref{fig:graph_vs_clique_comparison}, we compare model performance on graphs and the corresponding clique complexes, and note how graph, relational graph, and topological models all respond similarly to relational rewiring.

\begin{table}[htbp]
\centering
\scriptsize

\begin{subtable}[h]{\textwidth}
\centering
\caption{\textbf{SGC}}
\label{tab:sgc}
\begin{tabular}{lllllll}
\toprule
 &  & ENZYMES & IMDB-B & MUTAG & NCI1 & PROTEINS \\
Lift & Rewiring &  &  &  &  &  \\
\midrule
\multirow[c]{3}{*}{None} 
    & None & \cellcolor{babyblue}$18.3 \pm 1.2$ & \cellcolor{babyblue}$49.5 \pm 1.5$ & \cellcolor{babyblue}$64.5 \pm 5.8$ & \cellcolor{babyblue}$55.2 \pm 1.0$ & \cellcolor{babyblue}$62.2 \pm 1.4$ \\
    & AFR4 & $17.0 \pm 1.7$ & $48.6 \pm 2.3$ & $65.5 \pm 3.5$ & \cellcolor{babypink}$54.4 \pm 0.7$ & \cellcolor{babypink}$65.0 \pm 1.5$ \\
    & FoSR & $18.2 \pm 1.7$ & \cellcolor{babypink}$50.0 \pm 1.8$ & $52.5 \pm 6.6$ & $49.3 \pm 0.8$ & $64.8 \pm 1.1$ \\
    & SDRF & \cellcolor{babypink}$21.5 \pm 1.6$ & $49.7 \pm 1.7$ & \cellcolor{babypink}$70.0 \pm 2.6$ & $52.4 \pm 0.7$ & $59.9 \pm 1.4$ \\
    \cline{2-7}
\multirow[c]{3}{*}{Clique} 
    & None & \cellcolor{babyblue}$14.5 \pm 1.4$ & \cellcolor{babyblue}$48.7 \pm 2.2$ & \cellcolor{babyblue}$70.0 \pm 3.3$ & \cellcolor{babyblue}$50.0 \pm 1.3$ & \cellcolor{babyblue}$59.9 \pm 1.8$ \\
    & AFR4 & $16.3 \pm 1.0$ & \oneday & $62.5 \pm 3.0$ & \cellcolor{babypink}$56.8 \pm 0.8$ & \cellcolor{babypink}$59.1 \pm 1.4$ \\
    & FoSR & $16.0 \pm 1.4$ & \cellcolor{babypink}$47.8 \pm 1.6$ & $60.5 \pm 5.8$ & $50.0 \pm 0.6$ & $50.0 \pm 4.3$ \\
    & SDRF & \cellcolor{babypink}$16.8 \pm 0.9$ & \halfday & \cellcolor{babypink}$69.5 \pm 2.6$ & $50.8 \pm 0.7$ & $54.9 \pm 3.2$ \\
    \cline{2-7}
\multirow[c]{3}{*}{Ring} 
    & None & \cellcolor{babyblue}$16.5 \pm 1.6$ & \cellcolor{babyblue}$50.1 \pm 1.9$ & \cellcolor{babyblue}$65.5 \pm 3.6$ & \cellcolor{babyblue}$51.5 \pm 1.2$ & \cellcolor{babyblue}$44.8 \pm 2.3$ \\
    & AFR4 & \cellcolor{babypink}$19.3 \pm 1.2$ & \oneday & $72.5 \pm 3.2$ & $50.6 \pm 0.6$ & \oneday \\
    & FoSR & $18.7 \pm 1.6$ & \cellcolor{babypink}$49.9 \pm 1.9$ & \cellcolor{babypink}$75.0 \pm 5.7$ & $50.3 \pm 0.8$ & $47.6 \pm 3.0$ \\
    & SDRF & $19.0 \pm 1.3$ & \halfday & $70.0 \pm 2.7$ & \cellcolor{babypink}$51.4 \pm 0.4$ & \cellcolor{babypink}$49.3 \pm 3.6$ \\
\bottomrule
\end{tabular}
\end{subtable}
\begin{subtable}[h]{\textwidth}
\centering
\caption{\textbf{GCN}}
\label{tab:gcn}
\begin{tabular}{lllllll}
\toprule
 &  & ENZYMES & IMDB-B & MUTAG & NCI1 & PROTEINS \\
Lift & Rewiring &  &  &  &  &  \\
\midrule
\multirow[c]{3}{*}{None} 
    & None & \cellcolor{babyblue}$32.2 \pm 2.0$ & \cellcolor{babyblue}$49.1 \pm 1.4$ & \cellcolor{babyblue}$71.0 \pm 3.8$ & \cellcolor{babyblue}$48.3 \pm 0.6$ & \cellcolor{babyblue}$73.1 \pm 1.2$ \\
    & AFR4 & $30.5 \pm 2.5$ & \cellcolor{babypink}$48.6 \pm 1.0$ & $69.0 \pm 2.7$ & $48.4 \pm 0.6$ & $72.6 \pm 1.5$ \\
    & FoSR & $27.2 \pm 2.5$ & $47.9 \pm 1.0$ & \cellcolor{babypink}$83.0 \pm 1.5$ & \cellcolor{babypink}$49.1 \pm 0.8$ & \cellcolor{babypink}$75.8 \pm 1.7$ \\
    & SDRF & \cellcolor{babypink}$30.7 \pm 1.5$ & $46.7 \pm 1.2$ & $72.5 \pm 3.5$ & $48.9 \pm 0.5$ & $74.5 \pm 1.5$ \\
    \cline{2-7}
\multirow[c]{3}{*}{Clique} 
    & None & \cellcolor{babyblue}$30.7 \pm 1.2$ & \cellcolor{babyblue}$64.0 \pm 3.1$ & \cellcolor{babyblue}$67.0 \pm 3.5$ & \cellcolor{babyblue}$48.4 \pm 0.4$ & \cellcolor{babyblue}$69.9 \pm 0.6$ \\
    & AFR4 & $28.2 \pm 2.2$ & \oneday & $66.5 \pm 3.0$ & $48.3 \pm 0.6$ & $72.5 \pm 0.9$ \\
    & FoSR & $26.0 \pm 2.1$ & \cellcolor{babypink}$65.5 \pm 3.1$ & \cellcolor{babypink}$81.5 \pm 2.9$ & $48.8 \pm 1.1$ & \cellcolor{babypink}$75.0 \pm 1.4$ \\
    & SDRF & \cellcolor{babypink}$30.2 \pm 2.4$ & \halfday & $72.0 \pm 2.6$ & \cellcolor{babypink}$49.6 \pm 0.6$ & $72.1 \pm 1.2$ \\
    \cline{2-7}
\multirow[c]{3}{*}{Ring} 
    & None & \cellcolor{babyblue}$34.8 \pm 1.3$ & \cellcolor{babyblue}$46.9 \pm 1.4$ & \cellcolor{babyblue}$72.0 \pm 2.7$ & \cellcolor{babyblue}$49.3 \pm 0.9$ & \cellcolor{babyblue}$72.2 \pm 1.3$ \\
    & AFR4 & $29.0 \pm 1.5$ & \oneday & $72.5 \pm 2.3$ & $49.1 \pm 0.6$ & \oneday \\
    & FoSR & $28.8 \pm 1.6$ & \cellcolor{babypink}$48.0 \pm 1.2$ & \cellcolor{babypink}$77.5 \pm 2.4$ & $47.9 \pm 0.6$ & $71.3 \pm 1.6$ \\
    & SDRF & \cellcolor{babypink}$32.0 \pm 1.4$ & \halfday & $67.0 \pm 2.8$ & \cellcolor{babypink}$49.4 \pm 0.6$ & \cellcolor{babypink}$72.7 \pm 1.2$ \\
\bottomrule
\end{tabular}
\end{subtable}

\begin{subtable}[h]{\textwidth}
\centering
\caption{\textbf{GIN}}
\label{tab:gin}
\begin{tabular}{lllllll}
\toprule
 &  & ENZYMES & IMDB-B & MUTAG & NCI1 & PROTEINS \\
Lift & Rewiring &  &  &  &  &  \\
\midrule
\multirow[c]{3}{*}{None} 
    & None & \cellcolor{babyblue}$47.2 \pm 1.9$ & \cellcolor{babyblue}$71.7 \pm 1.5$ & \cellcolor{babyblue}$83.0 \pm 3.1$ & \cellcolor{babyblue}$77.2 \pm 0.5$ & \cellcolor{babyblue}$70.6 \pm 1.4$ \\
    & AFR4 & \cellcolor{babypink}$50.0 \pm 2.7$ & \cellcolor{babypink}$69.7 \pm 1.8$ & \cellcolor{babypink}$88.0 \pm 2.4$ & \cellcolor{babypink}$77.9 \pm 0.5$ & $70.1 \pm 0.9$ \\
    & FoSR & $33.3 \pm 1.7$ & $67.1 \pm 1.5$ & $75.5 \pm 1.7$ & $64.8 \pm 0.9$ & \cellcolor{babypink}$72.2 \pm 0.8$ \\
    & SDRF & $45.5 \pm 2.1$ & $65.8 \pm 2.0$ & $84.5 \pm 2.8$ & $76.3 \pm 1.0$ & $72.0 \pm 1.5$ \\
    \cline{2-7}
\multirow[c]{3}{*}{Clique} 
    & None & \cellcolor{babyblue}$44.0 \pm 1.7$ & \cellcolor{babyblue}$69.1 \pm 1.2$ & \cellcolor{babyblue}$83.0 \pm 2.8$ & \cellcolor{babyblue}$78.8 \pm 0.7$ & \cellcolor{babyblue}$68.7 \pm 1.4$ \\
    & AFR4 & \cellcolor{babypink}$48.5 \pm 2.2$ & \oneday & \cellcolor{babypink}$82.5 \pm 2.6$ & \cellcolor{babypink}$78.2 \pm 0.6$ & $67.6 \pm 0.8$ \\
    & FoSR & $39.3 \pm 1.4$ & \cellcolor{babypink}$70.8 \pm 1.1$ & $79.0 \pm 2.6$ & $63.5 \pm 0.7$ & \cellcolor{babypink}$72.8 \pm 1.2$ \\
    & SDRF & $41.7 \pm 2.2$ & \halfday & $73.5 \pm 3.4$ & $76.9 \pm 0.7$ & $69.2 \pm 0.8$ \\
    \cline{2-7}
\multirow[c]{3}{*}{Ring} 
    & None & \cellcolor{babyblue}$46.7 \pm 2.4$ & \cellcolor{babyblue}$70.1 \pm 1.7$  & \cellcolor{babyblue}$88.0 \pm 2.1$ & \cellcolor{babyblue}$78.9 \pm 0.6$ & \cellcolor{babyblue}$69.8 \pm 1.4$ \\
    & AFR4 & \cellcolor{babypink}$47.0 \pm 1.6$ & \oneday & \cellcolor{babypink}$89.0 \pm 1.9$ & \cellcolor{babypink}$77.5 \pm 0.9$ & \oneday \\
    & FoSR & $37.5 \pm 1.5$ & \cellcolor{babypink}$73.1 \pm 1.1$ & $84.0 \pm 2.1$ & $65.6 \pm 0.8$ & $72.0 \pm 1.3$ \\
    & SDRF & $41.2 \pm 1.6$ & \halfday & $79.0 \pm 2.7$ & $75.9 \pm 0.9$ & \cellcolor{babypink}$72.1 \pm 0.9$ \\
\bottomrule
\end{tabular}
\end{subtable}

\caption{Baseline and rewiring results for SGC, GCN, and GIN. Numbers highlighted in blue correspond to no rewiring and numbers highlighted in red are the best among the rewiring methods. }
\label{tab:all-graph-models}
\end{table}
\begin{table}[htbp]
\centering
\scriptsize
\begin{subtable}[h]{\textwidth}
\centering
\caption{\textbf{RGCN}}
\label{tab:rgcn}
\begin{tabular}{lllllll}
\toprule
 &  & ENZYMES & IMDB-B & MUTAG & NCI1 & PROTEINS \\
Lift & Rewiring &  &  &  &  &  \\
\midrule
\multirow[c]{3}{*}{None} 
    & None & \cellcolor{babyblue}$33.8 \pm 1.6$ & \cellcolor{babyblue}$47.6 \pm 1.4$ & \cellcolor{babyblue}$72.5 \pm 2.5$ & \cellcolor{babyblue}$53.2 \pm 0.7$ & \cellcolor{babyblue}$71.9 \pm 1.6$ \\
    & AFR4 & $36.2 \pm 2.3$ & $48.1 \pm 1.0$ & $69.5 \pm 3.4$ & $54.3 \pm 1.1$ & $72.8 \pm 1.2$ \\
    & FoSR & $39.3 \pm 1.6$ & \cellcolor{babypink}$68.0 \pm 1.3$ & \cellcolor{babypink}$83.5 \pm 1.8$ & $62.4 \pm 0.5$ & $71.2 \pm 1.8$ \\
    & SDRF & \cellcolor{babypink}$42.5 \pm 1.3$ & $59.1 \pm 2.2$ & $76.0 \pm 1.9$ & \cellcolor{babypink}$63.4 \pm 1.1$ & \cellcolor{babypink}$75.6 \pm 1.2$ \\
    \cline{2-7}
\multirow[c]{3}{*}{Clique} 
    & None & \cellcolor{babyblue}$48.8 \pm 1.2$ & \cellcolor{babyblue}$71.0 \pm 1.0$ & \cellcolor{babyblue}$79.5 \pm 1.7$ & \cellcolor{babyblue}$72.9 \pm 0.8$ & \cellcolor{babyblue}$72.4 \pm 1.6$ \\
    & AFR4 & $42.3 \pm 2.4$ & \oneday & $78.0 \pm 2.1$ & \cellcolor{babypink}$75.0 \pm 0.9$ & \cellcolor{babypink}$74.2 \pm 1.2$ \\
    & FoSR & $39.3 \pm 2.5$ & \cellcolor{babypink}$69.7 \pm 1.5$ & $79.0 \pm 3.2$ & $64.5 \pm 0.6$ & $70.9 \pm 1.4$ \\
    & SDRF & \cellcolor{babypink}$45.2 \pm 1.5$ & \halfday & \cellcolor{babypink}$81.5 \pm 3.8$ & $68.5 \pm 0.6$ & $70.6 \pm 1.6$ \\
    \cline{2-7}
\multirow[c]{3}{*}{Ring} 
    & None & \cellcolor{babyblue}$35.2 \pm 1.7$ & \cellcolor{babyblue}$71.1 \pm 1.4$ & \cellcolor{babyblue}$83.5 \pm 2.7$ & \cellcolor{babyblue}$73.9 \pm 0.5$ & \cellcolor{babyblue}$70.7 \pm 1.6$ \\
    & AFR4 & $34.8 \pm 2.1$ & \oneday & $80.5 \pm 2.9$ & \cellcolor{babypink}$73.5 \pm 0.5$ & \oneday \\
    & FoSR & $38.5 \pm 1.5$ & \cellcolor{babypink}$70.0 \pm 1.6$ & \cellcolor{babypink}$84.0 \pm 2.1$ & $64.2 \pm 0.8$ & \cellcolor{babypink}$71.3 \pm 1.2$ \\
    & SDRF & \cellcolor{babypink}$45.7 \pm 1.5$ & \halfday & $79.5 \pm 3.7$ & $68.3 \pm 0.6$ & $68.8 \pm 1.1$ \\
\bottomrule
\end{tabular}
\end{subtable}
\begin{subtable}[h]{\textwidth}
\centering
\caption{\textbf{RGIN}}
\label{tab:rgin}
\begin{tabular}{lllllll}
\toprule
 &  & ENZYMES & IMDB-B & MUTAG & NCI1 & PROTEINS \\
Lift & Rewiring &  &  &  &  &  \\
\midrule
\multirow[c]{3}{*}{None} 
    & None & \cellcolor{babyblue}$46.8 \pm 1.8$ & \cellcolor{babyblue}$69.6 \pm 1.6$ & \cellcolor{babyblue}$81.5 \pm 1.7$ & \cellcolor{babyblue}$76.8 \pm 1.1$ & \cellcolor{babyblue}$70.8 \pm 1.2$ \\
    & AFR4 & \cellcolor{babypink}$49.8 \pm 2.0$ & \cellcolor{babypink}$73.2 \pm 1.4$ & \cellcolor{babypink}$85.5 \pm 2.0$ & \cellcolor{babypink}$77.0 \pm 0.7$ & $71.1 \pm 1.3$ \\
    & FoSR & $48.2 \pm 1.4$ & $48.9 \pm 2.9$ & \cellcolor{babypink}$85.5 \pm 2.8$ & $55.1 \pm 2.6$ & \cellcolor{babypink}$72.4 \pm 1.4$ \\
    & SDRF & $49.7 \pm 2.1$ & $50.4 \pm 3.2$ & $85.0 \pm 1.7$ & $52.8 \pm 2.7$ & $70.9 \pm 1.2$ \\
    \cline{2-7}
\multirow[c]{3}{*}{Clique} 
    & None & \cellcolor{babyblue}$50.8 \pm 1.5$ & \cellcolor{babyblue}$71.6 \pm 0.9$ & \cellcolor{babyblue}$86.0 \pm 2.3$ & \cellcolor{babyblue}$79.2 \pm 0.6$ & \cellcolor{babyblue}$71.5 \pm 1.5$ \\
    & AFR4 & \cellcolor{babypink}$55.8 \pm 2.5$ & \oneday & \cellcolor{babypink}$85.0 \pm 2.4$ & \cellcolor{babypink}$79.5 \pm 0.4$ & $71.0 \pm 1.1$ \\
    & FoSR & $46.2 \pm 1.4$ & \cellcolor{babypink}$69.0 \pm 1.4$  & $79.0 \pm 2.4$ & $72.7 \pm 0.6$ & \cellcolor{babypink}$71.8 \pm 1.7$ \\
    & SDRF & $45.5 \pm 1.6$ & \halfday & $83.0 \pm 3.2$ & $76.4 \pm 0.9$ & $69.1 \pm 2.4$ \\
    \cline{2-7}
\multirow[c]{3}{*}{Ring} 
    & None & \cellcolor{babyblue}$45.3 \pm 1.3$ & \cellcolor{babyblue}$68.6 \pm 1.2$ & \cellcolor{babyblue}$87.0 \pm 2.9$ & \cellcolor{babyblue}$78.4 \pm 0.7$ & \cellcolor{babyblue}$68.8 \pm 1.5$ \\
    & AFR4 & \cellcolor{babypink}$49.2 \pm 1.5$ & \oneday & \cellcolor{babypink}$87.5 \pm 2.4$ & \cellcolor{babypink}$79.8 \pm 0.7$ & \oneday \\
    & FoSR & $47.3 \pm 1.6$ & \cellcolor{babypink}$67.2 \pm 1.8$ & $82.5 \pm 4.2$ & $72.9 \pm 0.7$ & \cellcolor{babypink}$71.3 \pm 1.5$ \\
    & SDRF & $47.2 \pm 2.6$ & \halfday & $81.0 \pm 3.3$ & $76.4 \pm 0.5$ & $70.0 \pm 0.9$ \\
\bottomrule
\end{tabular}
\end{subtable}
\caption{Baseline and rewiring results for RGCN and RGIN. Numbers highlighted in blue correspond to no rewiring and numbers highlighted in red are the best among the rewiring methods. }
\label{tab:all-relational-models}
\end{table}
\begin{table}[htbp]
\centering
\scriptsize

\begin{subtable}[h]{\textwidth}
\centering
\caption{\textbf{SIN}}
\label{tab:sin}
\begin{tabular}{lllllll}
\toprule
 &  & ENZYMES & IMDB-B & MUTAG & NCI1 & PROTEINS \\
Lift & Rewiring &  &  &  &  &  \\
\midrule
\multirow[c]{3}{*}{None} 
    & None & \cellcolor{babyblue}$47.5 \pm 2.3$ & \cellcolor{babyblue}$70.0 \pm 1.4$ & \cellcolor{babyblue}$88.5 \pm 3.0$ & \cellcolor{babyblue}$77.0 \pm 0.6$ & \cellcolor{babyblue}$70.2 \pm 1.3$ \\
    & AFR4 & $44.0 \pm 2.5$ & \cellcolor{babypink}$71.0 \pm 1.0$ & \cellcolor{babypink}$85.5 \pm 1.7$ & \cellcolor{babypink}$76.4 \pm 0.4$ & $69.6 \pm 1.3$ \\
    & FoSR & $45.7 \pm 2.6$ & $63.0 \pm 2.7$ & \cellcolor{babypink}$85.5 \pm 2.8$ & $61.3 \pm 2.4$ & \cellcolor{babypink}$73.2 \pm 1.5$ \\
    & SDRF & \cellcolor{babypink}$46.8 \pm 2.1$ & $62.8 \pm 3.3$ & $80.0 \pm 2.1$ & $54.7 \pm 3.5$ & $70.2 \pm 1.8$ \\
    \cline{2-7}
\multirow[c]{3}{*}{Clique} 
    & None & \cellcolor{babyblue}$51.0 \pm 2.4$ & \cellcolor{babyblue}$53.0 \pm 1.9$ & \cellcolor{babyblue}$87.0 \pm 3.2$ & \cellcolor{babyblue}$76.6 \pm 1.3$ & \cellcolor{babyblue}$66.9 \pm 1.3$ \\
    & AFR4 & \cellcolor{babypink}$46.5 \pm 1.2$ & \oneday & \cellcolor{babypink}$83.5 \pm 1.7$ & \cellcolor{babypink}$75.4 \pm 0.7$ & $69.6 \pm 1.0$ \\
    & FoSR & $38.2 \pm 2.2$ & \cellcolor{babypink}$64.0 \pm 2.3$ & \cellcolor{babypink}$83.5 \pm 2.7$ & $65.2 \pm 0.7$ & \cellcolor{babypink}$70.4 \pm 1.2$ \\
    & SDRF & $44.8 \pm 1.9$ & \halfday & $80.0 \pm 3.5$ & $73.9 \pm 0.4$ & $68.7 \pm 1.2$ \\
    \cline{2-7}
\multirow[c]{3}{*}{Ring} 
    & None & \cellcolor{babyblue}$40.3 \pm 2.2$ & \cellcolor{babyblue}$50.6 \pm 1.9$ & \cellcolor{babyblue}$85.0 \pm 2.1$ & \cellcolor{babyblue}$80.0 \pm 0.8$ & \cellcolor{babyblue}$70.6 \pm 1.1$ \\
    & AFR4 & \cellcolor{babypink}$48.0 \pm 2.0$ & \oneday & \cellcolor{babypink}$88.5 \pm 2.5$ & \cellcolor{babypink}$79.1 \pm 0.9$ & \oneday \\
    & FoSR & $39.8 \pm 1.6$ & \cellcolor{babypink}$60.9 \pm 2.1$ & $86.0 \pm 2.4$ & $71.2 \pm 0.6$ & \cellcolor{babypink}$72.1 \pm 0.7$ \\
    & SDRF & $39.2 \pm 1.4$ & \halfday & $80.5 \pm 2.5$ & $75.6 \pm 0.5$ & $67.5 \pm 1.1$ \\
\bottomrule
\end{tabular}
\end{subtable}

\begin{subtable}[h]{\textwidth}
\centering
\caption{\textbf{CIN}}
\label{tab:cin}
\begin{tabular}{lllllll}
\toprule
 &  & ENZYMES & IMDB-B & MUTAG & NCI1 & PROTEINS \\
Lift & Rewiring &  &  &  &  &  \\
\midrule
\multirow[c]{3}{*}{None}
    & None & \cellcolor{babyblue}$50.0 \pm 1.9$ & \cellcolor{babyblue}$58.1 \pm 4.0$ & \cellcolor{babyblue}$86.5 \pm 1.8$ & \cellcolor{babyblue}$51.4 \pm 2.5$ & \cellcolor{babyblue}$70.7 \pm 1.0$ \\
    & AFR4 & $48.2 \pm 1.9$ & \cellcolor{babypink}$58.8 \pm 4.3$ & \cellcolor{babypink}$87.0 \pm 2.4$ & $53.3 \pm 2.6$ & \cellcolor{babypink}$71.0 \pm 1.4$ \\
    & FoSR & \cellcolor{babypink}$49.0 \pm 2.0$ & $58.4 \pm 2.7$ & $81.0 \pm 1.8$ & \cellcolor{babypink}$66.2 \pm 2.0$ & $70.9 \pm 1.3$ \\
    & SDRF & $44.0 \pm 1.8$ & $55.3 \pm 2.5$ & $82.5 \pm 3.0$ & $59.0 \pm 4.3$ & $69.5 \pm 1.5$ \\
    \cline{2-7}
\multirow[c]{3}{*}{Clique} 
    & None & \cellcolor{babyblue}$49.8 \pm 1.9$ & \cellcolor{babyblue}$52.6 \pm 2.4$ & \cellcolor{babyblue}$85.5 \pm 2.8$ & \cellcolor{babyblue}$51.8 \pm 2.3$ & \cellcolor{babyblue}$70.7 \pm 1.2$ \\
    & AFR4 & \cellcolor{babypink}$46.7 \pm 1.3$ & \oneday & \cellcolor{babypink}$86.5 \pm 2.6$ & $49.0 \pm 0.5$ & $69.2 \pm 1.3$ \\
    & FoSR & $37.2 \pm 1.8$ & \cellcolor{babypink}$68.1 \pm 1.6$ & $82.5 \pm 3.0$ & $66.9 \pm 0.9$ & \cellcolor{babypink}$70.3 \pm 0.8$ \\
    & SDRF & $42.0 \pm 2.3$ & \halfday & $78.5 \pm 4.1$ & \cellcolor{babypink}$72.5 \pm 0.8$ & $67.9 \pm 1.8$ \\
    \cline{2-7}
\multirow[c]{3}{*}{Ring} 
    & None & \cellcolor{babyblue}$47.5 \pm 2.0$ & \cellcolor{babyblue}$48.6 \pm 1.6$ & \cellcolor{babyblue}$93.5 \pm 2.1$ & \cellcolor{babyblue}$51.6 \pm 3.2$ & \cellcolor{babyblue}$68.7 \pm 1.4$\\
    & AFR4 & \cellcolor{babypink}$49.5 \pm 2.0$ & \oneday              & $92.5 \pm 1.5$ & $58.2 \pm 4.2$ & \oneday\\
    & FoSR & $37.7 \pm 2.1$ & \cellcolor{babypink}$66.1 \pm 2.0$ & \cellcolor{babypink}$95.0 \pm 1.3$ & $70.4 \pm 0.8$ & \cellcolor{babypink}$68.5 \pm 1.6$\\
    & SDRF & $38.2 \pm 2.2$ & \halfday              & $88.0 \pm 2.4$ & \cellcolor{babypink}$76.5 \pm 0.5$ & $65.6 \pm 1.3$\\
\bottomrule
\end{tabular}
\end{subtable}

\begin{subtable}[h]{\textwidth}
\centering
\caption{\textbf{CIN++}}
\label{tab:cin++}
\begin{tabular}{lllllll}
\toprule
 &  & ENZYMES & IMDB-B & MUTAG & NCI1 & PROTEINS \\
Lift & Rewiring &  &  &  &  &  \\
\midrule
\multirow[c]{3}{*}{None} 
    & None & \cellcolor{babyblue}$48.5 \pm 1.9$ & \cellcolor{babyblue}$66.6 \pm 3.7$ & \cellcolor{babyblue}$85.0 \pm 3.4$ & \cellcolor{babyblue}$60.8 \pm 3.8$ & \cellcolor{babyblue}$67.9 \pm 1.9$ \\
    & AFR4 & \cellcolor{babypink}$51.0 \pm 1.5$ & \cellcolor{babypink}$61.0 \pm 4.1$ & \cellcolor{babypink}$91.0 \pm 2.3$ & $54.8 \pm 3.7$ & $70.7 \pm 1.3$ \\
    & FoSR & $50.5 \pm 1.8$ & $56.0 \pm 3.9$ & $82.0 \pm 2.5$ & \cellcolor{babypink}$64.8 \pm 3.1$ & \cellcolor{babypink}$71.4 \pm 1.4$ \\
    & SDRF & $45.8 \pm 2.2$ & $53.3 \pm 2.3$ & $87.5 \pm 2.3$ & $58.0 \pm 4.4$ & $69.3 \pm 1.6$ \\
    \cline{2-7}
\multirow[c]{3}{*}{Clique} 
    & None & \cellcolor{babyblue}$50.5 \pm 2.1$ & \cellcolor{babyblue}$62.8 \pm 3.8$ & \cellcolor{babyblue}$90.5 \pm 2.2$ & \cellcolor{babyblue}$61.5 \pm 4.6$ & \cellcolor{babyblue}$68.3 \pm 1.3$ \\
    & AFR4 & \cellcolor{babypink}$52.7 \pm 1.6$ & \oneday & \cellcolor{babypink}$84.5 \pm 3.3$ & $60.6 \pm 4.8$ & $66.7 \pm 1.3$ \\
    & FoSR & $43.0 \pm 2.1$ & \cellcolor{babypink}$64.7 \pm 1.5$ & $78.5 \pm 2.4$ & $72.9 \pm 0.5$ & \cellcolor{babypink}$71.9 \pm 1.0$ \\
    & SDRF & $47.7 \pm 1.5$ & \halfday & $81.0 \pm 3.0$ & \cellcolor{babypink}$76.8 \pm 0.4$ & $70.0 \pm 1.7$ \\
    \cline{2-7}
\multirow[c]{3}{*}{Ring} 
    & None & \cellcolor{babyblue}$47.5 \pm 1.7$ & \cellcolor{babyblue}$66.0 \pm 1.4$ & \cellcolor{babyblue}$85.5 \pm 2.0$ & \cellcolor{babyblue}$56.8 \pm 4.5$ & \cellcolor{babyblue}$68.1 \pm 1.2$ \\
    & AFR4 & \cellcolor{babypink}$46.3 \pm 1.8$ & \oneday              & \cellcolor{babypink}$90.0 \pm 2.7$ & $59.5 \pm 4.1$ & \oneday \\
    & FoSR & $46.0 \pm 2.8$ & \cellcolor{babypink}$67.8 \pm 1.3$ & $85.0 \pm 1.7$ & $71.9 \pm 0.8$ & \cellcolor{babypink}$70.1 \pm 1.2$ \\
    & SDRF & $43.2 \pm 1.5$ & \halfday              & $81.5 \pm 3.1$ & \cellcolor{babypink}$76.0 \pm 0.6$ & $69.1 \pm 1.0$ \\
\bottomrule
\end{tabular}
\end{subtable}

\caption{Baseline and rewiring results for SIN, CIN, and CIN++. Numbers highlighted in blue correspond to no rewiring and numbers highlighted in red are the best among the rewiring methods. }
\label{tab:all-simplicial-models}
\end{table}
\begin{figure}[htbp]
    \centering
    \includegraphics[width=0.9\textwidth]{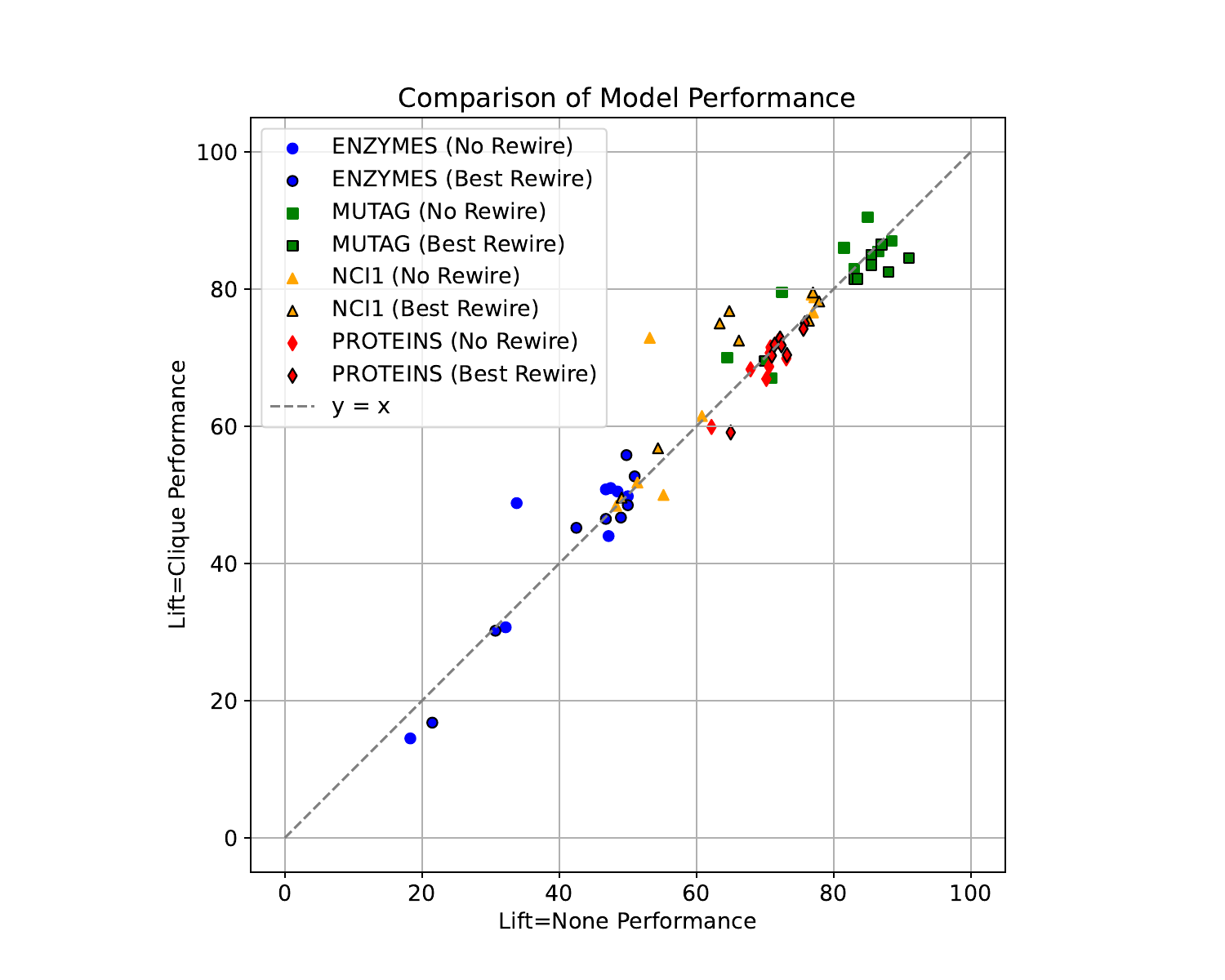}
    \caption{Comparison of model performance between graph representations (Lift=None) and corresponding clique complex representations (Lift=Clique). Each point corresponds to a specific model on a given dataset, comparing the performance with and without rewiring.}
    \label{fig:graph_vs_clique_comparison}
\end{figure}

\subsection{Synthetic Benchmark: RingTransfer}
\label{appendix:experiments:ringtransfer}
The RingTransfer benchmark is a graph transfer task introduced by \cite{bodnar2021weisfeiler}. Each graph in the dataset is a cycle graph of size $2k$ for some $k\geq 1$. A node $i$ is randomly selected as the \textit{root node} and the node $j = i+k (\text{ mod} ) 2k$ on the opposite side of the cycle is designated as the feature node.  All features are initialized to zero except for the feature node which contains a one-hot label among the 5 classes. The task is to then predict the label by reading the feature stored in the root node. This is shown for a cycle graph with 10 nodes in Figure~\ref{fig:ringtransfer_graph}, where the red node is the root node and the green node is the opposite node. Note that the size of the ring determines the number of layers necessary for a model to complete the task successfully. By increasing the size of the ring, we can test how depth and the task distance impact training performance. When increasing the ring size, the model depth is set to $2k+1$ so that task information remains within the model's receptive field.

\begin{figure}
    \centering
    \includegraphics[width=0.3\linewidth]{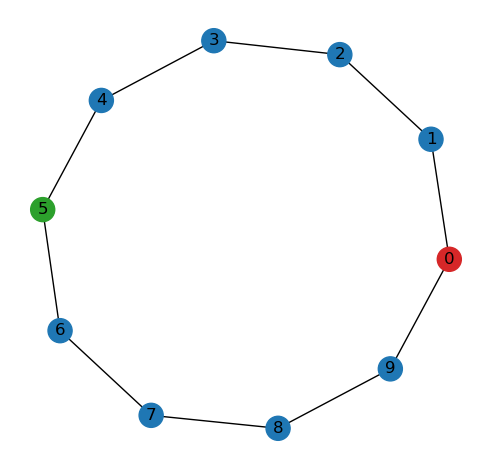}
    \caption{Graph for the RingTransfer experiment.}
    \label{fig:ringtransfer_graph}
\end{figure}

\begin{table}[htbp]
\centering
\begin{tabular}{|l|l|}
\hline
\textbf{Parameter}                  & \textbf{Value} \\ \hline
Number of Graphs                    & 1000           \\ \hline
Trials                              & 10              \\ \hline
Default Neighbors Match Cliques     & 3              \\ \hline
Default Neighbors Match Clique Size & 5              \\ \hline
Default Ring Transfer Nodes         & 10             \\ \hline
\end{tabular}
\caption{Synthetic Parameters}
\label{table:synthetic_config}
\end{table}

\section{Implementation Details}
\label{appendix:implementation-details}

\subsection{Model Details}
Each model consists of multiple graph-convolution layers. For more details about the convolutional layers, refer to Appendix~\ref{appendix:implementation:layers}. The hyperparameters are presented in Table~\ref{table:model_config}. For the real world experiments, each model has a neural network head consisting of a pooling operation followed by a two layer neural network with ReLU and dropout. The only exception to this is the SGC model which uses pooling and a single linear layer for the head. For the synthetic graph transfer experiments, there is no neural network head. Instead, the neural network just returns the feature of the root node at the final layer. 

\begin{table}[htbp]
\centering
\begin{tabular}{|l|l|}
\hline
\textbf{Parameter} & \textbf{Value} \\ \hline
Layers             & 4              \\ \hline
Hidden             & 64             \\ \hline
Dropout            & 0.5            \\ \hline
Pooling            & mean           \\ \hline
Multidimensional   & False          \\ \hline
Max Dimension      & 2              \\ \hline
\end{tabular}
\caption{Model Parameters}
\label{table:model_config}
\end{table}

\begin{table}[htbp]
\centering
\begin{tabular}{|l|l|}
\hline
\textbf{Parameter} & \textbf{Value} \\ \hline
Optimizer          & Adam           \\ \hline
Batch Size         & 64             \\ \hline
Learning Rate      & 0.001          \\ \hline
Max Epochs         & 500            \\ \hline
Stop Criteria      & Validation     \\ \hline
Stop Factor        & 1.01           \\ \hline
Stop Patience      & 100            \\ \hline
\end{tabular}
\caption{Optimization Parameters}
\label{table:optimization_config}
\end{table}

\subsection{Hardware}
Experiments were run on a system with Intel(R) Xeon(R) Gold 6152 CPUs @ 2.10GHz and NVIDIA GeForce RTX 2080 Ti GPUs. 

\subsection{Libraries}

Code Use:
Rewiring algorithms were adapted from the works of \cite{karhadkar2023fosr} and \cite{fesser2023mitigating}. 
We used PyTorch Geometric Benchmarks for neural network hyperparameters. 
We build on the work of \cite{giusti2024cinplusplus} for clique graph lifting and ring transfer. The real world graph datasets were sourced from the work of \cite{morris2020}. 
The SGC model was adapted from the work of \cite{wu2019simplifying}. 
For curvature implementations, we adapt code from \cite{pmlr-v202-nguyen23c} for Ollivier-Ricci and \cite{fesser2023mitigating} for BFC and AFC. 

Python libraries: Torch Geometric \citep{fey2019fast}. 
Torch \citep{paszke2019pytorch}. 
Scikit-Learn \citep{scikit-learn}. 
Ray \citep{moritz2018ray}. 
Ray Tune \citep{liaw2018tune}. 
Seaborn \citep{Waskom2021}. 
Matplotlib \citep{Hunter:2007}. 
Networkx \citep{osti_960616}. 
Pandas \cite{reback2020pandas}. 
Numpy \cite{harris2020array}. 
Numba \citep{10.1145/2833157.2833162}.  
Gudhi \citep{gudhi:urm}. 
POT \citep{JMLR:v22:20-451}.

\subsection{Graph Lifting}

We lift graphs to clique and ring complexes for relational message passing using Algorithm~\ref{alg:clique_complex}. We refer to this graph lifting as the \textit{Clique} and \textit{Ring} lifts. Note that when representing a graph as a complex (e.g., \texttt{max\_dim} = 0), we simply add the $\uparrow$ relation to the edges. We refer to this graph lifting as the \textit{None} lift. The resulting complex has trivial cell intersections and does not contain cell unions. This is outlined in Algorithm~\ref{alg:none}. We note that whereas clique lifts introduce higher-order simplices by introducing simplices for cliques in the graph, ring lifts instead promote cycles to $2$-dimensional cells.

\begin{algorithm}[htbp]
\caption{Graph Lifting (None)}
\label{alg:none}

\textbf{Input:} Unweighted undirected graph $G = (V, E)$, node features $(\mathbf{x}_v)_{v \in V}$\\
\textbf{Output:} Relational structure $(\mathcal{S}, \{R_\texttt{upper}\})$ and features $(\mathbf{x}_\sigma)_{\sigma \in \mathcal{S}}$

\begin{algorithmic}[1]
\State Initialize $\mathcal{S} = \emptyset$ \Comment{Entities}
\For{each node $v \in V$}
    \State $\sigma \gets \{v\}$
    \State $\mathcal{S} \gets \mathcal{S} \cup \{\sigma\}$
    \State $\mathbf{x}_\sigma \gets \mathbf{x}_v$ \Comment{Features}
\EndFor
\State Initialize $R_\texttt{upper} = \emptyset$ \Comment{Relations}
\For{each edge $\{u, v\} \in E$}
    \State $\sigma, \tau \gets u, v$
    \State $R_\texttt{upper} \gets R_\texttt{upper} \cup \{(\sigma, \tau), (\tau, \sigma)\}$
\EndFor
\State \textbf{return} $(\mathcal{S}, \{R_\texttt{upper}\})$ and $(\mathbf{x}_\sigma)_{\sigma \in \mathcal{S}}$
\end{algorithmic}
\end{algorithm}

\begin{algorithm}[htbp]
\caption{Graph Lifting (Complex)}
\label{alg:clique_complex}

\textbf{Input:} Unweighted undirected graph $G = (V, E)$, node features $(\mathbf{x}_v)_{v \in V}$, relation types $\mathcal{A} \subseteq \{\texttt{boundary}, \texttt{coboundary}, \texttt{lower}, \texttt{upper}\}$, max simplex dimension \texttt{max\_dim}\\
\textbf{Output:} Relational structure $(\mathcal{S}, \{R_r\}_{r \in \mathcal{A}})$ and features $(\mathbf{x}_\sigma)_{\sigma \in \mathcal{S}}$

\begin{algorithmic}[1]
\State Initialize $\mathcal{S} = \emptyset$ \Comment{Entities}
\For{each non-empty clique $\sigma \subseteq V$ with $|\sigma| \leq \texttt{max\_dim}+1$}
    \State $\mathcal{S} \gets \mathcal{S} \cup \{\sigma\}$
    \State $\mathbf{x}_\sigma \gets \frac{1}{|\sigma|} \sum_{v \in \sigma} \mathbf{x}_v$ \Comment{Features}
\EndFor
\For{each relation type $r \in \mathcal{A}$}
    \State Initialize $R_r = \emptyset$ \Comment{Relations}
    \For{each entity $\sigma \in \mathcal{S}$}
        \If{$r =$ \texttt{boundary}}
            \For{each $\tau \in \mathcal{B}(\sigma)$}
                \State $R_r \gets R_r \cup \{(\sigma, \tau)\}$
            \EndFor
        \ElsIf{$r =$ \texttt{coboundary}}
            \For{each $\tau \in \mathcal{C}(\sigma)$}
                \State $R_r \gets R_r \cup \{(\sigma, \tau)\}$
            \EndFor
        \ElsIf{$r =$ \texttt{lower}}
            \For{each $\tau \in \mathcal{N}_\downarrow(\sigma)$}
                \State $R_r \gets R_r \cup \{(\sigma, \tau, \sigma \cap \tau)\}$
            \EndFor
        \ElsIf{$r =$ \texttt{upper}}
            \For{each $\tau \in \mathcal{N}_\uparrow(\sigma)$}
                \State $R_r \gets R_r \cup \{(\sigma, \tau, \sigma \cup \tau)\}$
            \EndFor
        \EndIf
    \EndFor
\EndFor
\State \textbf{return} $(\mathcal{S}, \{R_r\}_{r \in \mathcal{A}})$ and $(\mathbf{x}_\sigma)_{\sigma \in \mathcal{S}}$
\end{algorithmic}
\end{algorithm}

\subsection{Layers}
\label{appendix:implementation:layers}
We review the implementation of the layers and update rules used in the models considered in this work. In the following, $\mathcal{B}(\cdot)$, $\mathcal{C}(\cdot)$, $\mathcal{N}_\uparrow(\cdot)$, and $\mathcal{N}_\downarrow(\cdot)$ follow the definitions in Definition~\ref{def:adjacency_relations}. Given a relational structure $\mathcal{R} = (\mathcal{S}, R_1, \ldots, R_k)$, the notation $r \in \mathcal{R}$ refers to iterating over the relation indices $r = 1, 2, \ldots, k$. For a binary relation $R \subseteq \mathcal{S} \times \mathcal{S}$, $\mathcal{R}(\cdot)$ denotes the incoming neighborhood of an entity, i.e., for any $\sigma, \tau \in \mathcal{S}$, if $\tau \in \mathcal{N}_R(\sigma)$, then $(\sigma, \tau) \in R$.

\paragraph{SGC}

Following \citet{wu2019simplifying}, the SGC update rule is:
\begin{equation*}
\mathbf{h}_\sigma^{(t + 1)} =  \sum_{r \in \mathcal{R}} \sum_{\tau \in \mathcal{N}_{R_r}(\sigma)} \mathbf{h}_\tau^{(t)},
\end{equation*}
where no learnable weights nor non-linearities are used, and signals are simply aggregated over neighborhoods. 

\paragraph{GCN}

Following \citet{kipf2017semisupervised}, the GCN update rule is: 
\begin{equation*}
\mathbf{h}_\sigma^{(t+1)} = \mathbf{f}\left( \mathbf{W}^{(t)}\left( \sum_{r \in \mathcal{R}} \sum_{\tau \in \mathcal{N}_{R_i}(\sigma) \cup \{\sigma\}} \frac{1}{\sqrt{\hat d_\sigma \hat d_\tau}}\mathbf{h}_\tau^{(t)}\right) \right), 
\end{equation*}
where $\mathbf{f}$ is a point-wise non-linearity (e.g., ReLU), $\mathbf{W}^{(t)}$ is a learnable weight matrix, and $\hat{d}_\tau = \text{deg}(\tau) + 1$.

\paragraph{GIN}

Following \citet{xu2018how}, the GIN update rule is:
\begin{equation*}
\mathbf{h}_{\sigma}^{(t+1)} = \text{MLP}^{(t)}\left((1 + \epsilon)\mathbf{h}_{\sigma}^{(t)} + \sum_{r \in \mathcal{R}} \sum_{\tau \in \mathcal{N}_{R_i}(\sigma)} \mathbf{W}^{(t)}\mathbf{h}_{\tau}^{(t)}\right) , 
\end{equation*}
where $\mathbf{W}^{(t)}$ is a learnable weight matrix and $\epsilon$ is a learnable scalar. Here, MLP consists of two layer fully connected layers, each followed by a point-wise non-linearity (e.g., ReLU) and a BatchNorm layer. 

\paragraph{R-GCN}

Following \citet{schlichtkrull2018modeling}, the R-GCN update rule is:
\begin{equation*}
    \mathbf{h}_{\sigma}^{(t+1)} = \mathbf{f}\left( \mathbf{W}_{\text{root}}^{(t)} \mathbf{h}_{\sigma}^{(t)} + \sum_{r\in \mathcal{R}} \sum_{j\in \mathcal{N}_r(i)}\frac{1}{|\mathcal{N}_r(i)|} \mathbf{W}_{\text{r}}^{(t)} \mathbf{h}_{\sigma}^{(t)} \right),
\end{equation*}
where $\mathbf{f}$ is a component-wise non-linearity (e.g., ReLU), and $\mathbf{W}_{\text{root}}^{(t)}$ and $\mathbf{W}_r^{(t)}$ are learnable weight matrices.

\paragraph{R-GIN}

Following \citet{xu2018how} and \citet{schlichtkrull2018modeling}, the R-GIN update rule is
\begin{equation*}
    \mathbf{h}_{\sigma}^{(t+1)}  =\mathbf{W}_{\text{root}}^{(t)} \mathbf{h}_{\sigma}^{(t)} + \sum_{r\in \mathcal{R}} \text{MLP}_r^{(t)}
    \left( (1+\epsilon_r) \mathbf{h}_{\sigma}^{(t)} + \sum_{\tau\in \mathcal{N}_r(\sigma)} \mathbf{h}_{\tau}^{(t)}\right) , 
\end{equation*}
where $\mathbf{W}_{\text{root}}^{(t)}$ is a learnable weight matrix  and each $\epsilon_\bullet$ is a learnable scalar. Here, each MLP consists of two layer fully connected layers, each followed by a point-wise non-linearity (e.g., ReLU) and a BatchNorm layer. 
and $\text{MLP}_r^{(t)}$ is a relation-specific MLP.

\paragraph{SIN}

Following \citet{pmlr-v139-bodnar21a}, the SIN update rule is:
\begin{equation*}
\mathbf{h}_{\sigma}^{(t+1)} = \text{MLP}_{U, p}^{(t)}\left[
    \begin{array}{c}
        \text{MLP}_{B, p}\left( (1+\epsilon_{\mathcal{B}})\, \mathbf{h}_{\sigma}^{(t)} + \displaystyle\sum\limits_{\delta\in\mathcal{B}(\sigma)} \mathbf{h}_\delta^{(t)} \right) \\[10pt]
        \text{MLP}_{\uparrow, p}\left( (1+\epsilon_\uparrow)\, \mathbf{h}_{\sigma}^{(t)} + \displaystyle\sum\limits_{\delta\in\mathcal{N}_\uparrow(\sigma)} \mathbf{h}_\delta^{(t)} \right)\\[10pt]
        \textit{additionally, when rewiring:}\\[5pt]
        \dbox{$\text{MLP}_{R_{new}, p}\left( (1+\epsilon_{R_{new}})\, \mathbf{h}_{\sigma}^{(t)} + \displaystyle\sum\limits_{\delta\in\mathcal{N}_{R_{new}}(\sigma)} \mathbf{h}_\delta^{(t)} \right)$}
    \end{array}
\right],
\end{equation*}
where $[\cdot]$ denotes column-wise concatenation and each $\epsilon_\bullet$ is a learnable scalar. Here, each MLP consists of two layer fully connected layers, each followed by a point-wise non-linearity (e.g., ReLU) and a BatchNorm layer. 
and $\text{MLP}_r^{(t)}$ is a relation-specific MLP.

\paragraph{CIN}

Following \citet{bodnar2021weisfeiler}, the CIN update rule is:
\begin{equation*}
\mathbf{h}_{\sigma}^{(t+1)} = \text{MLP}_{U, p}^{(t)}\left[
    \begin{array}{c}
        \text{MLP}_{B, p}\left( (1+\epsilon_{\mathcal{B}})\, \mathbf{h}_{\sigma}^{(t)} + \displaystyle\sum\limits_{\tau\in\mathcal{B}(\sigma)} \mathbf{h}_\tau^{(t)} \right) \\[15pt]
        \text{MLP}_{\uparrow, p}\left( (1+\epsilon_\uparrow)\, \mathbf{h}_{\sigma}^{(t)} + \displaystyle\sum\limits_{\substack{\tau\in\mathcal{N}_\uparrow(\sigma) \\ \delta \in C(\sigma, \tau)}} \text{MLP}_{M, p}\left[
            \begin{array}{c}
                \mathbf{h}_\tau^{(t)} \\[5pt]
                \mathbf{h}_\delta^{(t)}
            \end{array}
        \right] \right) \\[15pt]
        \textit{additionally, when rewiring:} \\[5pt]
        \dbox{$\text{MLP}_{{R_{new}}, p}\left( (1+\epsilon_{R_\text{new}})\, \mathbf{h}_{\sigma}^{(t)} + \displaystyle\sum\limits_{\delta\in\mathcal{N}_{R_\text{new}}(\sigma)} \mathbf{h}_\delta^{(t)} \right)$}
    \end{array}
\right],
\end{equation*}
where $[\cdot]$ denotes column-wise concatenation, each $\epsilon_\bullet$ is a learnable scalar, and $\text{MLP}_{M, p}$ is a single fully connected layer projecting back to the dimension of $\mathbf{h}_\sigma^{(t)}$. The other MLPs consists of two layer fully connected layers, each followed by a point-wise non-linearity (e.g., ReLU) and a BatchNorm layer. When entity unions aren't contained in the relational structure due to dimension constraints, we instead use $\mathbf{h}^{(t)}_\delta := \mathbf{0}$.

\paragraph{CIN++}

Following \citet{giusti2024cinplusplus}, the CIN++ update rule is:
\begin{equation*}
\mathbf{h}_{\sigma}^{(t+1)} = \text{MLP}_{U, p}^{(t)}\left[
    \begin{array}{c}
        \text{MLP}_{B, p}\left( (1+\epsilon_{\mathcal{B}})\, \mathbf{h}_{\sigma}^{(t)} + \displaystyle\sum\limits_{\tau\in\mathcal{B}(\sigma)} \mathbf{h}_\tau^{(t)} \right) \\[15pt]
        \text{MLP}_{\uparrow, p}\left( (1+\epsilon_\uparrow)\, \mathbf{h}_{\sigma}^{(t)} + \displaystyle\sum\limits_{\substack{\tau\in\mathcal{N}_\uparrow(\sigma) \\ \delta \in C(\sigma, \tau)}} \text{MLP}_{M, p}\left[
            \begin{array}{c}
                \mathbf{h}_\tau^{(t)} \\[5pt]
                \mathbf{h}_\delta^{(t)}
            \end{array}
        \right] \right) \\[15pt]
        \text{MLP}_{\downarrow, p}\left( (1+\epsilon_\downarrow)\, \mathbf{h}_{\sigma}^{(t)} + \displaystyle\sum\limits_{\substack{\tau\in\mathcal{N}_\downarrow(\sigma) \\ \delta \in B(\sigma, \tau)}} \text{MLP}_{M', p}\left[
            \begin{array}{c}
                \mathbf{h}_\tau^{(t)} \\[5pt]
                \mathbf{h}_\delta^{(t)}
            \end{array}
        \right] \right)\\[15pt]
        \textit{additionally, when rewiring:} \\[5pt]
        \dbox{$\text{MLP}_{{R_\text{new}}, p}\left( (1+\epsilon_{R_\text{new}})\, \mathbf{h}_{\sigma}^{(t)} + \displaystyle\sum\limits_{\delta\in\mathcal{N}_{R_\text{new}}(\sigma)} \mathbf{h}_\delta^{(t)} \right)$}
    \end{array}
\right],
\end{equation*}

where $[\cdot]$ denotes column-wise concatenation, each $\epsilon_\bullet$ is a learnable scalar. $\text{MLP}_{M, p}$ and $\text{MLP}_{M', p}$ are single fully connected layers projecting back to the dimension of $\mathbf{h}_\sigma^{(t)}$. The other MLPs consists of two layer fully connected layers, each followed by a point-wise non-linearity (e.g., ReLU) and a BatchNorm layer. When entity unions or intersections aren't contained in the relational structure due to dimension constraints, we instead use $\mathbf{h}^{(t)}_\delta := \mathbf{0}$.

\section{Higher-Order Graphs are Relational Structures}
\label{appendix:higher-order-graphs}

Although we do not focus on higher-order graph neural networks ($k$-GNNs) \citep{morris2019weisfeiler} in this work, we briefly illustrate how they naturally fit into our relational framework. In $k$-GNNs, graphs are lifted into higher-order graphs following the procedure outlined in Algorithm \ref{alg:kgnn_lifting}. In our framework, these higher-order graphs can be viewed as relational structures, where each entity represents a set of nodes of size $n=k$. The relations between these entities are of two types: \emph{local} relations, which connect entities differing by one node and where the differing nodes are adjacent in the original graph, and \emph{global} relations, where the differing nodes are not adjacent in the original graph. From this perspective, $k$-GNN message passing on higher-order graphs is equivalent to R-GCN message passing on the corresponding relational structures. We leave the theoretical and empirical study of these relational structures for future work.

\begin{algorithm}[h]
\caption{Graph Lifting (Higher-Order Graph)}
\label{alg:kgnn_lifting}

\textbf{Input:} Unweighted undirected graph $G = (V, E)$, node features $(\mathbf{x}_v)_{v \in V}$, entity size $n$\\
\textbf{Output:} Relational structure $(\mathcal{S}, \{R_\texttt{local}, R_\texttt{global}\})$ and features $(\mathbf{x}_\sigma)_{\sigma \in \mathcal{S}}$

\begin{algorithmic}[1]
\State Initialize $\mathcal{S} = \emptyset$ \Comment{Entities}
\For{each subset $\sigma \subseteq V$ with $|\sigma| = n$}
    \State $\mathcal{S} \gets \mathcal{S} \cup \{\sigma\}$
    \State $\mathbf{x}_\sigma \gets \frac{1}{|\sigma|} \sum_{v \in \sigma} \mathbf{x}_v$ \Comment{Features}
\EndFor
\State Initialize $R_{\texttt{local}} = \emptyset$, $R_{\texttt{global}} = \emptyset$ \Comment{Relations}
\For{each pair $\sigma, \tau \in \mathcal{S}$}
    \If{$|\sigma \cap \tau| = n-1$}
        \State Let $\{s_{\sigma}\} = \sigma \setminus \tau$ and $\{t_{\tau}\} = \tau \setminus \sigma$
        \If{$\{s_{\sigma}, t_{\tau}\} \in E$}
            \State $R_{\texttt{local}} \gets R_{\texttt{local}} \cup \{(\sigma, \tau), (\tau, \sigma)\}$
        \Else
            \State $R_{\texttt{global}} \gets R_{\texttt{global}} \cup \{(\sigma, \tau), (\tau, \sigma)\}$
        \EndIf
    \EndIf
\EndFor
\State \textbf{return} $(\mathcal{S}, \{R_{\texttt{local}}, R_{\texttt{global}}\})$ and $(\mathbf{x}_\sigma)_{\sigma \in \mathcal{S}}$
\end{algorithmic}
\end{algorithm}

{
\section{Worked Example: Relations and Operators for a Path Graph}
\label{appendix:worked-example}

In this section, we illustrate, via a small example, the definitions of a simplicial complex, relational structure, and influence graph.

\paragraph{Simplicial Complex.} Consider the simplicial complex $\mathcal{K}$ corresponding to the path graph $i - j - k$. The complex $\mathcal{K}$ consists of
\begin{itemize}
\item $0$-simplices (vertices): $\{i\}$, $\{j\}$, $\{k\}$, and
\item $1$-simplices (edges): $\{i, j\}$, $\{j, k\}$.
\end{itemize}

\paragraph{Relational Structure.} The relations on the set of entities $\mathcal{S} = \mathcal{K}$ are as follows:
\begin{itemize}
\item Relation \( R_1 \) (identity):
\begin{eqnarray*}
R_1 &=&\{(\sigma) \mid \sigma \in \mathcal{K}\} \\
 &=& \left\{\, (\{i\}),\ (\{j\}),\ (\{k\}),\ (\{i, j\}),\ (\{j, k\}) \,\right\}.
\end{eqnarray*}

\item Relation $R_2$ (boundary):
\begin{eqnarray*}
R_2 &=& \{(\sigma, \tau) \mid \sigma \in \mathcal{K}, \tau \in \mathcal{B}(\sigma)\} \\
&=& \left\{\,
\begin{aligned}
&(\{i, j\}, \{i\}),\ (\{i, j\}, \{j\}),\\
&(\{j, k\}, \{j\}),\ (\{j, k\}, \{k\})
\end{aligned}
\,\right\}.
\end{eqnarray*}

\item Relation $R_3$ (co-boundary):
\begin{eqnarray*}
R_3 &=& \{(\sigma,\tau) \mid \sigma \in \mathcal{K}, \tau \in \mathcal{C}(\sigma)\} \\
&=& \left\{\,
\begin{aligned}
&(\{i\}, \{i, j\}),\ (\{j\}, \{i, j\}),\\
&(\{j\}, \{j, k\}),\ (\{k\}, \{j, k\})
\end{aligned}
\,\right\}.
\end{eqnarray*}

\item Relation $R_4$ (lower adjacency):
\begin{eqnarray*}
R_4 &=& \{(\sigma,\tau,\delta) \mid \sigma \in \mathcal{K}, \tau \in \mathcal{N}_\uparrow(\sigma), \delta = \sigma \cap \tau\} \\
&=& \left\{\,
\begin{aligned}
&(\{i, j\}, \{j, k\}, \{j\}),\\
&(\{j, k\}, \{i, j\}, \{j\})
\end{aligned}
\,\right\}.
\end{eqnarray*}

\item Relation $R_5$ (upper adjacency):

\begin{eqnarray*}
R_5 &=& \{(\sigma, \tau, \delta) \mid \sigma \in \mathcal{K}, \tau \in \mathcal{N}_\downarrow(\sigma), \delta = \sigma \cup \tau\} \\
&=& \left\{\,
\begin{aligned}
&(\{i\}, \{j\}, \{i, j\}),\ (\{j\}, \{i\}, \{i, j\}),\\
&(\{j\}, \{k\}, \{j, k\}),\ (\{k\}, \{j\}, \{j, k\})
\end{aligned}
\,\right\}.
\end{eqnarray*}
\end{itemize}
If we assume, for the sake of example, that all relations are involved in message passing, this gives rise to the relational structure
\begin{equation*}
\mathcal{R}(\mathcal{K}) = (\mathcal{S}, R_1, R_2, R_3, R_4, R_5).
\end{equation*}
Of course, one could restrict message passing to a smaller subset of the relations $\{R_1, R_2, R_3, R_4, R_5\}$.

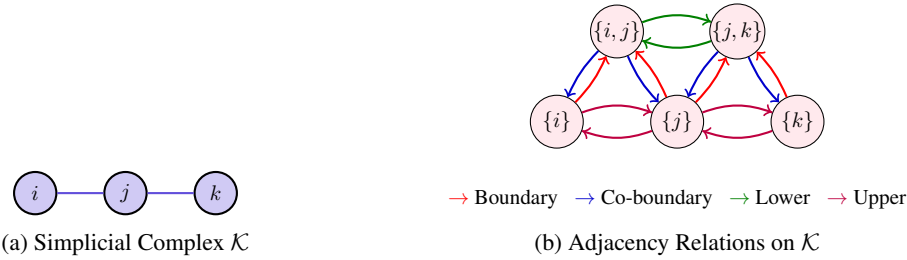
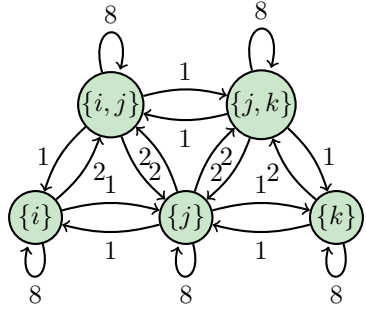
\begin{figure}[htbp]
    \centering

    \begin{minipage}[t]{0.35\textwidth}
        \centering
        \begin{subfigure}[t]{\linewidth}
            \centering
            \begin{tikzpicture}[scale=0.8, every node/.style={transform shape}]
                \tikzstyle{vertex}=[circle,fill=myblue!30,draw=black,thick,minimum size=20pt,inner sep=0pt]
                \tikzstyle{edge}=[thick, myblue]

                \node[vertex] (i) at (0,1) {$i$};
                \node[vertex] (j) at (1.5,1) {$j$};
                \node[vertex] (k) at (3,1) {$k$};
                
                \draw[edge] (i) -- (j);
                \draw[edge] (j) -- (k);
            \end{tikzpicture}
            \caption{Simplicial Complex \( \mathcal{K} \)}
        \end{subfigure}
    \end{minipage}
    \hfill
    \begin{minipage}[t]{0.6\textwidth}
        \centering
        \begin{subfigure}[t]{\linewidth}
            \centering
            \begin{tikzpicture}[scale=0.8, every node/.style={transform shape}]
                \tikzstyle{vertex}=[circle,fill=myred!30,draw=black,minimum size=25pt,inner sep=0pt]
                \tikzstyle{boundary}=[->,thick,blue!80!black,bend right=10]
                \tikzstyle{coboundary}=[->,thick,red,bend right=10]
                \tikzstyle{lower}=[thick,green!50!black,decoration={markings,mark=at position 0.99 with {\arrow{>}}},postaction={decorate}]
                \tikzstyle{upper}=[thick,purple,decoration={markings,mark=at position 0.99 with {\arrow{>}}},postaction={decorate}]

                \node[vertex] (i) at (0,0) {$\{i\}$};
                \node[vertex] (j) at (2,0) {$\{j\}$};
                \node[vertex] (k) at (4,0) {$\{k\}$};
                \node[vertex] (ij) at (1,1.5) {$\{i,j\}$};
                \node[vertex] (jk) at (3,1.5) {$\{j,k\}$};
                
                \draw[boundary] (ij) to (i);
                \draw[boundary] (ij) to (j);
                \draw[boundary] (jk) to (j);
                \draw[boundary] (jk) to (k);
                
                \draw[coboundary] (i) to (ij);
                \draw[coboundary] (j) to (ij);
                \draw[coboundary] (j) to (jk);
                \draw[coboundary] (k) to (jk);
                
                % Lower adjacency (R4)
                \draw[lower] (ij) to[out=20,in=160] (jk);
                \draw[lower] (jk) to[out=-160,in=-20] (ij);
                
                % Upper adjacency (R5)
                \draw[upper] (i) to[out=20,in=160] (j);
                \draw[upper] (j) to[out=20,in=160] (k);
                \draw[upper] (k) to[out=-160,in=-20] (j);
                \draw[upper] (j) to[out=-160,in=-20] (i);
                
                \node[anchor=north] at (2,-1) {\textcolor{red}{$\rightarrow$} Boundary \hspace{0.1cm} \textcolor{blue!80!black}{$\rightarrow$} Co-boundary   \hspace{0.1cm} \textcolor{green!50!black}{$\rightarrow$} Lower \hspace{0.1cm} \textcolor{purple}{$\rightarrow$} Upper};
            \end{tikzpicture}
            \caption{Adjacency Relations on \( \mathcal{K} \)}
            \label{fig:adjacency-relations}
        \end{subfigure}
    \end{minipage}
    
    \vspace{1cm} % Adjust vertical spacing as needed
    
    % Bottom row with one subfigure
    \begin{subfigure}[b]{0.7\textwidth}
        \centering
        % Subfigure (c): Influence Graph
        \begin{tikzpicture}[scale=1.0, every node/.style={transform shape}]
            % Styles
            \tikzstyle{vertex}=[circle,fill=mygreen!20,draw=black,thick,minimum size=20pt,inner sep=0pt]
            \tikzstyle{edge}=[->, thick]
            \tikzstyle{selfloop}=[->, thick, loop, min distance=10mm]
            
            % Nodes
            \node[vertex] (i) at (0,0) {\( \{i\} \)};
            \node[vertex] (j) at (2,0) {\( \{j\} \)};
            \node[vertex] (k) at (4,0) {\( \{k\} \)};
            \node[vertex] (ij) at (1,1.5) {\( \{i,j\} \)};
            \node[vertex] (jk) at (3,1.5) {\( \{j,k\} \)};
            
            % Self-loops
            \draw[selfloop] (i) edge[loop below] node[below] {\( 8 \)} (i);
            \draw[selfloop] (j) edge[loop below] node[below] {\( 8 \)} (j);
            \draw[selfloop] (k) edge[loop below] node[below] {\( 8 \)} (k);
            \draw[selfloop] (ij) edge[loop above] node[above] {\( 8 \)} (ij);
            \draw[selfloop] (jk) edge[loop above] node[above] {\( 8 \)} (jk);
            
            % Edges (weights from matrix B)
            \draw[edge] (i) to[bend left=15] node[above] {\( 1 \)} (j);
            \draw[edge] (j) to[bend left=15] node[below] {\( 1 \)} (i);
            \draw[edge] (i) to[bend right=15] node[right] {\( 2 \)} (ij);
            \draw[edge] (j) to[bend right=15] node[left] {\( 2 \)} (ij);
            \draw[edge] (j) to[bend left=15] node[right] {\( 2 \)} (jk);
            \draw[edge] (k) to[bend left=15] node[left] {\( 2 \)} (jk);
            \draw[edge] (j) to[bend left=15] node[above] {\( 1 \)} (k);
            \draw[edge] (k) to[bend left=15] node[below] {\( 1 \)} (j);
            \draw[edge] (ij) to[bend right=15] node[left] {\( 1 \)} (i);
            \draw[edge] (ij) to[bend right=15] node[right] {\( 2 \)} (j);
            \draw[edge] (jk) to[bend left=15] node[right] {\( 1 \)} (k);
            \draw[edge] (jk) to[bend left=15] node[left] {\( 2 \)} (j);
            \draw[edge] (ij) to[bend left=15] node[above] {\( 1 \)} (jk);
            \draw[edge] (jk) to[bend left=15] node[below] {\( 1 \)} (ij);
        \end{tikzpicture}
        \caption{Influence Graph \( \mathcal{G}(\mathcal{S}, \mathbf{B}) \)}
        \label{fig:influence-graph}
    \end{subfigure}
    
    \caption{(a) The simplicial complex \( \mathcal{K} \) consisting of nodes \( i \), \( j \), \( k \), and edges \( \{i,j\} \), \( \{j,k\} \). (b) The adjacency relations on \( \mathcal{K} \), showing boundary, co-boundary, lower, and upper relations. (c) The influence graph \( \mathcal{G}(\mathcal{S}, \mathbf{B}) \) representing the information flow between the entities.}
    \label{fig:combined-figures}
\end{figure}

\paragraph{Shift Operators.}
In what follows, we index the entries in matrices as follows:
\begin{center}
    \begin{tabular}{cl}
        \hline
        Index & Simplex \\
        \hline
        1 & $\{i\}$ \\
        2 & $\{j\}$ \\
        3 & $\{k\}$ \\
        4 & $\{i, j\}$ \\
        5 & $\{j, k\}$ \\
        \hline
    \end{tabular}
\end{center}
For instance, the $14$th entry in a matrix would correspond to the connection $(\{i\}, \{i,j\})$. We assume, for the sake of example, that all shift operators correspond to the adjacency tensors, i.e., the entries with indices corresponding to entities that are related by a relation are equal to $1$, and the remaining entries are all $0$. Of course, normalized adjacency tensors and others could be similarly considered. As such, we get the following:
\begin{itemize}
\item Shift operators $\mathbf{A}^{R_1}$ and $\tilde{\mathbf{A}}^{R_1}$ (identity relation):
\begin{equation*}
\mathbf{A}^{R_1} = \tilde{\mathbf{A}}^{R_1} =\begin{pmatrix}
1 & 0 & 0 & 0 & 0 \\ 
0 & 1 & 0 & 0 & 0 \\ 
0 & 0 & 1 & 0 & 0 \\ 
0 & 0 & 0 & 1 & 0 \\ 
0 & 0 & 0 & 0 & 1 \\ 
\end{pmatrix}.
\end{equation*}

\item Shift operators $\mathbf{A}^{R_2}$ and $\tilde{\mathbf{A}}^{R_2}$ (boundary relation):
\begin{equation*}
\mathbf{A}^{R_2} = \tilde{\mathbf{A}}^{R_2} = \begin{pmatrix}
0 & 0 & 0 & 0 & 0 \\
0 & 0 & 0 & 0 & 0 \\
0 & 0 & 0 & 0 & 0 \\
1 & 1 & 0 & 0 & 0 \\
0 & 1 & 1 & 0 & 0 \\
\end{pmatrix}.
\end{equation*}

\item Shift operators $\mathbf{A}^{R_3}$ and $\tilde{\mathbf{A}}^{R_3}$ (co-boundary relation):
\begin{equation*}
\mathbf{A}^{R_3} = \tilde{\mathbf{A}}^{R_3} = \begin{pmatrix}
0 & 0 & 0 & 1 & 0 \\
0 & 0 & 0 & 1 & 1 \\
0 & 0 & 0 & 0 & 1 \\
0 & 0 & 0 & 0 & 0 \\
0 & 0 & 0 & 0 & 0 \\
\end{pmatrix}.
\end{equation*}

\item Shift operator $\mathbf{A}^{R_4}$ and $\tilde{\mathbf{A}}^{R_4}$ (lower adjacency relation):

The non-zero entries of $\mathbf{A}^{R_4}$ are
\begin{equation*}
\mathbf{A}^{R_4}:
\begin{cases}
\mathbf{A}^{R_4}_{\{i,j\}, \{j, k\}, \{j\}} = 1 \\
\mathbf{A}^{R_4}_{\{j,k\}, \{i, j\}, \{j\}} = 1 \\
\end{cases}.
\end{equation*}

The aggregated influence for \( \tilde{\mathbf{A}}^{R_4} \):

\begin{equation*}
\tilde{\mathbf{A}}^{R_4} = \begin{pmatrix}
0 & 0 & 0 & 0 & 0 \\ 
0 & 0 & 0 & 0 & 0 \\ 
0 & 0 & 0 & 0 & 0 \\ 
0 & 1 & 0 & 0 & 1 \\ 
0 & 1 & 0 & 1 & 0 \\ 
\end{pmatrix}.
\end{equation*}

\item Shift Operator $\mathbf{A}^{R_5}$ and $\tilde{\mathbf{A}}^{R_5}$ (upper adjacency relation):

The non-zero entries of $\mathbf{A}^{R_5}$ are

\begin{equation*}
\mathbf{A}^{R_5}:
\begin{cases}
\mathbf{A}^{R_5}_{\{i\}, \{j\}, \{i, j\}} = 1 \\
\mathbf{A}^{R_5}_{\{j\}, \{i\}, \{i, j\}} = 1 \\
\mathbf{A}^{R_5}_{\{j\}, \{k\}, \{j, k\}} = 1 \\
\mathbf{A}^{R_5}_{\{k\}, \{j\}, \{j, k\}} = 1 \\
\end{cases}.
\end{equation*}

The aggregated influence for $\tilde{\mathbf{A}}^{R_5}$:
\begin{equation*}
\tilde{\mathbf{A}}^{R_5} = \begin{pmatrix}
0 & 1 & 0 & 1 & 0 \\ 
1 & 0 & 1 & 1 & 1 \\ 
0 & 1 & 0 & 0 & 1 \\ 
0 & 0 & 0 & 0 & 0 \\ 
0 & 0 & 0 & 0 & 0 \\ 
\end{pmatrix}
\end{equation*}.
\end{itemize}
Lastly, the aggregated influence matrix $\tilde{\mathbf{A}}$ arising from including all relations $R_1$ to $R_5$:

\begin{eqnarray*}
\tilde{\mathbf{A}} &=& \tilde{\mathbf{A}}^{R_1} + \tilde{\mathbf{A}}^{R_2} + \tilde{\mathbf{A}}^{R_3} + \tilde{\mathbf{A}}^{R_4} + \tilde{\mathbf{A}}^{R_5} \\
 &=& \begin{pmatrix}
1 & 1 & 0 & 2 & 0 \\ 
1 & 1 & 1 & 2 & 2 \\ 
0 & 1 & 1 & 0 & 2 \\ 
1 & 2 & 0 & 1 & 1 \\ 
0 & 2 & 1 & 1 & 1 \\ 
\end{pmatrix}.
\end{eqnarray*}
The maximum row sum of $\tilde{\mathbf{A}}$ is $\gamma = 7$, giving rise to the augmented adjacency matrix
\begin{eqnarray*}
\mathbf{B} &=& \gamma \mathbf{I} + \tilde{\mathbf{A}} \\
&=& \begin{pmatrix}
8 & 1 & 0 & 2 & 0 \\ 
1 & 8 & 1 & 2 & 2 \\ 
0 & 1 & 8 & 0 & 2 \\ 
1 & 2 & 0 & 8 & 1 \\ 
0 & 2 & 1 & 1 & 8 \\ 
\end{pmatrix}.
\end{eqnarray*}
The influence graph is finally constructed from $\mathbf{B}$ as follows:
\begin{itemize}
    \item Nodes: Each simplex in $\mathcal{K}$.
    \item Edges: For each $(\sigma, \tau)$, if $\mathbf{B}_{\sigma, \tau} > 0$, there is a directed edge from $\tau$ to $\sigma$ with weight $\mathbf{B}_{\sigma, \tau}$.
\end{itemize}
}

\section{Computational Complexity}

It is well known in the TDL community that topological methods can require much more computation due to directly modeling higher-order structures. This typically incurs combinatorial complexity in the algorithms. A common strategy is to truncate the dimensionality of the objects considered to maintain polynomial complexity. Another option is to use a sparse complex construction method. As noted earlier, more work is needed in understanding graph lifting and which method is preferred in a given setting. In this section, we analyze the running times required for TDL with clique complexes.

\subsection{Clique Complex Size}
The exact computational requirements of performing clique lifting depend on the graph in consideration. However, there are a few cases where we can provide precise estimates. 

\paragraph{Clique lifting of dimension 1}
In this case, the highest dimensional cell is the edge cell. Consider a graph $G = (V, E)$ with $|V| = n$ nodes, $|E| = m$ edges, and maximum node degree $d$. The corresponding clique complex $\mathcal{K}$ has: 
\begin{enumerate}
    \item $n+m$ cells
    \item $2 m$ boundary relations
    \item $2 m$ co-boundary relations
    \item $\sum_{v\in V} 2\binom{\deg(v)}{2} \leq 2md$ lower relations
    \item $2m$ upper relations
\end{enumerate}
Note that we need to count the upper and lower relations for each direction of the edge so we get a factor of 2. 

\paragraph{Clique lifting of dimension 2}
If $G$ has $|F| = k$ faces, then the clique $\mathcal{K}$ contains
\begin{enumerate}
    \item $n+m+k$ cells
    \item $2m+3k$ boundary relations
    \item $2m+3k$ co-boundary relations
    \item $\sum_{v\in V} 2\binom{\deg(v)}{2} + \sum_{\sigma\in F}\sum_{\tau\in F} \chi(\sigma\cap \tau)$ lower relations
    \item $ 2m + 6k$ upper relations
\end{enumerate}
In dense graphs, the number of lower relations can be very large. Depending on the dataset under consideration, one may opt for a model like CIN where lower relations are ignored. 

\paragraph{Complete Graph}
We can provide precise estimates when we have a complete graph. This can help us understand performance in the limiting case that the graph is very dense or contains large cliques. Consider a complete graph $K_n$ on $n$ nodes and the cell $\sigma = \{1, 2, ..., m\}$ for $m \geq 1$. This is the boundary of $n-m$ cells and is the co-boundary of $m$ cells. It is upper adjacent to $(n-m)m$ cells. It is lower adjacent to $m(n-m)$ cells. Counting all the cells of dimension $d = m-1$, we get
\begin{enumerate}
    \item $n_d = \binom{n}{d+1}$ cells
    \item $n_d (n-d-1)$ boundary relations (if $d > 0$)
    \item $n_{d} (d+1)$ co-boundary relations
    \item $n_d (d+1)(n-d-1)$ lower relations (if $d > 0$)
    \item $n_d (d+1)(n-d-1)$ upper relations
\end{enumerate}
, where the $d>0$ condition comes from the fact that there isn't a cell $\{\}$ in the complex so there aren't lower or co-boundary relations for zero dimensional cells. 
Computing the total number of cells gives $2^n - 1$. The total number of edges is 
$$-n - n(n-1) + \sum_{d = 0}^{n-1} \binom{n}{d+1}\left(n + (n-d-1)(d+1)\right) $$
$$= \frac{n}{2} (2^n n + 2^n -2) - n^2 = O(n^2 2^n)$$
This bound is prohibitive for large $n$. If we restrict to two dimensional cells, we get the total number of cells is
$$n + \binom{n}{2} + \binom{n}{3} = \frac{n^3 + 5n}{6} = O(n^3)$$
and the total number of relations is
$$\frac{n}{6} (-24 + 41 n - 24 n^2 + 7 n^3) = O(n^4)$$

\subsection{Clique Lift Complexity}
Our main use case is computing complexes for two-dimensional cells. 
In this section, we compute the complexity of clique lifting. Note that computing the 0 and 1-dimensional cells is fast since they are encoded directly in the graph as nodes and edges, respectively. The boundary and coboundary relations encode edge incidence to a node and are also easily obtained. The upper relations on the 0-cells are just the edges in the graph. The lower relations on the edges are all pairs of edges in the graph sharing a node. This can be efficiently computed by storing the edge incidences in a dictionary in $O(m)$ time and then computing all pairs in the neighborhood of each node. This requires $O\left(n + m + \sum_{v\in V} \binom{\deg(v)}{2}\right)$ computation. For connected graphs, this reduces to $O(md)$. 

Computing the two-dimensional cells can also be done efficiently. By storing the neighborhood dictionary, we just need to loop over pairs of edges $(u, v), (u, w)$ in the neighborhood of $u$ and check if $(v, w)\in G$. This can be done in $O\left(\sum_{v\in V} \binom{\deg(v)}{2} \right) = O(md)$ time. The boundary, co-boundary, and upper adjacencies can be easily computed in $O(k)$ time. The lower adjacency computation for 2-cells can require more time. For a given edge $(u, v)$, consider the 2-cells $\sigma$ such that $(u, v) \prec \sigma$. Each pair of these gets a lower adjacency and can be efficiently added. This results in 
$$O\left(\sum_{e\in E} \binom{|\mathcal{C}(e)|}{2}\right)$$
computation. We can bound this as
$$\sum_{e\in E} \binom{|\mathcal{C}(e)|}{2} \leq \frac{1}{2} \sum_{e\in E} |\mathcal{C}(e)|^2 \leq \frac{1}{2} \max_e |\mathcal{C}(e)| \sum_{e\in E} |\mathcal{C}(e)| = \frac{3k}{2} \max_e |\mathcal{C}(e)|$$
and we get the complexity bound
$$O\left(k \max_e |\mathcal{C}(e)|\right)$$
This is analogous to the result with lower adjacency of 1-cells that relied on the maximum node degree in $G$. The maximum co-boundary degree $\max_e |\mathcal{C}(e)|$ is the analogue of node degree for 2-cells. So, the clique complex construction up to dimension 2 has time complexity:
$$O(md + k \max_e |\mathcal{C}(e)|)$$

\subsection{Running Times}

\begin{table}[h]
    \centering
    \begin{tabular}{llll}
        \toprule
        Lift        & None  & Clique    & Ring      \\
        \midrule
        ENZYMES     & 2.12  &  6.89     & 11.81     \\
        IMDB-B      & 3.51  & 45.04     & 101.98    \\
        MUTAG       & 1.60  & 1.22      & 1.48      \\
        NCI1        & 4.70  & 20.26     & 26.62     \\
        PROTEINS    & 2.99  & 16.07     & 30.73     \\
        ZINC        & 63.47 & 94.70     & 106.37    \\
        TEXAS       & 6.02  & 2.64      & 0.65      \\
        WISCONSIN   & 6.09  & 2.78      & 0.74      \\
        CORNELL     & 5.05  & 2.76      & 0.62      \\
        CORA        & 5.77  & 8.29      & 45.71     \\
        CITESEER    & 7.05  & 7.65      & NA        \\
        \bottomrule
    \end{tabular}
    \caption{Time to perform graph lifting and download datasets from Torch Geometric.}
    \label{tab:time_lift}
\end{table}

\end{document}